\documentclass[sigconf]{aamas} 

\usepackage{balance} 

\usepackage{ucs}
\usepackage[utf8]{inputenc}
\usepackage[T1]{fontenc}
\usepackage{graphicx}
\usepackage{xspace}
\usepackage{amsmath,amsthm}
\usepackage{mathtools}
\usepackage{comment}
\usepackage{hyperref}
\usepackage{dsfont}
\usepackage{nicefrac}
\usepackage{enumerate}
\usepackage[inline,shortlabels]{enumitem}
\usepackage{wrapfig}
\usepackage{mathrsfs}
\usepackage{caption}
\usepackage{float}
\usepackage[ruled, vlined]{algorithm2e}
\usepackage[labelformat=simple]{subcaption}

\usepackage{scalerel}
\usepackage{booktabs}
\usepackage{multicol}
\usepackage{tikz}
\usetikzlibrary{positioning, quotes, patterns}
\newtheorem{assumption}{Assumption}
\newtheorem{problem}{Problem}

\newcommand{\figref}[1]{Figure~\ref{fig:#1}\xspace}
\newcommand{\secref}[1]{Sect.~\ref{sec:#1}\xspace}

\newcommand{\planner}{\ensuremath{P}\xspace}

\newcommand{\G}{{\mathcal{G}}}

\newcommand{\T}{{\mathcal{T}}}

\newcommand{\M}{{\mathcal{M}}}
\newcommand{\HM}{{\mathcal{H}}}

\newcommand{\zug}[1]{\langle #1  \rangle}
\newcommand{\set}[1]{\left\{ #1 \right\}}
\newcommand{\stam}[1]{}

\newcommand{\overbar}[1]{\mkern 1.5mu\overline{\mkern-1.5mu#1\mkern-1.5mu}\mkern 1.5mu}

\newcommand{\pacman}{%
    \begin{tikzpicture}[scale=0.125] 
        \draw[thick] (0,0) -- (30:1) arc (30:330:1) -- cycle;

        \fill[yellow] (0,0) -- (30:1) arc (30:330:1) -- cycle;

        \fill[black] (0.4,0.5) circle (0.125);
    \end{tikzpicture}%
}
\newcommand{\ghost}[1]{\tikz[baseline=.1em,scale=.4]{
  \draw [fill=#1] (0,0) -- (0,.5) arc (+180:0:.3) -- (.6,0) --
  (.5,.15) -- (.4,0) -- (.3,.15) -- (.2,0) -- (.1,.15) -- cycle;
    \coordinate (eye) at (180:.03);
    \foreach \x in {.17,.43}{
      \fill[white] (\x,.5) circle[radius=.1];
      \fill[black] (\x,.5) ++(eye) circle[radius=.05];
    }
}}

\pgfdeclarepatternformonly{falling_tiles}{\pgfqpoint{0pt}{0pt}}{\pgfqpoint{1mm}{1mm}}{\pgfqpoint{1mm}{1mm}}%
{
  \pgfsetlinewidth{0.4pt}
  \pgfpathmoveto{\pgfqpoint{0pt}{0pt}}
  \pgfpathlineto{\pgfqpoint{1mm}{1mm}}
  \pgfpathmoveto{\pgfqpoint{1.7mm}{-0.3mm}}
  \pgfpathlineto{\pgfqpoint{2.3mm}{0.3mm}}
  \pgfpathmoveto{\pgfqpoint{-0.3mm}{1.7mm}}
  \pgfpathlineto{\pgfqpoint{0.3mm}{2.3mm}}
  \pgfusepath{stroke}
}
\definecolor{gold}{rgb}{0.83, 0.69, 0.22}
\newcommand{\ghosttile}[1]{\tikz[baseline=.1em,scale=.4]{
    \fill[pattern=falling_tiles,pattern color=red] (0,0) -- (0,.5) arc (+180:0:.3) -- (.6,0) --
    (.5,.15) -- (.4,0) -- (.3,.15) -- (.2,0) -- (.1,.15) -- cycle;
    \draw[red] (0,0) -- (0,.5) arc (+180:0:.3) -- (.6,0) --
    (.5,.15) -- (.4,0) -- (.3,.15) -- (.2,0) -- (.1,.15) -- cycle;
}}

\newcommand{\fun}[1]{\ensuremath{\mathopen{}\mathclose\bgroup\left(#1\aftergroup\egroup\right)}}
\newcommand{\tuple}[1]{\zug{#1}}

\newcommand{\Reals}{\ensuremath{\mathbb{R}}}
\newcommand{\hmdp}{\ensuremath{\mathcal{H}}}
\newcommand{\condition}[1]{\ensuremath{\mathds{1}\set{{#1}}}}
\newcommand{\abs}[1]{\ensuremath{\left| #1 \right|}}
\newcommand{\norm}[1]{\ensuremath{\left\| #1 \right\|}}

\newcommand{\expect}{\ensuremath{\mathbb{E}}}
\newcommand{\Prob}{\ensuremath{\mathbb{P}}}
\newcommand{\expectedsymbol}[1]{\expect_{#1\,}}
\newcommand{\expected}[2]{\expect_{#1}\left[#2\right]}
\newcommand{\sampledot}{\ensuremath{{\cdotp}}}
\newcommand{\distributionssymbol}{\ensuremath{\Delta}}
\newcommand{\distributions}[1]{\ensuremath{\distributionssymbol\fun{#1}}}
\newcommand{\measurableset}{\ensuremath{\mathcal{X}}}

\newcommand{\support}[1]{\text{supp}\fun{#1}}

\newcommand{\divergence}{\ensuremath{\mathcal{D}}}

\newcommand{\mdp}{\ensuremath{\mathcal{M}}}
\newcommand{\car}{\ensuremath{\scaleto{\rotatebox[origin=c]{90}{$\circlearrowright$}}{1.4ex}}}
\newcommand{\selfloop}[2]{\ensuremath{#1^{\nobreak\hspace{.065em} \car #2 }}}
\newcommand{\istates}{\ensuremath{\mathcal{S}}}
\newcommand{\actions}{\ensuremath{\mathcal{A}}}
\newcommand{\transitionfn}{\ensuremath{\mathbf{P}}}
\newcommand{\probtransitions}{\ensuremath{\transitionfn}}
\newcommand{\mdpI}{\ensuremath{\mathbf{I}}}
\newcommand{\mdptuple}{\tuple{\istates, \actions, \transitionfn, \mdpI}}
\newcommand{\istate}{\ensuremath{s}}
\newcommand{\action}{\ensuremath{a}}
\newcommand{\pathdistribution}[2]{\ensuremath{{\textnormal{Pr}}^{#1}_{#2}}}
\newcommand{\policy}{\ensuremath{\pi}}
\newcommand{\policystates}{\ensuremath{\mathcal{Q}}}
\newcommand{\policystate}{\ensuremath{q}}

\newcommand{\mealyaction}[1]{\ensuremath{#1_{\textnormal{a}}}}
\newcommand{\mealyupdate}[1]{\ensuremath{#1_{\textnormal{u}}}}

\newcommand{\mdppath}{\ensuremath{\rho}}
\newcommand{\objective}{\ensuremath{\mathbb{O}}}
\newcommand{\valuefn}[4]{\ensuremath{
    \ifthenelse{\equal{#4}{}}{
        \ifthenelse{\equal{#2}{}}{
            V^{#1}_{\mdpI}
        }{
            V^{#1}_{\mdpI}\fun{#2}
        }
    }{\ifthenelse{\equal{#4}{none}}{
        \ifthenelse{\equal{#2}{}}{
            V^{#1}
        }{
            V^{#1}\fun{#2}
        }
    }{
        \ifthenelse{\equal{#2}{}}{
            V^{#1}\fun{#4}
        }{
            V^{#1}\fun{#4, #2}
        }
    }
}}}
\newcommand{\val}[3]{\valuefn{#1}{#2}{#3}{}}
\newcommand{\values}[3]{\valinit{#1}{#2}{\discount}{#3}}
\newcommand{\latentvalues}[3]{\latentvalinit{#1}{#2}{\discount}{#3}}
\newcommand{\valinit}[4]{\valuefn{#1}{#2}{#3}{#4}}
\newcommand{\valinitfast}[1]{\valinit{\policy}{\objective}{\discount}{#1}}
\newcommand{\discount}{\gamma}
\newcommand{\eventually}[1]{\ensuremath{\Diamond #1}}

\newcommand{\reach}[1]{\ensuremath{\eventually{#1}}}
\newcommand{\reachavoid}[2]{\ensuremath{\objective(#1,#2)}}
\newcommand{\stationary}[1]{\ensuremath{\xi_{#1}}}
\newcommand{\sreset}{\ensuremath{\istate_{\text{reset}}}}
\newcommand{\sinit}{\sreset}
\newcommand{\reward}{\ensuremath{\mathit{rew}}\xspace}

\newcommand{\graph}{\ensuremath{\mathcal{G}}}
\newcommand{\vertices}{\ensuremath{\mathcal{V}}}
\newcommand{\vertex}{\ensuremath{v}}
\newcommand{\varvertex}{\ensuremath{u}}
\newcommand{\edges}{\ensuremath{E}}

\newcommand{\neighbors}[1]{\ensuremath{N\fun{#1}}}

\renewcommand{\planner}{\ensuremath{\mathrm{\Pi}}}

\newcommand{\controller}{\ensuremath{\tau}}
\newcommand{\out}[1]{\ensuremath{\textit{out}\fun{#1}}}
\newcommand{\room}{\ensuremath{R}}

\newcommand{\rooms}{\ensuremath{\mathcal{R}}}
\newcommand{\labelingfn}{\ensuremath{\ell}}
\newcommand{\directions}{\ensuremath{D}}
\newcommand{\direction}{\ensuremath{d}}
\newcommand{\exitfn}{\ensuremath{\mathcal{O}}}
\newcommand{\entrancefn}{\ensuremath{\mathcal{I}}}
\newcommand{\stationaryroom}{\ensuremath{\stationary{\latentpolicy_{\room, \direction}}^{\room_{\direction, \policy}}}}

\newcommand{\stationaryroomtrain}{\ensuremath{\stationary{\latentpolicy_{\room, \direction}}^{\room}}}

\newcommand{\refinedmdp}[1]{\ensuremath{\mdp_{\planner}^{#1}}}
\newcommand{\refinedstates}{\ensuremath{\istates_{\planner}}}
\newcommand{\refinedactions}{\ensuremath{\actions_{\planner}}}
\newcommand{\mdpplanner}[1]{\ensuremath{\refinedmdp{#1}}}

\newcommand{\istatesplanner}{\ensuremath{\refinedstates}}
\newcommand{\actionsplanner}{\ensuremath{\refinedactions}}
\newcommand{\transitionfnplanner}{\ensuremath{\transitionfn_{\planner}}}
\newcommand{\latenttransitionfnplanner}{\ensuremath{\latentprobtransitions_{\planner}}}
\newcommand{\mdpIplanner}{\ensuremath{\mdpI}_{\planner}}
\newcommand{\plannerobjective}{\ensuremath{\objective}}

\newcommand{\ldqn}{\ensuremath{L_{\text{DQN}}}}

\newcommand{\overbarit}[1]{\,\overline{\!{#1}}}
\newcommand{\latentmdp}{\ensuremath{\overbarit{\mdp}}}
\newcommand{\latentstates}{\ensuremath{\overbarit{\mathcal{\istates}}}}
\newcommand{\latentstate}{\ensuremath{\bar{\istate}}}

\newcommand{\latentaction}{\ensuremath{{\action}}}
\newcommand{\latentpolicy}{\ensuremath{\overbar{\policy}}}
\newcommand{\latentprobtransitions}{\ensuremath{\overbar{\transitionfn}}}

\newcommand{\latentmdpI}{\ensuremath{\overline{\mdpI}}}
\newcommand{\latentmdptuple}{\ensuremath{\langle{\latentstates, \actions, \latentprobtransitions, \latentmdpI \rangle}}}
\newcommand{\embed}{\ensuremath{\phi}}
\newcommand{\latentvaluefn}[4]{\ensuremath{
    \ifthenelse{\equal{#4}{}}{
        \ifthenelse{\equal{#2}{}}{
            {\overbar{V}_{\latentmdpI}^{\scriptstyle #1}}
        }{
            {\overbar{V}_{\latentmdpI}^{\scriptstyle #1}\fun{#2}}
        }
    }{\ifthenelse{\equal{#4}{none}}{
        \ifthenelse{\equal{#2}{}}{
            {\overbar{V}^{\scriptstyle #1}}
        }{
            {\overbar{V}^{\scriptstyle #1}\fun{#2}}
        }
    }{
        \ifthenelse{\equal{#2}{}}{
            \overbar{V}^{\scriptstyle #1}\fun{#4}
        }{
            \overbar{V}^{\scriptstyle #1}\fun{#4, #2}
        }
    }
}}}
\newcommand{\latentvalinit}[4]{\latentvaluefn{#1}{#2}{#3}{#4}}

\newcommand{\localtransitionloss}[1]{\transitionloss^{}}
\newcommand{\transitionloss}{L_{\probtransitions}}
\newcommand{\localtransitionlossupper}[1]{{L}_{\probtransitions}^{\uparrow}}
\newcommand{\latententrancefn}{\ensuremath{\overline{\entrancefn}}}
\newcommand{\entranceloss}{\ensuremath{L_{\entrancefn}}}

\newcommand{\error}{\ensuremath{\varepsilon}}
\newcommand{\varerror}{\ensuremath{\zeta}}
\newcommand{\proberror}{\ensuremath{\delta}}
\newcommand{\localtransitionlossapprox}[1]{\widehat{L}_{\probtransitions}^{#1}}
\newcommand{\stationaryapprox}{\ensuremath{\widehat{\xi}_{\text{reset}}}}

\usepackage[framemethod=TikZ]{mdframed}
\usepackage{array,multirow,graphicx}
\usepackage{orcidlink}

\usepackage{xcolor}
\colorlet{graphcolor}{black!50!green}
\colorlet{mdpcolor}{blue}

\definecolor{amaranth}{rgb}{0.7, 0.17, 0.9}

\newcommand{\smallparagraph}[1]{\noindent\textbf{#1}}

\theoremstyle{definition}
\newtheorem{theorem}{Theorem}
\newtheorem{lemma}{Lemma}

\newtheorem{property}{Property}
\theoremstyle{remark}
\newtheorem{example}{Example}

\mdfsetup{skipabove=\topskip,skipbelow=\topskip}

\makeatletter
\gdef\@copyrightpermission{
  \begin{minipage}{0.2\columnwidth}
   \href{https://creativecommons.org/licenses/by/4.0/}{\includegraphics[width=0.90\textwidth]{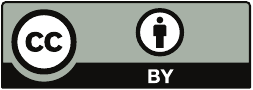}}
  \end{minipage}\hfill
  \begin{minipage}{0.8\columnwidth}
   \href{https://creativecommons.org/licenses/by/4.0/}{This work is licensed under a Creative Commons Attribution International 4.0 License.}
  \end{minipage}
  \vspace{5pt}
}
\makeatother

\setcopyright{ifaamas}
\acmConference[AAMAS '25]{Proc.\@ of the 24th International Conference
on Autonomous Agents and Multiagent Systems (AAMAS 2025)}{May 19 -- 23, 2025}
{Detroit, Michigan, USA}{Y.~Vorobeychik, S.~Das, A.~Nowé  (eds.)}
\copyrightyear{2025}
\acmYear{2025}
\acmDOI{}
\acmPrice{}
\acmISBN{}

\newcommand{\highlight}[1]{#1}
\newcommand{\highlightt}[1]{#1}

\acmSubmissionID{aaai022}

\title{\highlight{Composing Reinforcement Learning Policies, with~Formal~Guarantees}}
\subtitle{AAAI Track}

\keywords{Planning and Reasoning under Uncertainty; Controller Synthesis; Model Checking; Representation Learning; Reinforcement Learning}

\author{Florent Delgrange \orcidlink{0000-0003-2254-0596}}
\affiliation{
  \institution{Vrije Universiteit Brussel}
  \city{Brussels}
  \country{Belgium}}
\email{florent.delgrange@vub.be}

\author{Guy Avni \orcidlink{0000-0001-5588-8287}}
\affiliation{
  \institution{University of Haifa}
  \city{Haifa}
  \country{Israel}}
\email{gavni@cs.haifa.ac.il}

\author{Anna Lukina \orcidlink{0000-0001-9525-0333}}
\affiliation{
  \institution{TU Delft}
  \city{Delft}
  \country{The Netherlands}}
\email{a.lukina@tudelft.nl}

\author{Christian Schilling \orcidlink{0000-0003-3658-1065}}
\affiliation{
  \institution{Aalborg University}
  \city{Aalborg}
  \country{Denmark}}
\email{christianms@cs.aau.dk}

\author{Ann Now\'e \orcidlink{0000-0001-6346-4564}}
\affiliation{
  \institution{Vrije Universiteit Brussel}
  \city{Brussels}
  \country{Belgium}}
\email{ann.nowe@vub.be}

\author{Guillermo A. P\'erez \orcidlink{0000-0002-1200-4952}}
\affiliation{
  \institution{University of Antwerp}
  \city{Antwerp}
  \country{Belgium}}
\email{guillermo.perez@uantwerpen.be}

\begin{abstract}
\highlight{
We propose a novel framework to controller design in environments with a two-level structure: a known high-level graph (``map'') in which each vertex is populated by a Markov decision process, called a ``room''. The framework ``separates concerns'' by using different design techniques for low- and high-level tasks. 
We apply reactive synthesis for high-level tasks: given a specification as a logical formula over the high-level graph and a collection of low-level policies obtained together with ``concise'' {\em latent} structures, we construct a ``planner'' that selects which low-level policy to apply in each room. 
We develop a \highlightt{reinforcement learning}
procedure to train low-level policies on latent structures, which unlike previous approaches, circumvents a model distillation step. We pair the policy with probably approximately correct guarantees on its performance and on the abstraction quality, and lift these guarantees to the high-level task. 
These formal guarantees are the main advantage of the framework. Other advantages include scalability (rooms are large and their dynamics are unknown) and reusability of low-level policies. 
We demonstrate feasibility in challenging case studies where an agent navigates environments with moving obstacles \highlightt{and visual inputs}.
}
\stam{
We propose a novel framework to controller design in environments with a two-level structure: a high-level graph in which each vertex is populated by a  Markov decision process, called a ``room'', with several low-level objectives. 
We proceed as follows. 
First, we apply deep reinforcement learning (DRL) to obtain low-level policies for each room and objective.
Second, we apply reactive synthesis to obtain a planner that selects which low-level policy to apply in each room. 
Reactive 
synthesis refers to constructing a planner for a given model of the environment that satisfies a given objective (typically specified as a temporal logic formula) by design. 
The main advantage of the framework is formal guarantees. 
In addition, the framework enables a ``separation of concerns'': low-level tasks are addressed using DRL, which enables scaling to large rooms of unknown dynamics, reward engineering is only done locally, and policies can be reused, whereas users can specify high-level tasks intuitively and naturally.
The central challenge in synthesis is the need for a model of the rooms. We address this challenge by developing a DRL procedure to train concise ``latent'' policies together with latent abstract rooms, both paired with probably approximately correct (PAC) guarantees on performance and abstraction quality. Unlike previous approaches, this circumvents a model distillation step.
We demonstrate feasibility in a case study involving agent navigation in an environment with moving obstacles.
}
\end{abstract}

\begin{document}

\maketitle

\section{Introduction}\label{sec:introduction}
We consider the fundamental problem of constructing control {\em policies} for environments modeled as {\em Markov decision processes} (MDPs) with formal guarantees. 
We deal with long-horizon tasks in environments with prior structural knowledge: the input to our method is a (high-level) {\em map} given as a \emph{graph}, where each vertex is populated by an MDP with unknown dynamics called a {\em room}, and the long-horizon task is given on the map. 
Such settings arise naturally. 
As a running example, a robot delivers a package in a warehouse with moving obstacles (e.g., forklifts, workers, or other robots); while it is infeasible to model the low-level interactions of the agent with its immediate surroundings, modeling the high-level map of the rooms in the warehouse requires minimal engineering effort. 
\highlight{%
We list other examples of two-level domains with prior knowledge of the high-level architecture and in which our method is relevant: 
(i)~\emph{routing}~\cite{DBLP:conf/cav/JungesS22}: the network topology, e.g., connection between routers, is often known but modeling low-level routing decisions is intricate; 
(ii)~\emph{skill graphs}~\cite{DBLP:conf/icml/BagariaS021} of agents, e.g., ``grab a key'' and ``open a door'', and their dependencies are naturally modeled as a graph; (iii)~\emph{software systems}~\cite{DBLP:conf/sosp/RyzhykCKSH09}, in particular \emph{probabilistic programs}~\cite{DBLP:conf/cav/JungesS22}: each vertex represents a software component (an MDP in a probabilistic program) and edges capture dependencies or interactions.
}%

Our framework ``separates concerns'' by using different design techniques for low- and high-level tasks with complementary benefits and drawbacks.
For high-level tasks, we apply {\em reactive synthesis}~\cite{PnueliR89}, which constructs an optimal policy \emph{based on a model of the environment} and \emph{a specification as a logical formula}, yielding a {\em guarantee that the policy satisfies the specification}. Logic is an \emph{intuitive and natural specification language}. The reliance on an explicit environment model hinders scalability and application to domains with partially-known dynamics. 
\highlightt{Hence, we solve low-level tasks via {\em reinforcement learning} (RL~\cite{SuttonB98}).
In particular, we may use \emph{deep} RL (DRL~\cite{MK+13})}, which is successful in domains of \emph{high-dimensional feature spaces with unknown dynamics}.
However, RL generally lacks formal guarantees and struggles with long-term objectives, where one needs to deal with the notorious problem of sparse rewards~\cite{DBLP:journals/inffus/LadoszWKO22} by guiding the agent
~\cite{DBLP:conf/ijcai/LiuZ022}, which in turn poses an engineering effort. 
\begin{figure*}[t]
    \begin{subfigure}[c]{.2\linewidth}
        \centering
        \includegraphics[width=\linewidth]{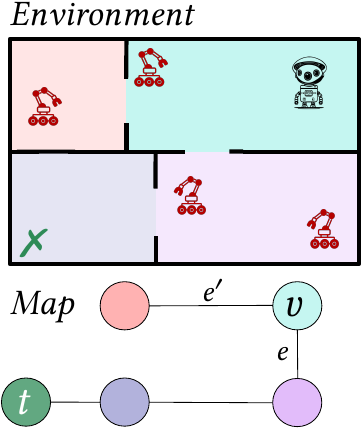}
        \caption{%
        \highlight{%
            Two-level environment partitioned into rooms.
            }%
        }%
        \label{fig:framework-a}
    \end{subfigure}
    \hfill
    \begin{subfigure}[c]{.35\linewidth}
        \centering
        \includegraphics[width=.95\linewidth]{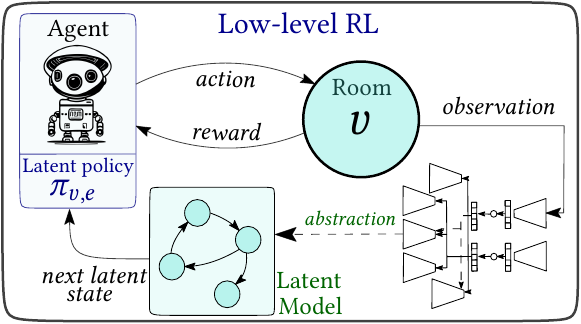}
        \caption{Latent models and policies are learned conjointly with the RL process. Both are paired with PAC guarantees on the abstraction quality of the model and the performance of the policy.}
        \label{fig:intro-1}
    \end{subfigure}
    \hfill
    \begin{subfigure}[c]{.41\linewidth}
        \centering
        \includegraphics[width=\linewidth]{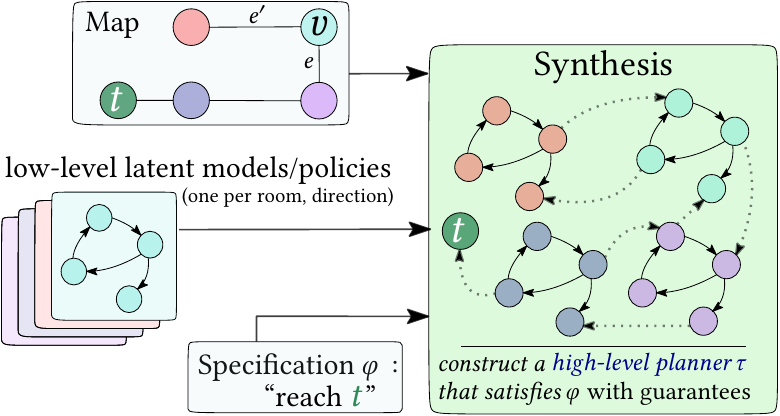}
        \caption{Planner synthesis.}
        \label{fig:intro-2}
    \end{subfigure}
        \vspace{-.5em}
    \caption{%
        \highlight{%
        \textbf{(a)}~Environment in which the agent (top-right) needs to reach the target (green, bottom-left) while avoiding moving adversaries (red).
        The target appears in the map as a dedicated vertex
        {\color{green!50!black}$t$}.
        }%
        \textbf{(b)}~The agent is trained to exit \emph{each room}, in \emph{every possible direction}. Training is performed in \emph{parallel simulations}. An abstraction of the environment is learned via neural networks, yielding a latent model for each room.
        Simultaneously, a policy is learned via RL on the learned latent representation, which guarantees the agent's low-level behavior conformity through PAC bounds.
        \emph{More details in Sect.~\ref{sec:low_level}}.
        \textbf{(c)}~Given a high-level description of the environment, a collection of latent models and policies for each room, and the specifications, synthesis outputs a high-level planner guaranteed to satisfy the specifications.
    The challenge resides in the way the low-level components are merged to apply synthesis while maintaining their guarantees.
    \emph{More details 
    in Sect.~\ref{sec:high_level}}.
    }
    \label{fig:framework}
    \Description{Three depictions of the framework proposed in the paper: first, a maze and the corresponding graph representation for the map; second, the usual agent-environment loop showing a neural network observing the environment and then abstracting it into a latent model updated and feeding its output (latent) observation back to the agent; finally, a depiction of how the map feeds into the synthesis procedure that can string together the low-level latent policies to satisfy a specified requirement.}
    \vspace{-.5em}
\end{figure*}

\highlight{%
\paragraph{Framework (Fig.~\ref{fig:framework})} 
We output a {\em two-level controller} for an agent, consisting of a collection of low-level policies $\Pi$ and a high-level {\em planner} $\tau$. When the agent enters a room corresponding to a vertex $v$ of the map, the planner chooses an outgoing edge $e$ and deploys the associated policy $\pi_{v, e} \in \Pi$. The agent follows $\pi_{v, e}$ until it exits the room. For example, $e$ can model a door between two rooms. Note that the agent may exit from direction $e' \neq e$. It is thus key to have an estimate of the success probability of $\pi_{v, e}$ when designing~$\tau$.

We obtain low-level policies by developing a novel \highlightt{RL} procedure that is run locally in each room $v$ and outputs {\em latent policies} $\pi_{v, e}$, for each direction $e$. These policies are represented on a concise model of the room (Fig.~\ref{fig:intro-1}).
Again, we only assume simulation access to the rooms;  {\em the latent policies are learned} and 
paired with probably approximately correct (PAC) performance guarantees.

Finally, given a map, a collection of policies $\Pi$, and a high-level specification $\varphi$ given as a logical formula over the map, we design an algorithm to find a planner $\tau$ that optimizes for $\varphi$ while lifting the guarantees on the policies in $\Pi$ to $\tau$ (Fig.~\ref{fig:intro-2}).
}%

\paragraph{Advantages.} We point to the advantages of the framework. First and foremost, it \emph{provides guarantees on the operation of the controller}. A key design objective is to ease the engineering burden: reward engineering is only done locally (for each room), and the high-level map and tasks are given in an intuitive specification language. Second, our framework enables \emph{reusability}: a policy~$\pi_{v, e}$, including its guarantees, is reusable across similar rooms $v'$ and when the high-level task or structure changes. Finally, our framework offers a remedy for the notorious challenge of sparse rewards in RL.

\paragraph{Case study.}
We complement our theoretical results with
\highlight{%
illustrations of feasibility in two case studies,
}%
where an agent must reach a distant location while avoiding mobile adversarial obstacles with stochastic dynamics. 
\highlight{%
The first case study is a grid world; the second case study is a vision-based Doom environment~\cite{DBLP:conf/cig/KempkaWRTJ16}.
}
DQN~\citep{DBLP:journals/nature/MnihKSRVBGRFOPB15} struggles to find a policy in our domain, even with reward shaping. 
In the rooms, we demonstrate our novel procedure for training concise latent policies directly. 
We synthesize a planner based on the latent policies and show the following results.
First, our two-level controller achieves high success probability, demonstrating that our approach overcomes the challenge of sparse rewards.
Second, the values predicted in the latent model are close to those observed, demonstrating the quality of our automatically constructed model.

\highlight{%
\paragraph{Contributions.}
We outline our key theoretical contributions.
\begin{enumerate}[(i)]
\item \emph{Learning guarantees for low-level policies.} 
We tie between the {\em values} (the 
probability that the low-level objective is satisfied)
of the latent model and that of the environment via a loss function (Thm.~\ref{thm:initial-value-bound}) and demonstrate that PAC bounds can be computed for these value differences  (Thm.~\ref{thm:init-values-pac-learnable}).
\item \emph{Guarantees on the synthesized controller.} We prove memory bounds on the size of an optimal high-level planner (Thm.~\ref{thm:mdp-planner-equiv}). Moreover, we show that an optimal planner can be obtained by solving an MDP whose size is proportional to the size of the map, i.e., disregarding the size of the rooms (Thm.~\ref{thm:succint-mdp-plan}).
\item \emph{Unified learning and synthesis guarantees.} We show that the learning guarantees for the low-level policies can be lifted to the two-level controller.
Specifically, minimizing the loss function to learn an abstraction of each room independently (and in parallel) guarantees that the values obtained under the two-level controller in the abstraction closely match those obtained in the true two-level environment (Thm.~\ref{thm:lifting-guarantees}).
\end{enumerate}
}%

\smallparagraph{Related work.}
\highlight{%
Hierarchical RL (HRL)~\cite{DBLP:journals/csur/PateriaSTQ21} (see also the {\em option} framework~\cite{SuttonPS99}) is an approach that outputs two-level controllers. 
Our approach is very different despite similarity in terms of outputs and motivation (e.g., both enable reusability and modularity). The most significant difference is that \emph{our framework provides guarantees}, which \emph{HRL generally lacks}. In our framework, high-level planners are synthesized based on prior knowledge of the environment (the map) and only after the low-level policies are learned. In HRL, both low-level policies and two-level controllers can be learned concurrently and with no prior knowledge.
Another difference is that in HRL, the low-level objectives generally need to be learned, whereas in our approach they are known.
We argue that the ``separation of concerns'' in our framework eases the engineering burden while HRL notoriously requires significant engineering efforts.
\highlightt{
Finally, unlike option-inspired approaches, where the integration of high- and low-level components results in a ``semi-''Markovian process, our framework ensures that a small amount of memory for the high-level planner is sufficient to enable the agent to operate within a \emph{fully} Markovian process. This facilitates the design of planning and synthesis solutions at the highest level of the environment.
}%

{\em Distillation}~\cite{HVD15} is an established approach: a neural network (NN) is trained then {\em distilled} into a concise {\em latent} model.
Verification of NN controllers is challenging, e.g.,~\cite{ASK21}. Verification-based distillation is a popular approach in which verification is applied to a latent policy, e.g.,~\cite{ernst2005tree,DBLP:conf/aaai/DelgrangeN022,BPS18,BGP21,CJT21}. 
In contrast, {\em we study controller-synthesis based on latent policies.} To our knowledge, only~\cite{AA+20} develops a synthesis based on distillation approach, but with no guarantees. 
In addition, {\em we develop a novel training procedure that trains a latent policy directly and circumvents the need for model distillation}. 
We stress that the abstraction is learned unlike~\cite{RoderickGT18,JothimuruganBA21}.

CLAPS~\cite{vzikelic2023compositional} is a recent approach that outputs a two-level controller with correctness guarantees. 
The technique to obtain guarantees is very different from ours. Low-level policies are trained using~\cite{ZLHC23}, which accompanies a policy with a super-martingale on the environment states that gives rise to reach-avoid guarantees. 
On the other hand, our policies are given on a learned latent model, which we accompany with PAC guarantees on the quality of the abstraction. 
We point to further differences: they assume prior knowledge of the transitions whereas we only assume simulation access, their policy is limited to be stationary and deterministic whereas our policies are general, and their high-level structure arises from the logical specification whereas ours arises from the structure of the environment.

It is known that safety objectives in RL are intractable~\cite{AlurBBJ22}. 
{\em Shielding}~\cite{AlshiekhBEKNT18,KonighoferBEP22,BrorholtJLLS23,BrorholtLS25} circumvents the difficulty of ensuring safety during training by monitoring a policy and blocking unsafe actions. Shielding has been applied to low-level policies in a hierarchical controller~\cite{XiongAJ22}. The limitation of this approach is that interference with the trained policy might break its guarantees. 
LTL objectives add intractability~\cite{YangLC21} to the already complex hierarchical scenarios in RL~\cite{KulkarniNST16} and only allow for PAC guarantees if the MDP structure is known~\cite{FuT14}.
Reactive synthesis is applied in~\cite{PrakashEDSJ23} to obtain low-level controllers, but scalability is a shortcoming of synthesis. 
Approaches encouraging but not ensuring safety use constrained policy optimization~\cite{AchiamHTA17}, safe padding in small steps~\cite{HasanbeigAK20}, time-bounded safety~\cite{GiacobbeHKW21}, safety-augmented MDPs~\cite{SootlaCJWMWA22}, differentiable probabilistic logic~\cite{YangMRR23}, or distribution sampling~\cite{BadingsRAPPSJ23}.
}

\section{Preliminaries}
\label{sec:prelim}

\paragraph{Markov Decision Processes (MDPs).}
Let~$\distributions{\measurableset}$ denote the set of distributions on~$\measurableset$.
An {\em MDP} is a tuple $\mdp = \mdptuple$ with states~$\istates$, actions~$\actions$, transition function~$\transitionfn \colon\istates \times \actions \to \distributions{\istates}$, and initial distribution~$\mdpI \in \distributions{\istates}$.
\highlight{%
An agent interacts with $\mdp$ as follows.
At each step, the agent is in some state $\istate \in \istates$. It performs an action $\action \in \actions$ and subsequently goes to the next state according to the transition function: $\istate' \sim \probtransitions\fun{\sampledot \mid \istate, \action}$.
}
A
\emph{policy} $\policy \colon\istates \to \distributions{\actions}$ 
\highlight{%
prescribes which action to choose at each step and
}%
gives rise to a distribution over \emph{paths} of~$\mdp$, denoted by~$\pathdistribution{\mdp}{\policy}$. The probability of finite paths is defined inductively. Trivial paths~$\istate \in \istates$ have probability~$\pathdistribution{\mdp}{\policy}\fun{\istate} = \mdpI\fun{\istate}$.
Paths~$\mdppath = \istate_0, \istate_1, \ldots, \istate_n$ have probability
$\pathdistribution{\mdp}{\policy}\fun{\istate_0, \istate_1, \ldots, \istate_{n-1}} \cdot \expectedsymbol{\action \sim \policy\fun{\sampledot \mid \istate_{n - 1}}}\transitionfn\fun{\istate_n \mid\istate_{n-1},\action}$. 
\paragraph{\highlight{Limiting behaviors in MDPs}}
The \highlight{\emph{transient measure}~\cite{BK08}}
\[
\highlight{\mu_{\policy}^{n}}\fun{\istate' | \istate} =  \Prob_{\mdppath \sim \pathdistribution{\mdp}{\policy}}[\rho \in \set{\istate_0, \dots, \istate_n | \istate_n = \istate'} \mid \istate_0 = \istate ]
\]
gives the probability of visiting
each state $\istate'$ after exactly $n$ steps starting from~$\istate \in \istates$.
Under policy~$\policy$, $C \subseteq \istates$ is a \emph{bottom strongly connected component} (BSCC) of~$\mdp$ if (i)~$C$ is a maximal subset satisfying $\highlight{\mu_{\policy}^n}\fun{\istate'\mid \istate} > 0$  for any $\istate, \istate' \in C$ and some~$n \geq 0$, and (ii) $\expectedsymbol{\action \sim \policy\fun{\sampledot \mid \istate}} \probtransitions\fun{C \mid \istate, \action}=1$ for all~$\istate \in \istates$.
MDP~$\mdp$ is \emph{ergodic} if, under any stationary policy~$\policy$, the set of reachable states
\[
\highlight{\textit{Reach}\fun{\mdp, \policy}} = 
\set{\istate \in \istates \mid \exists n \geq 0, \expectedsymbol{\istate_0 \sim \mdpI}\, \highlight{\mu_{\policy}^{n}}\fun{\istate \mid \istate_0} > 0}
\vspace{-.25em}
\]
consist of a unique aperiodic BSCC.
Then, for~$\istate \in \istates$, the \emph{stationary distribution} of~$\mdp$ under~$\policy$ is given by $\stationary{\policy} = \lim_{n \to \infty} \highlight{\mu_{\policy}^n}\fun{\sampledot \mid \istate}$.

\paragraph{Objectives and values.}~%
A qualitative {\em objective} is a set of infinite paths~$\objective \subseteq \istates^\omega$.
For~$B, T \subseteq \istates$,
we consider {\em reach-avoid objectives} $\reachavoid{T}{B} = \set{s_0, s_1, \ldots \mid \exists i. \, \istate_i \in T \text{ and } \forall j \leq i, \, \istate_j \notin B}$
(or just~$\objective$ if clear from context) where the goal is to reach a ``target'' in~$T$ while avoiding the ``bad'' states~$B$.
Henceforth, fix a \emph{discount factor}
$\discount \in \mathopen(0, 1\mathclose)$. 
In this work, we consider \emph{discounted} value functions (see, e.g., ~\cite{DBLP:conf/icalp/AlfaroHM03}).
The {\em value} of any state $\istate \in \istates$ for policy~$\policy$ w.r.t.\ objective~$\objective$ is denoted by $\values{\policy}{\objective}{\istate}$ and corresponds to the probability of satisfying $\objective$ from state $\istate$ as $\discount$ goes to one, i.e., $\lim_{\discount \to 1} \values{\policy}{\objective}{\istate} = \Prob_{\mdppath \sim \pathdistribution{\mdp}{\policy}}\left[\mdppath \in \objective \mid \istate_0 = \istate \right]$.
In particular, for the reach-avoid objective~$\reachavoid{T}{B}$, 
$\values{\policy}{\objective}{\istate}$ corresponds to the 
discounted probability of visiting $T$ for the first time while avoiding $B$, i.e.,
$\values{\policy}{\objective}{\istate} = \expectedsymbol{\mdppath \sim \pathdistribution{\mdp}{\policy}}\left[  \sup_{i \geq 0} \discount^{i} \cdot \condition{\istate_i \in T \wedge \forall j \leq i,\,\istate_j \not\in B} \mid \istate_0 = \istate\right ]$, where $\istate_i$, $\istate_j$ are respectively the~${i^\text{th}}, {j^\text{th}}$ state of $\mdppath$.
We are particularly interested in the values obtained from the beginning of the execution, written $\values{\policy}{\objective}{} = \expected{\istate_0 \sim \mdpI}{\values{\policy}{\objective}{\istate_0}}$. 
We may sometimes omit~$\objective$ and simply write $\values{\policy}{}{none}$ and $\values{\policy}{}{}$.

\emph{Reinforcement learning}~obtains a policy in a model-free way.
Executing action~$\action_i$ in state~$\istate_i$ and transitioning to~$\istate_{i + 1}$ incurs a 
reward $r_{i} = \reward(\istate_i, \action_i, \istate_{i + 1}),$ computed via a \emph{reward function} $\reward \colon\istates \times \actions \times \istates \to \Reals$.
An RL agent's goal is to learn a policy~$\policy^*$ maximizing the return $\expected{\mdppath \sim \Pr^{\mdp}_{\policy^*}}{\sum_{i \geq 0}\discount^i r_i}$.
The agent is trained by interacting with the environment in episodic simulations, each ending in one of three ways: success, failure, or an eventual reset.

\begin{figure*}[t]
    \begin{subfigure}[c]{.65\linewidth}
        \centering
        \includegraphics[width=\linewidth,height=43mm,keepaspectratio]{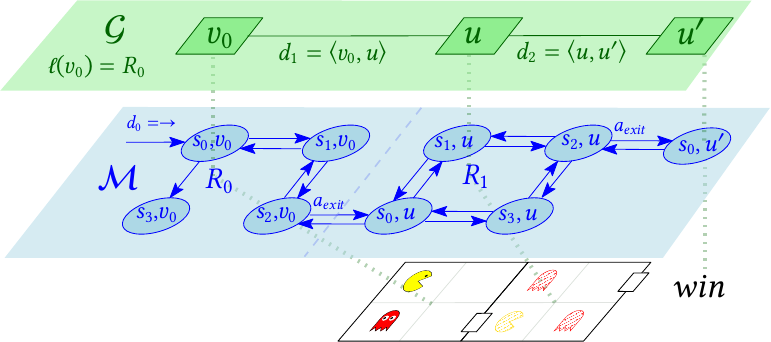}
        \caption{A two-level model of a simple grid-world environment.\\$\,$}
        \label{fig:example-hierarchy}
    \end{subfigure}
    \hfill
    \begin{subfigure}[c]{.345\linewidth}
        \centering
        \includegraphics[width=\linewidth,height=30mm,keepaspectratio]{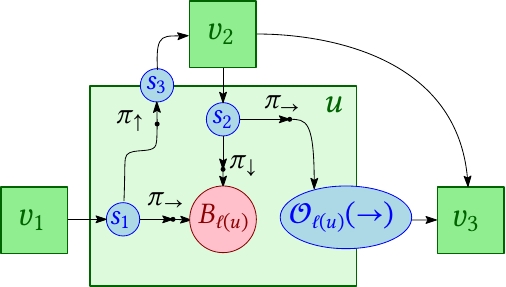}
        \caption{A two-level model for which an optimal planner requires memory, here flattened in $2$D.}
        \label{fig:example-controller}
    \end{subfigure}
        \vspace{-1.5em}
        \caption{(a)~Top:~The high-level graph~$\color{graphcolor}{\graph}$ with two rooms $\color{mdpcolor}{\room_0} = \labelingfn(\vertex_0)$ and $\color{mdpcolor}{\room_1} = \labelingfn(\varvertex)$. 
    Middle:~Part of the explicit MDP for the bottom layer; e.g., the MDP $\color{mdpcolor}{\room_0}$ contains $16$ states.
    Traversing the edge $\tuple{\tuple{\istate_2,\vertex_0}, \tuple{\istate_0, \varvertex}}$ corresponds to exiting $\color{mdpcolor}{\room_0}$ and entering $\color{mdpcolor}{\room_1}$ from direction $d_1 = \tuple{\vertex_0, \varvertex}$. 
The goal of~\protect\pacman \ is to reach $\varvertex'$ by exiting the room $\color{mdpcolor}{\room_1}$ from direction $d_2 = \tuple{\varvertex, \varvertex'}$ while avoiding the moving adversaries~\protect\ghost{red}. For $i \in \set{0,1}$, the entrance function~$\color{mdpcolor}{\entrancefn_{\room_i}}$ models the distribution from which the initial location of~\protect\ghost{red} in $\room_i$ is drawn.
    (b)~A room with four policies for a planner to choose from; e.g., $\pi_\rightarrow(\sampledot \mid s_1)$ leads to $B_{\labelingfn(\varvertex)}$ and $\pi_\uparrow(\sampledot \mid s_1)$ leads to $s_3$. Note that, while these are deterministic policies, in general, the policies in rooms are probabilistic. 
    }
    \Description{Two figures describing how a high-level map yields a Markov decision process. The first one shows the abstract map view on top of actual states and transitions of the MDP to show how each vertex of the map gives rise to multiple states based on the direction in which the room is entered. The second picture shows a very simple situation in which playing optimally requires using memory to recall in which direction the room was entered.}
\end{figure*}

\section{Problem Formulation}\label{sec:problem}
In this section, we 
formally model a two-level environment and state the problem of two-level controller synthesis.
The environment MDP is
a high-level \emph{map}: an undirected graph whose vertices are associated with ``low-level'' MDPs called {\em rooms} (Fig.~\ref{fig:example-hierarchy}). A two-level controller works
as follows. In each room, we assume access to a set of {\em low-level policies}, each optimizing a local (room) reach-avoid objective (Fig.~\ref{fig:example-controller}). When transitioning to a new room, a high-level {\em planner} selects the next low-level policy.

\paragraph{Two-level model.}~%
A \emph{room} $\room = \tuple{\istates_\room, \actions_\room, \transitionfn_\room, \directions_\room, \entrancefn_\room, \exitfn_\room}$ consists of~$\istates_\room$, $\actions_\room$, $\transitionfn_\room$ as in an MDP, a set of {\em directions}~$\directions_\room$, an {\em entrance function} $\entrancefn_\room \colon \directions_\room \rightarrow \distributions{\istates_\room}$ taking a direction from which the room is entered and producing an initial distribution over states, and an {\em exit function} $\exitfn_\room \colon \directions_\room \rightarrow 2^{\istates_\room}$ returning a set of \emph{exit states} from the room in a given direction $\direction \in \directions_\room$.
States are assigned to at most one exit, i.e., if $\istate \in \exitfn_{\room}\fun{\direction}$ and $\istate \in \exitfn_{\room}\fun{\direction'}$, then $\direction' = \direction$.

    \begin{figure}[H]
    \vspace{-0.7em}
        \centering
        \includegraphics[width=.56\linewidth]{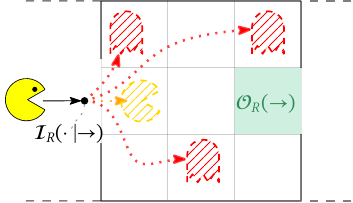}
        \vspace{-.5em}
        \caption{Small room in a grid world.}\label{fig:grid-room}
    \vspace{-1em}
    \Description{The figure shows a pacman entering a room from the left toward the right direction. The random location of a red ghost is shown too, as well as the desired exit through the right side.}
    \end{figure}
\begin{example}[Room]\label{example:room}
    Consider the grid world of Fig.~\ref{fig:grid-room} as a room $\room$ populated by an adversary \ghost{red}\ .
    One can encode the position of \ghost{red}\ in $\istates_{\room}$ and its behaviors through $\transitionfn_{\room}$.
    \highlight{%
    This can be achieved by, e.g., considering states of the form $\istate = \tuple{\fun{x_1, y_1}, \fun{x_2, y_2}} \in \istates_{\room}$ where $\fun{x_1, y_1}$ is the position of \pacman\ and $\fun{x_2, y_2}$ the one of \ghost{red}\ in the grid.
    }
    Note that the position of \ghost{red}\, depends on the direction from which the agent \pacman\ enters $\room$.
    The agent enters from the left in direction $\rightarrow$ to the states of $\room$ distributed according to the entrance function 
    {\raggedright $\entrancefn_{\room}\fun{\sampledot \mid \direction = \rightarrow}$} (the tiling patterns highlight its support). Precisely, while the agent \pacman\, enters (in a deterministic way) in the leftmost cell (yellow tiling), $\entrancefn_{\room}$ allows to (probabilistically) model the possible positions of \ghost{red}\, when entering the room (red tiling) from direction $d = \rightarrow$.
    When reaching the green area, depicting states from $\exitfn_{\room}\fun{\rightarrow}$, \pacman\, exits $\room$ by the right direction $\rightarrow$.
\end{example}

A \emph{map} is a graph $\graph = \tuple{\vertices, \edges}$ with vertices $\vertices$ and undirected edges $\edges \subseteq \vertices \times \vertices$.
The {\em neighbors} of $\vertex \in \vertices$ are $\neighbors{\vertex} = \{\varvertex \in \vertices \mid \allowbreak\tuple{\varvertex,\vertex} \in \edges\}$ and the outgoing edges from $v$ are $\out{v} = \set{e = \zug{v, u} \in \edges}$.
A {\em two-level model} $\HM = \zug{\graph, \rooms, \labelingfn, v_0, \tuple{d_0, d_1}}$ consists of a map~$\graph = \zug{\vertices, \edges}$, a set of rooms~$\rooms$, a labeling $\labelingfn \colon\vertices \rightarrow \rooms$ of each vertex~$\vertex \in \vertices$ with a room~$\labelingfn\fun{\vertex}$ and directions~$\directions_{\labelingfn\fun{\vertex}} = out(\vertex)$, an initial room~$v_0 \in \vertices$, 
and directions~$d_0, d_1 \in \out{\vertex_0}$ in which~$v_0$ is respectively entered and must be exited.

Fix a two-level model~$\hmdp = \zug{\graph, \rooms, \labelingfn, \allowbreak v_0, \allowbreak \tuple{d_0, d_1}}$.
Intuitively, the \emph{explicit MDP}~$\M$ corresponding to~$\hmdp$ is obtained by ``stitching'' MDPs~$\room\in\rooms$ corresponding to neighboring rooms (Fig.~\ref{fig:example-hierarchy}). Formally, $\mdp = \mdptuple$,
where
$\istates = \set{\tuple{\istate, \vertex} \colon \istate \in \istates_{\labelingfn\fun{\vertex}}, \vertex \in \vertices}$, $\actions = \bigcup_{\room \in \rooms} \actions_{\room} \cup \set{\action_{\mathit{exit}}}$.
The initial distribution~$\mdpI$ simulates starting in room~$\labelingfn(v_0)$ from direction $d_0$; thus, for each~$\istate \in \istates_{\labelingfn(v_0)}$, $\mdpI(\zug{\istate, v_0}) = \entrancefn_{\labelingfn(v_0)}(\istate \mid d_0)$.
The transitions $\transitionfn$ coincide with $\transitionfn_\room$ for non-exit states.
Let $\direction = \tuple{\vertex, \varvertex} \in \edges$ with $\vertex \in \neighbors{\varvertex}$; 
$\exitfn_{\room}\fun{\direction}$
are the exit states in room~$\room$ associated with $\vertex$ in direction $\direction$, and $\entrancefn_{\labelingfn\fun{\varvertex}}\fun{\sampledot \mid \direction}$ is the entrance distribution in~$\room$ associated with $\varvertex$ in direction $\direction$.
The successor state of $\istate \in \exitfn_{\room}\fun{\direction}$ follows~$\entrancefn_{\labelingfn\fun{\varvertex}}\fun{\sampledot \mid \direction}$ when $a_\mathit{exit}$ is chosen.
Each path $\mdppath$ in~$\mdp$ corresponds to a unique $\text{path}(\mdppath)$ in~$\graph$ traversing the rooms.

\paragraph{High-level reach and low-level reach-avoid objectives.}
The high-level reachability objective we consider is \highlight{ ``$\reach{T}$,''} where~$T \subseteq \vertices$ is a subset of vertices
in the graph of~$\hmdp$.
\highlight{%
Here, $\reach{T}$ is a temporal logic notation meaning “eventually visit the set $T$.”
}%
Formally, a path~$\mdppath$ in~$\M$ satisfies~$\reach{T}$ iff~$\text{path}(\mdppath)$ visits a vertex~$\vertex$ in~$T$.
The low-level safety objective is defined over states of the rooms in~$\rooms$. For each room~$\room$, let~$B_\room \subseteq \istates_\room$ be a set of ``bad'' states.
For room~$\room$ and direction~$d \in \directions_\room$, the reach-avoid objective $\objective^d_\room \in \istates_\room^*$ is $\{\istate_0, \dots, \istate_n \mid \istate_n \in \exitfn_\room(d) \allowbreak \linebreak[0]\text{ and } \istate_i \notin B_\room \text{ for all } i \leq n \}$, i.e., exit~$\room$ via~$d$ avoiding~$B_\room$.

\paragraph{High-level control.}~%
We define a \emph{high-level planner}~$\tau \colon\vertices^* \rightarrow E$ and a set of low-level policies~$\planner$ such that, for each room~$\room \in \rooms$ and a direction~$d \in \directions_\room$, $\planner$ contains a policy~$\policy_{\room, d}$ for the objective~$\objective_{\room}^d$.
The pair~$\policy = \zug{\tau, \Pi}$ is a {\em two-level controller} for~$\hmdp$,
defined inductively as follows. Consider the initial vertex~$\vertex_0 \in \vertices$.
First, the planner always chooses $\tau(v_0) = d_1$, thus
control in~$\labelingfn(\vertex_0)$ follows $\policy_{\labelingfn(\vertex_0), \direction_1}$.
Then, let~$\mdppath$ be a path in~$\hmdp$ ending in~$\istate \in \istates_\room$, for some room~$\room  = \labelingfn\fun{\vertex}$. 
If~$s$ is not an exit state of~$\room$, then control follows a policy~$\policy_{\room, \direction}$ with~$\direction = \tuple{\vertex, \varvertex}$ and~$\varvertex \in \neighbors{\vertex}$. 
If~$s$ is an exit state in direction~$\direction$ and~$\text{path}(\mdppath)$ ends in~$v$, i.e.,
$s \in \exitfn_\room(\direction)$,
then~$a_\mathit{exit}$ is taken in~$s$ and the next state is an initial state in~$\room' = \labelingfn(u)$ drawn from~$\entrancefn_{\room'}(\sampledot\mid\direction)$. 
The planner chooses a direction~$d' = \tau(\text{path}(\mdppath) \cdot u) \in \out{u}$ to exit~$\room'$.
Control of~$\room'$ proceeds with the low-level policy~$\policy_{\room', d'}$.
Note that~$\policy$ is a policy in the explicit MDP~$\mdp$.

\begin{problem}
Given
a two-level model~$\hmdp = \zug{\graph, \allowbreak \rooms, \allowbreak \labelingfn, v_0,\allowbreak \tuple{d_0, d_1}}$, discount factor~$\discount \in \mathopen(0, 1\mathclose)$,
high-level objective $\reach T$,
and low-level objectives $\{\objective^{\direction}_{\room} \mid \room \in \rooms, \direction \in \directions_{\room}\}$,
construct a two-level controller~$\policy = \zug{\tau, \Pi}$ maximizing the probability of satisfying the objectives.
\end{problem}

\section{Obtaining Low-Level RL policies}\label{sec:low_level}
There are
challenges in reasoning about RL policies%
\highlight{---especially those obtained via DRL, which are typically represented by large NNs}.
We develop a novel, unified \highlight{approach}
which outputs a latent model together with a concise policy.
\highlight{%
The idea is to learn a \emph{tractable} latent model for each room, where the values of the low-level objectives can be explicitly computed.
Each latent model is accompanied by \emph{probably approximately correct} (PAC) guarantees on their abstraction quality.
}%
We first focus on those guarantees.
\highlight{%
In the next section, we will then focus on how to synthesize a planner (with guarantees) based on these learned models and policies.
}%

\subsection{Quantifying the quality of the abstraction}\label{sec:quality-abstr} 
In this section, we fix an MDP environment $\mdp=\mdptuple$. 
A {\em latent model} abstracts a concrete MDP and is itself an MDP $\latentmdp = \latentmdptuple$ whose state space is linked to $\mdp$ via a \emph{state-embedding function} $\embed\colon \istates \to \latentstates$. We focus on latent MDPs with a finite state space, \highlight{ where values can be exactly computed}.

Let $\latentpolicy$ be a policy in $\latentmdp$, called a {\em latent policy}.
The key feature is that $\embed$ allows to control $\mdp$ using $\latentpolicy$: for each state $\istate \in \istates$, let $\latentpolicy\fun{\sampledot \mid \istate}$ in $\mdp$ follow the distribution $\latentpolicy\fun{\sampledot \mid \embed\fun{\istate}}$ in $\latentmdp$. Abusing notation, we refer to $\latentpolicy$ as a policy in $\mdp$.
We write $\latentvalues{\latentpolicy}{}{none}$ for the value function of $\latentmdp$ operating under $\latentpolicy$. 

Given $\latentmdp$ and $\latentpolicy$, we bound the difference between $\values{\latentpolicy}{}{none}$ and $\latentvalues{\latentpolicy}{}{none}$; the smaller the difference, the more accurately $\latentmdp$ abstracts $\mdp$. Computing $\valuefn{\latentpolicy}{}{\gamma}{none}$ is intractable. To overcome this, in the same spirit as \cite{DBLP:conf/icml/GeladaKBNB19,DBLP:conf/aaai/DelgrangeN022},
we define a local measure 
on the transitions of $\mdp$ and $\latentmdp$ to bound the difference between the values obtained under $\latentpolicy$ 
(cf.~Fig.~\ref{fig:distance-abstraction}).
We define the \emph{transition loss} $L_{\probtransitions}^{\latentpolicy}$ w.r.t.\ a distance metric $\divergence$ on distributions over $\latentstates$. We focus on the \emph{total variation distance} (TV) $\divergence(P, P') = \nicefrac{1}{2}\norm{P - P'}_1$ for $P,P' \in \distributions{\latentstates}$.
We compute~$L_{\probtransitions}^{\latentpolicy}$ by taking the expectation according to the stationary distribution~$\stationary{\latentpolicy}$:
\begin{equation}\label{eq:transition-loss}
    L_{\probtransitions}^{ \latentpolicy} = 
    \expectedsymbol{\istate \sim \stationary{\latentpolicy}, \action \sim \latentpolicy(\sampledot \mid \istate)}\,
    \divergence\fun{\embed\probtransitions\fun{\sampledot \mid \istate, \latentaction}, \latentprobtransitions\fun{\sampledot \mid \embed\fun{\istate}, \latentaction}}.
\end{equation}
The superscript is omitted when clear from the context.
Efficiently sampling from the stationary distribution can be done via randomized algorithms, even for unknown probabilities~\cite{DBLP:journals/combinatorics/LovaszW95,DBLP:journals/jal/ProppW98}.

\begin{figure}[b]
    \vspace{-.5em}
    \centering
    \begin{minipage}{.4\linewidth}
    \begin{tikzpicture}[t/.style={->,thick}]
    \draw[fill=mdpcolor!30!white, draw opacity=0, fill opacity=0.45] (-0.3,-2.75) rectangle (1.4,0.65);
    \draw[fill=orange!30!white, draw opacity=0, fill opacity=0.6] (1.4,-2.75) rectangle (3,0.65);
    \node (s) {$s$};
    \node[right=22.5mm of s] (sb) {$\latentstate$};
    \node[below=7.5mm of s] (a) {$\action$};
    \node[at=(sb|-a)] (ab) {$\action$};
    \node[below=7.5mm of a] (s') {$\istate'$};
    \node[at=(ab|-s')] (ss) {$\latentstate_2$};
    \node[left=4mm of ss] (sss) {$\latentstate_1$};
    \node[at=(ss)] (phantomss) {};
    \node[at=(sss)] (phantomsss) {};
    \node[above=0mm of s,xshift=4.5mm] (M) {$\mdp$};
    \node[above=0mm of sb,xshift=-4.5mm] (Mb) {$\latentmdp$};
    \draw[t] (s) to node[above]{$\embed$} (sb);
    \draw[t] (sb) to node[left]{$\latentpolicy$} (ab);
    \draw[dashed, t] (ab) to node{} (a);
    \draw[t] (a) to node[right]{$\transitionfn$} (s');
    \draw[t] (ab) to node[left]{$\latentprobtransitions$} (ss);
    \draw[t] (s') to node[above]{$\embed$} (sss);
    \draw[dotted, very thick, red] (phantomsss) -- (phantomss);
\end{tikzpicture}
    \end{minipage}
    \hfill
    \begin{minipage}{.55\linewidth}
    \vspace{-.75em}
    \caption{
    {To run $\latentpolicy$ in the original environment $\mdp$, (i)~map $\istate$ to $\embed\fun{\istate} = \latentstate$,
    (ii)~draw $\action \sim \latentpolicy\fun{\sampledot \mid \latentstate}$.
    $\transitionloss$ measures {the gap} (in {\color{red}red}) between latent states produced via $\latentstate_1 = \embed\fun{\istate'}$ with $\istate' \sim \probtransitions\fun{\sampledot \mid \istate, \action}$ (shortened as $\latentstate_1 \sim \embed\probtransitions\fun{\sampledot \mid \istate, \action}$) and those produced directly in the latent space: $\latentstate_2\sim \latentprobtransitions\fun{\sampledot \mid \latentstate, \action}$.}
    }
    \label{fig:distance-abstraction}
    \end{minipage}
    \Description{The figure shows a schematic of how a latent policy gives rise to a concrete policy by taking the state in the original environment and mapping it to a latent state, then feeding that to the latent policy to obtain a concrete action that can be used in the latent and concrete environments.}
\end{figure}

Recall that RL is episodic, terminating when the objective is satisfied/violated or via a reset. We thus restrict $\mdp$ to an {\em episodic process}, which implies ergodicity of both~$\mdp$ and $\latentmdp$ under mild conditions (cf.~\cite{DBLP:conf/nips/Huang20} for a discussion).
\begin{assumption}[Episodic process]\label{assumption:episodic}
The environment~$\mdp$ has a reset state $\sreset$ such that (i)~$\sreset$ is almost surely visited under any policy, and (ii)~$\mdp$ follows the initial distribution once reset: $\probtransitions\fun{\sampledot \mid \sreset, \action} = \mdpI$ for any $\action \in \actions$. The latent model~$\latentmdp$ is also episodic with reset state~$\embed\fun{\sreset}$.
\end{assumption}

\begin{assumption}\label{assumption:original-vs-latent}
The abstraction preserves information regarding the objectives. Formally, let $\tuple{T, \overline{T}}$, $\tuple{B, \overline{B}} \subseteq \istates \times \latentstates$ be sets of \emph{target} and \emph{bad} states, respectively.
Then, for $\measurableset \in \set{T, B}$, $\istate \in \measurableset$ iff $\embed\fun{s} \in \overline{\measurableset}$.%
    \footnote{By labeling states with atomic propositions, a standard in model checking~\cite{DBLP:conf/aaai/DelgrangeN022}.}
    We consider the objective $\objective\fun{T, B}$ in $\mdp$ and $\objective\fun{\overline{T}, \overline{B}}$ in $\latentmdp$.
\end{assumption}
{\raggedleft The following lemma establishes a bound on the difference in values based on $\transitionloss$. Notably, as $\transitionloss$ goes to zero, the two models \emph{almost surely} have the same values from every state.}
\begin{lemma}[\cite{DBLP:conf/aaai/DelgrangeN022}]\label{theorem:value-diff-bound}
    Let $\latentpolicy$ be a latent policy and $\stationary{\latentpolicy}$ be the unique stationary measure of $\mdp$, then
    \emph{the average value difference} is bounded by $\transitionloss$:
    $\expectedsymbol{\istate \sim \stationary{\latentpolicy}} \left|\values{\latentpolicy}{}{\istate} - \latentvalues{\latentpolicy}{}{\embed\fun{\istate}} \right| \leq \ \frac{\discount\localtransitionloss{\stationary{\latentpolicy}}}{1 - \discount}.$
\end{lemma}

The next theorem provides a
bound applicable to the initial distribution, removing the need of the expectation in Lem.~\ref{theorem:value-diff-bound}.
It follows from plugging the stationary distribution in $\sreset$
into Lem.~\ref{theorem:value-diff-bound} and observing that $\nicefrac{1}{\stationary{\latentpolicy}\fun{\sreset}}$ is \emph{the average episode length}
\cite{serfozo2009basics}.

\begin{theorem}\label{thm:initial-value-bound}    
The value difference \emph{from the initial states} is bounded by $\transitionloss$:
      $\left| \values{\latentpolicy}{}{}{} -
      \latentvalues{\latentpolicy}{}{}{} \right| \leq \frac{\localtransitionloss{\stationary{\latentpolicy}}}{\stationary{\latentpolicy}\fun{\sreset}\fun{1 {-} \discount }}.$
\end{theorem}

\subsection{PAC estimates of the abstraction quality}\label{sec:pac}
Thm.~\ref{thm:initial-value-bound} establishes a bound on the quality of the abstraction based on $\localtransitionloss{}$ and $\stationary{\latentpolicy}\fun{\sinit}$. Computing these quantities
is not possible in practice since the transition probabilities of $\mdp$ are unknown.

Instead, we obtain PAC bounds on $\stationary{\latentpolicy}\fun{\sinit}$ and $\localtransitionloss{}$ by simulating $\mdp$. The estimate of~$\stationary{\latentpolicy}\fun{\sinit}$ is obtained by taking the portion of visits to $\sinit$ in a simulation and Hoeffding's inequality.
The estimate of~$\localtransitionloss{}$ is obtained as follows. 
When the simulation goes from~$s$ to~$s'$
following action~$a$, we add a ``reward'' of $\latentprobtransitions\fun{\embed\fun{\istate'} \mid \embed\fun{\istate}, \action}$. 
Since~$\transitionloss{}$ is a loss, we subtract the average reward from~$1$.

\begin{lemma}\label{lem:pac-bounds}
    Let $\set{\tuple{\istate_t, \action_t, \istate'_{t}} \colon 1 \leq t \leq \T}$ be a set of $\T$ transitions drawn from $\stationary{\latentpolicy}$ by simulating $\mdp_{\latentpolicy}$.
    Let
    \hypertarget{LP}{
    \[\localtransitionlossapprox{} \!=\! 1-\frac{1}{\T} \sum_{t = 1}^{\T} \latentprobtransitions\fun{\embed\fun{\istate'_{t}} \mid \embed\fun{\istate_t}, \action_t}
        \text{ and }\stationaryapprox \!=\! \frac{1}{\T} \sum_{t = 0}^{\T} \condition{\istate_t = \sinit}\!.\]
    }
    Then, for all~$\error, \proberror > 0$ and $\T \geq \left\lceil \nicefrac{-\log\fun{\varerror}}{2 \error^2}\right\rceil$, with at least probability $1 - \proberror$ we have that
    \begin{enumerate}[(i)]
        \item if $\varerror \leq \delta$,
        $\localtransitionlossapprox{} + \error > \localtransitionloss{}$,
        \item if $\varerror \leq \nicefrac{\delta}{2}$,
        $\localtransitionlossapprox{} + \error > \localtransitionloss{}$ and ${\stationary{\latentpolicy}\fun{\sinit}} >{\stationaryapprox} - \error$.
    \end{enumerate}
\end{lemma}

The following theorem has two key implications:
\ref{enum-thm-PAC:i}~it establishes a lower bound on the minimum number of samples necessary to calculate the PAC upper bound for the average value difference;
\ref{enum-thm-PAC:ii}~it suggests an online algorithm with a termination criterion for the value difference bound \highlight{obtained from the initial states}.
\begin{figure*}[t]
    \centering
    \includegraphics[width=\linewidth]{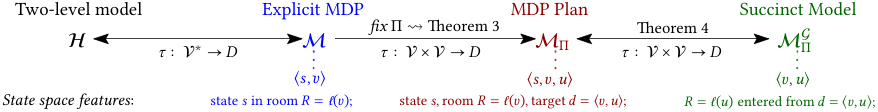}
    \vspace{-2em}
    \caption{Chain of reductions for synthesizing a planner $\controller$ in a two-level model $\hmdp$.
    $\hmdp$ can be formulated as an explicit MDP $\mdp$.
    Once the low-level policies $\planner$ are learned (Fig.~\ref{fig:intro-1}), the synthesis problem reduces to constructing a stationary policy in an \emph{MDP plan} $\mdp_{\planner}$ where $\planner$ is fixed and the state space of $\mdp_{\planner}$ encodes the directions chosen in each room.
    From this policy, one can derive a $\abs{\vertices}$-memory planner $\controller$ for $\hmdp$  (Thm.~\ref{thm:mdp-planner-equiv}).
    Finally, finding a policy in $\mdp_{\planner}$ is equivalent to finding a policy in a \emph{succinct model} $\mdp^{\graph}_{\planner}$ where (i) the state space corresponds to the directions from which rooms are entered, (ii) the actions to the choices of the planner, and (iii) the transition probabilities to the values achieved by the latent policy chosen (Thm.~\ref{thm:succint-mdp-plan}).
    }\label{fig:synthesis-overview}
    \Description{The figure shows the chain of reductions described in the work to realize synthesis of a planner in a two-level model. In short, the model becomes an explicit MDP which yields an MDP plan and in turn a finite-memory planner for the original model. Finally, this is the same as finding a policy in the succinct model where states encode directions, action choices of the planner, and transition probabilities to the values achieved by the latent policy chosen.}
\end{figure*}
\begin{theorem}[The value bounds are PAC learnable]
\label{thm:init-values-pac-learnable}
    Consider $\T$ transitions
    $\set{\tuple{\istate_t, \action_t, \istate'_{t}} \colon 1 \leq t \leq \T}$ drawn from $\stationary{\latentpolicy}$ by simulating $\mdp$ under ${\latentpolicy}$.
    Then, for any $\error, \proberror > 0$, $\T \geq \left\lceil \nicefrac{- \discount' \log(\delta')}{\fun{2 \error^2(1 - \discount)^2 \varerror}}\right\rceil$, with at least probability $1 - \proberror$, \highlight{the following value bounds hold, on}
    \begin{enumerate}[(i)]
        \item \label{enum-thm-PAC:i} \highlight{\emph{the \textbf{average} value gap}}:
        $
        \expectedsymbol{\istate \sim \stationary{\latentpolicy}} \left|\values{\latentpolicy}{}{\istate} - \latentvalues{\latentpolicy}{}{\embed\fun{\istate}} \right| \leq \frac{\discount\localtransitionlossapprox{}}{1 - \discount} + \error
        $
        with $\proberror' = \proberror$, $\discount' = \discount^2$, and $\varerror = 1$, and
    \item \label{enum-thm-PAC:ii} \emph{\highlight{the value gap \textbf{from the initial states}:}}\\
        \vspace{-.75em}
        \[\abs{\values{\latentpolicy}{}{} - \latentvalues{\latentpolicy}{}{}} \leq \frac{ \localtransitionlossapprox{} }{\stationaryapprox \, \fun{1 - \discount}} + \error\]
        with $\delta' = \nicefrac{\delta}{2}$, $\discount' = ( \localtransitionlossapprox{} + \stationaryapprox\fun{1 + \error\fun{1 - \discount}})^2$, and $\varerror = \stationaryapprox^{\, 4}$.
    \end{enumerate}
\end{theorem}
Unlike~\ref{enum-thm-PAC:i}, which enables precomputing the number of samples to estimate the bound, \ref{enum-thm-PAC:ii} allows estimating
with an 
algorithm, almost surely terminating but without predetermined endpoint since $\T$ relies in that case on the current approximations of $\transitionloss$ and $\stationary{\latentpolicy}$.

\subsection{Obtaining latent policies during training}\label{sec:wae-dqn}
\highlight{%
As highlighted in the last section, our guarantees rely on learning a policy on the representation induced by a suitable, latent abstraction.
Accordingly, we propose
}%
a DRL procedure that trains the policy and the latent model \emph{simultaneously}. Previous approaches used a two-step process: train a policy $\policy$ in $\mdp$ and then \emph{distill} it.
In contrast, our one-step approach alternates between optimizing a latent policy~$\latentpolicy$ via DQN~\citep{DBLP:journals/nature/MnihKSRVBGRFOPB15} and representation learning through \emph{Wasserstein auto-encoded MDPs} (WAE-MDPs~\citep{delgrange2023wasserstein}). This process \emph{avoids the distillation step} by directly learning $\latentpolicy$ and minimizing $\transitionloss$. That way, the DQN policy is directly optimized on the learned latent space (cf.~Fig.~\ref{fig:intro-1}). We call this procedure \emph{WAE-DQN}.

The combination of these techniques is nontrivial and requires addressing stability issues.
To summarize, WAE-DQN ensures the following properties:
(i)~$\embed$ groups states with close values, supporting the learning of $\latentpolicy$;
(ii)~$\latentpolicy$ prescribes the same actions for states with close behaviors, improving robustness and enabling reuse of the latent space for rooms with similar structure.

\section{Obtaining a Planner}\label{sec:high_level}

Fix $\planner$ as a collection of low-level, latent policies.
In this section, we show that synthesizing a planner reduces to constructing a policy in a succinct model, where the action space coincides with the edges of the map $\graph$ (i.e., the choices of the planner).
In the following, we describe the chain of reductions leading to this result. An overview is given in Fig.~\ref{fig:synthesis-overview}.
We further discuss the memory requirements of the planner.
Precisely, we study the following problem:

\begin{problem}
Given a two-level model $\hmdp$,
 a collection of latent policies $\planner$, and an objective $\objective$, construct a planner $\tau$ such that the controller $\tuple{\tau, \planner}$ is optimal for $\objective$ in $\hmdp$.
\end{problem}

\begin{example}[Planners require memory]\label{ex:memory-requirements}
Consider again
Fig.~\ref{fig:example-controller}.
To reach~$\vertex_3$ and avoid~$B_{\labelingfn\fun{\varvertex}}$ from $\varvertex$,
$\controller$ must remember from where the room $\labelingfn(\varvertex)$ is entered:
$\controller$ must choose~$\uparrow$ from~$\vertex_1$, and~$\rightarrow$ from~$\vertex_2$.
\end{example}

Next, we establish a memory bound for an optimal planner. Upon entering a room $\room \in \rooms$, the planner selects a direction $d \in E$, so the policy operating in $\room$ is $\latentpolicy_{\room, d} \in \planner$, optimizing the objective $\objective_{\room}^{d}$ to exit $\room$ via $d$. We construct an \emph{MDP plan} $\mdpplanner{} = \tuple{\istatesplanner, \actionsplanner, \transitionfnplanner, \mdpIplanner}$ to simulate this interaction. A state $\istate^* = \tuple{\istate, \vertex, \varvertex} \in \istatesplanner$ represents $\hmdp$ being at vertex $\vertex$, the room $\room = \labelingfn(\vertex)$ at state $\istate$, and the operating policy $\latentpolicy_{\room, \direction = \tuple{\vertex, \varvertex}}$. For non-exit states~$\istate$, the transition function~$\probtransitions_{\planner}(\cdot \mid \istate^*)$ follows $\probtransitions_{\room}(\cdot \mid \istate, \action)$ with $\action \sim \latentpolicy_{\room, \direction}(\cdot \mid \istate)$;
for exit states, the planner chooses direction $\direction' \in \directions_{\room'}$ for the next room $\room' = \labelingfn(\varvertex)$, where $\probtransitions_{\planner}(\cdot \mid \istate^{}, \direction')$ follows
$\entrancefn_{\room'}(\cdot \mid \direction)$ from $\direction = \tuple{\vertex, \varvertex}$. 

An optimal stationary policy exists for $\mdpplanner{}$~\citep{DBLP:books/wi/Puterman94} and can be implemented by a planner that memorizes the room's entry direction. This requires \emph{memory of size}~$\abs{\vertices}$, as decisions depend on any of the $\abs{\vertices}$ preceding vertices.

\begin{theorem}\label{thm:mdp-planner-equiv}
    Given low-level policies~$\planner$, 
    there is a $\abs{\vertices}$-memory planner~$\controller$ maximizing~$\objective$ in~$\hmdp$ iff there is a deterministic stationary policy~$\policy^{\star}$ maximizing~$\plannerobjective$ in~$\mdpplanner{}$.
\end{theorem}

\paragraph{Planner synthesis.}
As a first step, we construct a \emph{succinct MDP} $\mdp_{\planner}^{\graph}$ that preserves the value of~$\mdpplanner{}$.
States of~$\mdp_{\planner}^{\graph}$ are pairs~$\tuple{\vertex, \varvertex}$ indicating room~$\room = \labelingfn(\varvertex)$ is entered via direction~$\direction = \tuple{\vertex, \varvertex}$.
As in~$\mdpplanner{}$, a planner selects an exit direction~$d' = \tuple{\varvertex, \vertex'}$ for $\room$. We use the following trick. Recall that we consider discounted properties; when $\room$ is exited via direction $d'$ after $j$ steps, the utility is $\gamma^j$. In $\mdp_{\planner}^{\graph}$, we set the probability of transitioning to $\vertex'$ upon choosing $d'$ to the expected value achieved by policy~$\latentpolicy_{\room, \direction'}$ in~$\room$.

\highlight{%
The next example illustrates how setting transition probabilities to be expected values maintains the values between the models.
\begin{example}
Consider the explicit model of Fig.~\ref{fig:example-hierarchy}, projected on two dimensions in Fig.~\ref{fig:example-hierarchy:2d}. 
Each directed arrow corresponds to a transition with a non-zero probability.
    A state of the form $\tuple{\istate, \vertex}$ indicates that the agent is in state~$\istate$ of room~$\labelingfn\fun{\vertex}$. 
Consider a path $\mdppath$ that enters $\labelingfn\fun{\vertex_0} = \room_0$, exits after~$i=3$ steps ($\color{blue} \istate_0 \rightarrow \istate_1 \rightarrow \istate_2 \xrightarrow{\action_{\textit{exit}}}$), enters $\labelingfn\fun{\varvertex} = \room_1$, exits after~$j=3$ steps ($\color{blue} \istate_0 \rightarrow \istate_1 \rightarrow s_2 \xrightarrow{\action_{\textit{exit}}}$), and finally reaches the high-level goal.
The prefix of $\mdppath$ 
in $\room_0$
is discounted to $\gamma^3$ when the agent exits.
Similarly, the suffix of $\mdppath$ in $\room_1$ is discounted to $\gamma^3$.
Once in the goal, the agent gets a ``reward" of one (the goal is reached).
The discounted reward obtained along $\mdppath$ is thus $\gamma^{i+j}=\gamma^{6}$.
In expectation, this corresponds to multiplying the values in the individual rooms and, in turn, with the semantics of $\mdp_{\planner}^{\graph}$ where probabilities are multiplied along a path.
\end{example}
\begin{figure}[H]
\vspace{-1em}
\includegraphics[width=.8\linewidth]{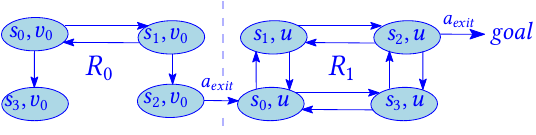}
    \captionof{figure}{Projection of Fig.~\ref{fig:example-hierarchy} on two dimensions.}\label{fig:example-hierarchy:2d}
    \Description{The figure is a recollection of the left element of Figure 2 where room-states with entrance directions are shown below the map vertices that abstract them.}
\end{figure}
}%

Let $\mdp_{\planner}^{\graph} = \tuple{\istates, \actions, \probtransitions, \mdpI}$
with~$\istates = \edges \cup \set{\bot}$, $\actions = \edges$, $\mdpI\fun{\direction_0} = 1$,
\begin{equation}
\probtransitions\fun{\tuple{\varvertex, t} | \tuple{\vertex, \varvertex}, \direction} = 
    \expected{\istate \sim \entrancefn_{\labelingfn\fun{\varvertex}}\fun{\sampledot \mid \tuple{\vertex, \varvertex}}}{\valinit{\latentpolicy_{\labelingfn\fun{\varvertex}, \direction}}{\objective_{\labelingfn\fun{\varvertex}}^{\direction}}{\discount}{\istate}}\label{eq:succint-mdp-transition},
\end{equation}
and
$\probtransitions\fun{\bot \mid \tuple{\vertex, \varvertex}, \direction} = 1 - \probtransitions\fun{\tuple{ \varvertex, t} \mid \tuple{\vertex, \varvertex}, \direction}$ for any~$\tuple{\vertex, \varvertex} \in E$ with target direction $\direction = \tuple{\varvertex, t} \in \directions_{\labelingfn\fun{\varvertex}}$, while $\probtransitions\fun{\bot \mid \bot, \direction} = 1$.
The sink state~$\bot$ captures when low-level policies do not satisfy the objective. 

\begin{theorem}\label{thm:succint-mdp-plan}
    Let $\tuple{\controller, \planner}$ be a $\abs{\vertices}$-memory controller for $\hmdp$ and $\policy$ be an equivalent policy in $\mdp_{\planner}$, the values obtained under $\policy$ for $\objective$ in $\mdp_{\planner}$ are equal to those under $\controller$ obtained in $\mdp_{\planner}^{\graph}$ for the reachability objective to states $\vertices \times T$.
\end{theorem}

We are ready to describe the algorithm to synthesize a planner.
Note that the values $V^{\latentpolicy_{\room, {\direction}}}$ 
in Eq.~\eqref{eq:succint-mdp-transition}
are either unknown or computationally intractable.
Instead, we leverage the latent model to evaluate the \emph{latent value} of each low-level objective using standard techniques for discounted reachability objectives~\cite{DBLP:conf/icalp/AlfaroHM03}.
We construct 
$\latentmdp_{\planner}^{\graph}$ similar to $\mdp_{\planner}^{\graph}$ and obtain the controller $\tuple{\controller, \planner}$ by computing a planner $\controller$  
optimizing the values of $\latentmdp_{\planner}^{\graph}$~\cite{DBLP:books/wi/Puterman94}.
As $\latentmdp_{\planner}^{\graph}$ and $\mdp_{\planner}^{\graph}$ have identical state spaces, planners for $\latentmdp_{\planner}^{\graph}$ are compatible with~$\mdp_{\planner}^{\graph}$.

\paragraph{Lifting the guarantees.}%
We now lift the guarantees for low-level policies to a planner operating on the two-level model, overcoming the following challenge. To learn one latent model per room~$\room$ and the set of low-level policies~$\planner$, we run WAE-DQN independently (and possibly in parallel) in each room~$\room$ (Fig.~\ref{fig:intro-1}).
Viewing $\room$ as an MDP, we obtain a transition loss~$\transitionloss^{\room, \direction}$ for every direction~$\direction$, associated with latent policy~$\latentpolicy_{\room, \direction} \in \planner$.
Independent training introduces complications. Each room~$\room$ has its own initial distribution~$\mdpI_{\room}$, while at synthesis time, the initial distribution depends on the controller~$\policy = \tuple{\controller, \planner}$ and marginalizes $\entrancefn_{\room}(\cdot \mid \direction)$ over directions~$\direction$ chosen by~$\controller$. Recall that $\transitionloss^{\room, \direction}$ is the TV between original and latent transition functions, averaged over $\stationary{\latentpolicy_{\room, \direction}}$, i.e., states likely to be visited under~$\latentpolicy$ when using~$\mdpI_{\room}$ as the entrance function.
The latter differs from~$\entrancefn_{\room}$, used at synthesis time.
As $\stationary{\latentpolicy_{\room, \direction}}$ may not align with the state distribution visited under the two-level controller~$\policy$, $\transitionloss^{\room, \direction}$ (and thus the guarantees from the latent model) may become obsolete or non-reusable.

Fig.~\ref{fig:distr-shift} illustrates the distribution shift.
Assume that~$\controller$ chooses the right direction~$\rightarrow$ in~$\room$.
As~$\color{blue!75!white}\mdpI_{\room}$ is uniform, every state is included in
\begin{wrapfigure}{r}{.35\linewidth}
\includegraphics[width=\linewidth,keepaspectratio,clip,trim=0mm 1mm 11.5mm 8.5mm]{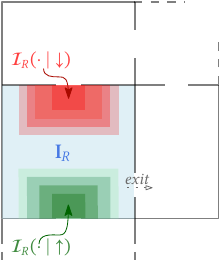}
  \caption{Uniform distribution $\mdpI_{\room}$~({\color{blue!75!white}blue}), entrance~function~$\entrancefn_{\room}$ (\textcolor{red}{red}:~$\downarrow$, \textcolor{green!50!black}{green}:~$\uparrow$).
  }%
    \label{fig:distr-shift}
    \Description{The figure shows the distribution shift when the entrance function is uniform. Hence, every state is included in the support of the distribution of states visited at training time.}
\end{wrapfigure}
the support of the distribution $\stationary{\latentpolicy_{\room, \rightarrow}}$ of states visited under~$\latentpolicy_{\room, \rightarrow}$ at training time.
In contrast, under a two-level controller, rooms are entered according to $\entrancefn_{\room}\fun{\sampledot \mid \direction \in \set{{\color{red}\downarrow}, {\color{green!50!black}\uparrow}}}$.
Since the goal is to exit on the
right,
all states of~$\room$ need not be visited under~$\latentpolicy_{\room, \rightarrow}$, so the distribution over visited states may differ.
The question is whether we can
recover the latent models' guarantees
at synthesis time.

Fortunately, as we will show in the following theorem, it turns out that if the initial distribution~$\mdpI_{\room}$ of each room~$\room$ is well designed and provides sufficient coverage of the state space of~$\room$, it is possible to learn a \emph{latent entrance function~$\latententrancefn_{\room}$} so that the guarantees associated with each room can be lifted to the two-level controller.
\begin{theorem}\label{thm:lifting-guarantees}
Let~$\tuple{\controller, \planner}$ be a $\abs{\vertices}$-memory controller for~$\hmdp$ and
$\policy$ be an equivalent stationary policy in $\mdpplanner{}$.
\begin{itemize}
    \item \emph{\highlight{(Entrance loss)}} Define
    $\latententrancefn_{\room} \colon \directions_{\room} \to \Delta\big({\latentstates}\big)$ and
\vspace{-.5em}
\[
	\entranceloss = \expectedsymbol{\room, \direction \sim \stationary{\policy}} \divergence\fun{{\embed\entrancefn_{\room}\fun{\sampledot \mid \direction}, \,  \latententrancefn_{\room}\fun{\sampledot \mid \direction}}}, 
\]
where
$\stationary{\policy}$ is the stationary measure of $\mdpplanner{}$
under $\policy$
and
\[
\embed\entrancefn_{\room}\fun{\latentstate \mid \direction} = \Prob_{\istate \sim \entrancefn_{\room}\fun{\sampledot \mid \direction}}[\latentstate = \embed_{\room}\fun{\istate}]
\quad \text{for all } \latentstate \in \latentstates;
\]
\item \emph{\highlight{(State coverage)}} Assume that for any training room $\room \in \rooms$ and direction $\direction \in \directions_{\room}$, the projection of the BSCC of $\mdpplanner{}$ under $\policy$ to $\istates_{\room}$ is included in the BSCC of $\room$ under~$\latentpolicy_{\room, \direction}$;
\end{itemize}
Then, there exists a constant $K \geq 0$ so that:
\begin{align*}
      \Big| \, V_{\mdpI}^{\mdp_{\planner}, \policy} -\overline{V}_{\latentmdpI}^{\latentmdp_{\planner}^{\graph}, \controller}\, \Big| \leq 
      \frac{\entranceloss +  K \cdot \expectedsymbol{\room, \direction \sim \stationary{\policy}} \transitionloss^{\room, \direction}}
      {\stationary{\policy}\fun{\sreset} \cdot \fun{1 {-} \discount }}.
\end{align*}
\end{theorem}

\highlight{
Essentially, under mild conditions, the guarantees obtained for \emph{individually trained} rooms can be \emph{reused} for the entire two-level environment. By minimizing losses within each room \emph{independently}, the true environment's values increasingly align with those computed in the latent space for the high-level objective.

This
is the building block that enables our technique, as low-level latent policies are trained
\emph{before} performing synthesis. 
}

\section{Case Studies}\label{sec:evaluation}
\highlight{%
While the focus of this work is primarily of a theoretical nature, we show in the following that our theory is grounded
through a navigation domain involving an agent required to reach a distant location while avoiding moving adversaries.
We consider two challenging case studies.
The first one consists of a large grid world of scalable size with a nontrivial observation space.
The second one is a large \texttt{ViZDoom} environment~\cite{DBLP:conf/cig/KempkaWRTJ16} with visual inputs.
}%

\highlight{
Our framework
allows formally verifying the values of the specification in a learned model, providing PAC bounds on the abstraction quality of this model, and synthesizing a controller in such large environments with guarantees. 
Thus, this section aims to show the following:} 
\begin{enumerate*}[(1)]
    \item our method successfully trains latent policies in non-trivial settings;
    \item the theoretical bounds are a good
prediction for the observed behavior;
    \item our low-level policies are reusable and can be composed into a strong global policy.
\end{enumerate*}

\begin{figure}[tb]
    \centering
    \begin{subfigure}{.495\linewidth}
        \centering
        \includegraphics[width=\textwidth,keepaspectratio]{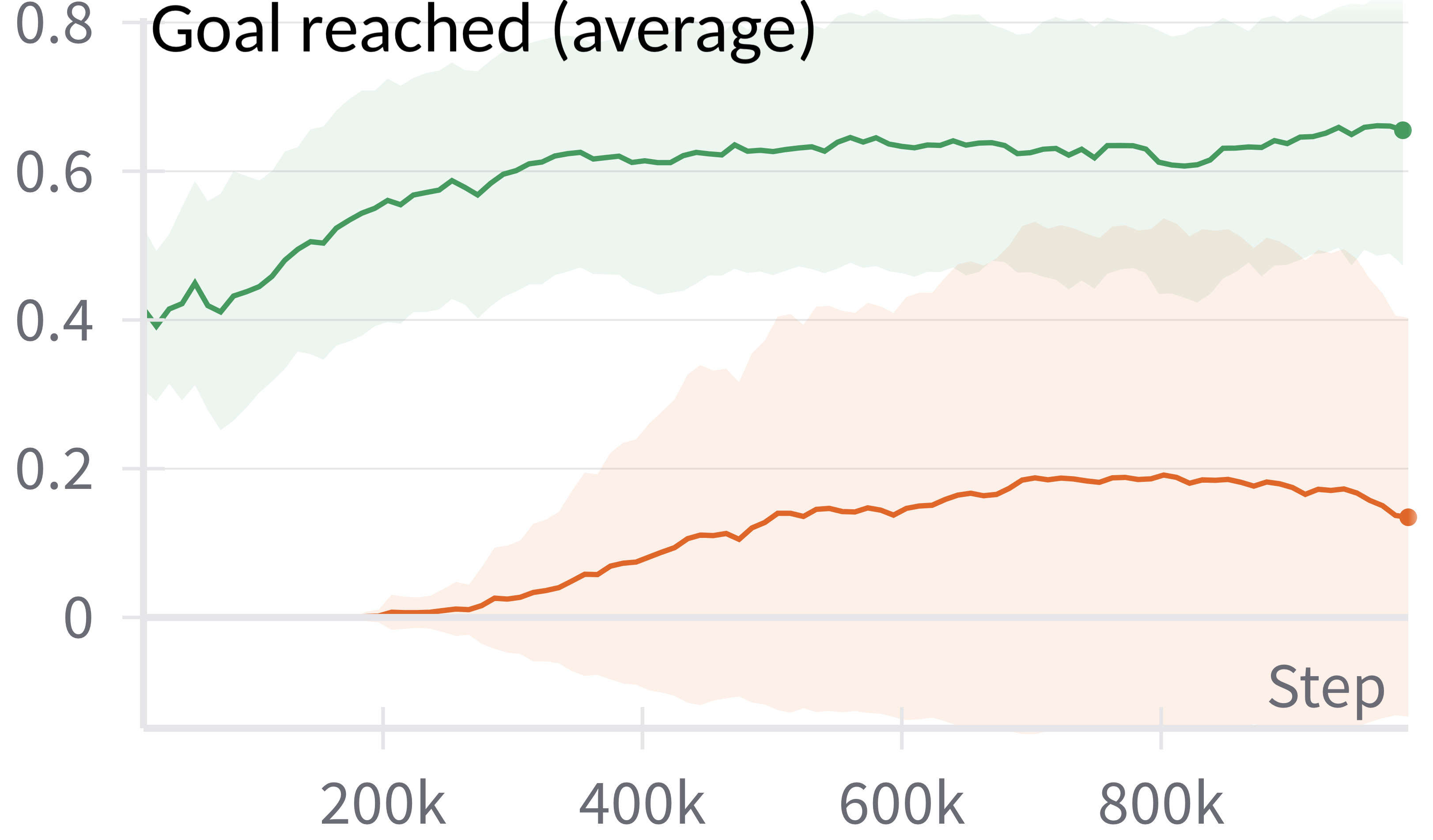}
    \end{subfigure}
    \hfill
    \begin{subfigure}{.495\linewidth}
        \centering
        \includegraphics[width=\textwidth,keepaspectratio]{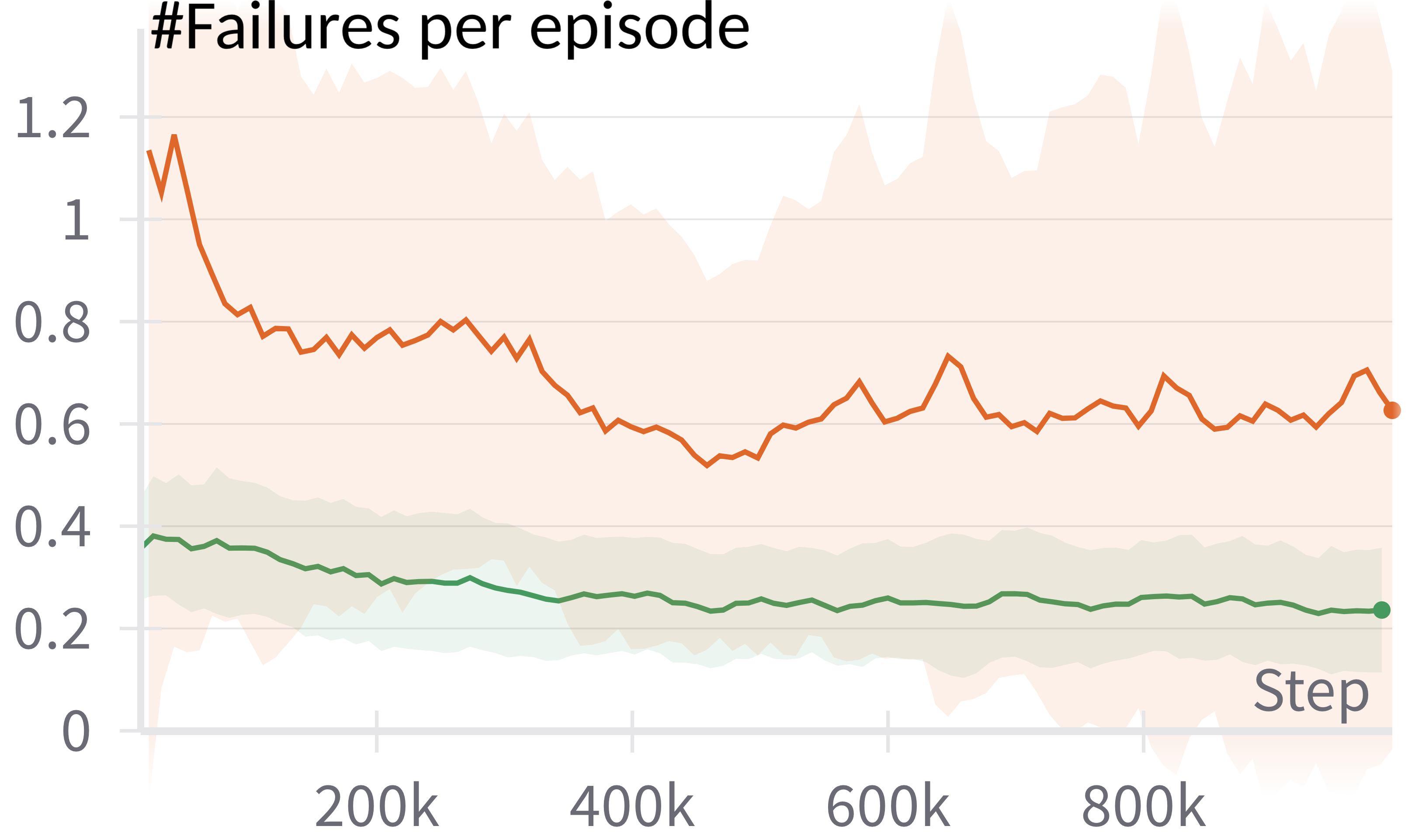}
    \end{subfigure}
    \caption{Evaluation of {\color{green!45!black}WAE-DQN (low-level)} and {\color{red!55!yellow!90!black}DQN (high-level)} policies respectively {\color{green!45!black}in each room/direction} and {\color{red!55!yellow!90!black} in a $9$-room, $20 \times 20$ grid-world environment} (avg.~over $30$ rollouts).
    }
    \label{fig:DRL}
    \Description{The figures show how WAE-DQN fares against DQN in each room/direction in a 9-room 20 by 20 grid-world environment. It is visibly clear that WAE-DQN reaches the goal more often on average and that it fails less often on average.}
\vspace{-.5em}
\end{figure}

\paragraph{Grid world.}
The grid-world environments consist of $N$ rooms of $m \times n$ cells, each containing at most $l$ possible items: walls, entries/exits, power-ups, and $A$ adversaries. The latter patrol, \emph{moving between rooms}, with varying \emph{stochastic behaviors} (along walls, chasing the agent, or fully random). 
\emph{The rooms need not be identical}.
Each state features
(i) a bitmap of rank $4$ and shape $\left[N, l, m, n\right]$ and
(ii) step, power-up, and life-point (\textsc{Lp}) counters.
The
state space is
large and policies may require, e.g.,
convolutional NNs
to process the observations.
Fig.~\ref{fig:DRL} shows that DRL (here, DQN with SOTA extensions and reward shaping, \cite{DBLP:conf/aaai/HesselMHSODHPAS18,DBLP:conf/icml/NgHR99}) struggles to learn for $9$ rooms/$11$ adversaries, while applying WAE-DQN independently in each room
allows learning to satisfy low-level reach-avoid objectives.

\highlight{%
\paragraph{ViZDoom.}
We designed a map for the
video game \emph{Doom} consisting of $N=8$ distinct rooms. The map includes $A$ adversaries that
pursue and attack the agent, reducing the agent's health upon successful hits. Additional adversaries spawn randomly on the map (every ${\sim}60$ steps). Similar to the grid-world environment, adversaries can move freely between rooms. The agent has the ability to shoot; however, missed shots incur a negative reward during the RL phase, penalizing wasted ammunition.
The agent's observations consist of (i) a single frame of the game (\emph{visual input}), (ii) the velocity of the agent along the $x, y$ axes, (iii) the agent's angle w.r.t.~the map, and (iv) its current health.
Notice that the resulting state space is inherently colossal due to the inclusion of these variables.
}%

\begin{wraptable}{r}{.4\linewidth}
    \vspace{-1em}
    \resizebox{\linewidth}{!}{%
    \begin{tabular}{@{}ccc@{}}
    \toprule
    {$d$} & {Grid World} & \highlight{\texttt{ViZDoom}} \\ \midrule
    $\rightarrow$      & 0.50412 & 0.32011                                 \\
    $\leftarrow$       & 0.77787 &          0.44883                      \\
    $\uparrow$         & 0.49631 &                   0.37931             \\
    $\downarrow$       & 0.48058 &                  0.48108              \\ \bottomrule
    \end{tabular}%
    }
    \captionof{table}{PAC bounds \protect\hyperlink{LP}{\ensuremath{\localtransitionlossapprox{d}}}.}
    \label{tab:pac-bounds}
    \vspace{-2em}
\end{wraptable}
\paragraph{Results.}
We use WAE-DQN to train low-level latent models and policies in a $9$-room, $20\times 20$ grid world as well as in the \texttt{ViZDoom} environment. At the start of each episode, the agent is placed in a random room, and the episode concludes successfully when the agent reaches a sub-goal.
Leveraging the representation learning capabilities of WAE-MDPs, the latent space
\begin{table}[t]
\centering
    \resizebox{\linewidth}{!}{%
    \begin{tabular}{@{}lcccccccc@{}}
\toprule
 & $N$ & \textsc{Lp} & $A$ & \multicolumn{2}{c}{avg.~return ($\discount=1$)} & latent value & \multicolumn{2}{c}{avg.~value (original)} \\ \midrule
\multirow{5}{*}{\rotatebox[origin=c]{90}{Grid World}} 
 & $9$  & $1$  & $11$  & $0.5467 \pm 0.1017$  &  & $0.1378$       & $0.07506 \pm 0.01664$    &   \\
 & $9$  & $3$  & $11$  & $0.7 \pm 0.09428$    &  & $0.4343$       & $0.01 \pm 0.00163$       &   \\
 & $25$ & $3$  & $23$  & $0.4933 \pm 0.09832$ &  & $0.1763$       & $0.007833 \pm 0.002131$  &   \\
 & $25$ & $5$  & $23$  & $0.5667 \pm 0.07817$ &  & $0.346$        & $0.00832 \pm 0.00288$    &   \\
 & $49$ & $7$  & $47$  & $0.02667 \pm 0.01491$&  & $0.004229$     & $5.565\text{e{-6}} \pm 7\text{e-6}$ &   \\ \midrule
\multirow{3}{*}{\rotatebox[origin=c]{90}{\texttt{\small \highlight{ViZDoom}}}}
 & $8$  & /    & $8$   & $0.89333 \pm 0.059628$ &  & $0.24171$   & $0.23405 \pm 0.014781$   &   \\
 & $8$  & /    & $14$  & $0.78 \pm 0.064979$    &  & $0.16459$   & $0.16733 \pm 0.023117$   &   \\
 & $8$  & /    & $20$  & $0.39333 \pm 0.11643$  &  & $0.086714$  & $0.06898 \pm 0.017788$   &   \\ \bottomrule
\end{tabular}%
    }
\caption{
Synthesis for 
$\discount = 0.99$.
\highlight{%
\emph{Avg.~return} is the observed, empirical probability of reaching the high-level goal when running the synthesized two-level controller in the environment.
This metric serves as a reference for the controller's performance.
\emph{Latent value} is the predicted value of the high-level objective computed in the latent model.
\emph{Avg.~value} is the empirical value of this objective approximated by simulating the environment under the controller.
}%
}
\vspace{-1.7em}
\label{tab:synthesis}
\end{table}
 generalizes over all rooms: \emph{we only train $4$ policies} (one for each direction). 
 PAC bounds for each direction are reported in Tab.~\ref{tab:pac-bounds} ($\error=0.01$, $\proberror=0.05$).
 \highlight{%
 The lower the bounds, the more accurately the latent model is guaranteed to represent the true underlying dynamics (Thm.~\ref{thm:init-values-pac-learnable}).
}%
From those policies,
we apply our synthesis procedure to construct a two-level controller.
The results are shown in Tab.~\ref{tab:synthesis}.
To emphasize the reusability of the low-level components, we modify the environments by significantly increasing both the number of rooms and adversaries in the grid world (up to $50$ each)
\highlight{%
and the initial number of adversaries in the \texttt{ViZDoom} environment (from $8$ to $20$),
}%
while keeping the same latent models and policies unchanged.

In the grid world, the predicted latent values are consistent with the observed ones and comprised of the approximate return and values in the environment (averaged over $30$ rollouts).
\highlight{%
 In \texttt{ViZDoom}, the PAC bounds (Tab.~\ref{tab:pac-bounds}) are lower, theoretically indicating that the latent model is of higher quality and greater accuracy. This theoretical insight is supported by the results, as the latent values are closer to the empirical, observed ones.
}%

\section{Conclusion}\label{sec:conclusion}
Our approach enables synthesis in environments where traditional formal synthesis does not scale.
Given a high-level map,
we integrate RL in the low-level rooms by training latent policies, which ensure PAC bounds on their value function.
Composing with the latent policies allows to construct a high-level planner in a two-level model, where the guarantees can be lifted.
Experiments show the feasibility in scenarios that are even challenging for pure DRL.

While we believe the map is a mild requirement, future work involves its relaxation to ``emulate'' synthesis with only the specification as input (end-to-end).
In that sense, integrating skill discovery~\cite{DBLP:conf/icml/BagariaS021} or goal-oriented~\cite{DBLP:conf/ijcai/LiuZ022} RL are promising directions.
The problem tackled in this work involves, in essence, multiple objectives. 
A natural extension is to incorporate traditional multi-objective reasoning (e.g., \cite{DBLP:conf/stacs/ChatterjeeMH06,DBLP:journals/aamas/HayesRBKMRVZDHH22}) into the decision process, allowing to reason about the trade-offs between the different low-level objectives.

\begin{acks}

We thank Sterre Lutz and Willem Röpke for providing valuable feedback during the preparation of this manuscript.

This research was supported by the Belgian Flemish AI Research Program, the ``DESCARTES'' iBOF and ``SynthEx'' (G0AH524N) FWO  projects; the Dutch Research Council (NWO) Talent Programme (VI.Veni.222.119); 
Independent Research Fund Denmark (10.46540/3120-00041B), DIREC - Digital Research Centre Denmark (9142-0001B), Villum Investigator Grant S4OS (37819); and the ISF grant (1679/21). 
This work was done in part while Anna Lukina was visiting the Simons Institute for the Theory of Computing.

\end{acks}

\bibliographystyle{ACM-Reference-Format} 
\bibliography{references}

\newpage
\appendix
\onecolumn
\section*{Appendix}

\highlight{
\section{Further Related Work}
\paragraph{On the importance of reachability specifications.}
RL --- and in particular, DRL --- algorithms lack both theoretical and practical guarantees. Our approach aims to advance towards formal guarantees in partially known environments.
In our work, we consider \emph{reach-avoid specifications}, which have attracted considerable attention in recent years from the AI community --- particularly within goal-oriented RL~\cite{DBLP:conf/ijcai/LiuZ022} and reliable safe AI~\cite{dalrymple2024guaranteedsafeaiframework}, both prominent research areas.

Importantly,
\emph{reachability properties are building blocks to verify specifications in stochastic systems}, e.g., for LTL~\cite{DBLP:conf/focs/Pnueli77} or PCTL~\cite{DBLP:journals/fac/HanssonJ94}. Specifically, verifying specifications in an MDP typically boils down to checking the reachability to recurrent regions within a product of the MDP and an (omega-regular) automaton, or a tree decomposition of the formula, involving repeated reachability to satisfiability regions within the MDP~\cite{BK08,DBLP:reference/mc/BaierAFK18}. Although safety can be reduced to reachability --- minimize the value of reaching bad states --- we directly included safety in the specifications for convenience, given its widespread demand in RL. 
The analysis of such properties is necessary and the first step to enable reactive synthesis in settings like ours when limited information about the environment’s dynamics is available.

\paragraph{Multi-objective reasoning.}
The framework introduced in this paper provides latent models and policies that allow to formally reason about the behaviors of the agent.
Real-world systems are complex and often involve multiple trade-offs between (possibly conflicting) constraints, costs, rewards, and specifications.
In fact, the willingness to achieve sub-goals at the lower level of the environment while ensuring that a set of safety requirements are met is a typical example of a multi-objective problem.
In essence, then, our problem involves multiple objectives, not just at the same decision level, but in a multi-level classification of decisions.

Our framework tackles \emph{one} aspect of multi-objective decision making, which we note is not standard: traditional methods \cite{reymond2019pareto,DBLP:conf/icml/AbelsRLNS19,DBLP:conf/atal/ReymondBN22,DBLP:journals/aamas/HayesRBKMRVZDHH22,DBLP:conf/atal/AlegreBRN023,DBLP:conf/stacs/ChatterjeeMH06,DBLP:journals/lmcs/EtessamiKVY08,DBLP:conf/atva/ForejtKP12,DBLP:conf/tacas/HartmannsJKQ18,DBLP:conf/tacas/DelgrangeKQR20} involve the ability to reason about the multiple trade-offs by conducting multi-objective analyses (e.g., generating the \emph{Pareto} curve/set/frontier, embedding all the compromises).
In contrast, we focus on dealing and \emph{composing with the different objectives} in order to satisfy the high-level specification.

We note that \citet{DBLP:conf/tacas/WatanabeVHRJ24} consider multi-level environments {while} approximating Pareto curves to deal with the compromises incurred by the low-level tasks. 
However, that approach relies on a model and thus exhibits tractability issues while being inapplicable when the dynamics are not fully known. 
Furthermore, the formalization of our multi-level environment is more permissive and allows to encode information from neighboring rooms (e.g., obstacles or adversaries moving between rooms), which also requires memory for the planner (see \secref{high_level} for more information on memory requirements).
}
\section{Remark about Episodic Processes and Ergodicity}\label{rmk:ergodicity}
Assumption~\ref{assumption:episodic} implies ergodicity of both~$\mdp$ and $\latentmdp$ under mild conditions~\cite{DBLP:conf/nips/Huang20}.
In ergodic MDPs, each state is almost surely visited infinitely often~\cite{BK08}.
Thus, for unconstrained reachability goals ($B = \emptyset$), while a discount factor still provides insights into how quickly the objective is achieved, optimizing the values associated with reaching the target~$T$ before the episode concludes ($B = \set{\sreset}$) is often more appealing. This involves finding a policy~$\policy$ maximizing $\val{\policy}{\reachavoid{T}{B=\set{\sreset}}}{\discount}{}$. In essence, this is how an RL agent is trained: learning to fulfill the low-level objective before the episode concludes.
\section{Proofs from \secref{low_level}}\label{appendix:proofs:low-level}
\paragraph{Notation} For convenience, in the remaining of this Appendix, we may write $\istate, \action \sim \stationary{\policy}$ as shorthand for the distribution over~$\istates\times \actions$ obtained by sampling~$\istate$ from $\stationary{\policy}$ and then sampling~$\action$ from~$\policy(\sampledot \mid \istate)$.
\begin{proof}[Proof of Thm.~\ref{thm:initial-value-bound}]
    Note that
    \begin{equation*}   
    \left| \values{\latentpolicy}{\objective}{\istate} -  \latentvalues{\latentpolicy}{\objective}{\embed\fun{\istate}} \right| \leq \frac{1}{\stationary{\latentpolicy}\fun{\istate}} \expectedsymbol{\istate' \sim \stationary{\latentpolicy}}\left| \values{\latentpolicy}{\objective}{\istate'} -  \latentvalues{\latentpolicy}{\objective}{\embed\fun{\istate'}} \right|
    \end{equation*}
    for any $\istate \in \istates$.
    Since $\sinit$ is almost surely visited episodically, \emph{restarting} the MDP (i.e., visiting $\sreset$) is a measurable event, meaning that $\sreset$ has a non-zero probability  $\stationary{\latentpolicy}\fun{\sreset} \in \mathopen( 0, 1\mathclose)$.
    This gives us:
    \begin{align*}
        & \left| \values{\latentpolicy}{\objective}{} -  \latentvalues{\latentpolicy}{\objective}{} \right| \\
        =&  \left| \expectedsymbol{\istate \sim \mdpI\, } \values{\latentpolicy}{\objective}{\istate} -  \expectedsymbol{\latentstate \sim \latentmdpI\, } \latentvalues{\latentpolicy}{\objective}{\latentstate} \right| \\
        =&  \frac{1}{\discount} \left| \expectedsymbol{\istate \sim \mdpI} \left[ \discount \cdot \values{\latentpolicy}{\objective}{\istate}\right] -  \expectedsymbol{\latentstate \sim \latentmdpI} \left[ \discount \cdot\latentvalues{\latentpolicy}{\objective}{\latentstate}\right] \right| \\
        = & \frac{1}{\discount} \left| \values{\latentpolicy}{\objective}{\sinit} -  \latentvalues{\latentpolicy}{\objective}{\embed\fun{\sinit}} \right| \tag{by Assumption~\ref{assumption:episodic}}\\
        \leq& \, \frac{1}{\discount\stationary{\latentpolicy}\fun{\sinit}} \expectedsymbol{\istate \sim \stationary{\latentpolicy}} \left| \values{\latentpolicy}{\objective}{\istate} -  \latentvalues{\latentpolicy}{\objective}{\embed\fun{\istate}} \right|  \\
        \leq& \frac{\localtransitionloss{\stationary{\latentpolicy}}}{\stationary{\latentpolicy}\fun{\sinit}\fun{1 {-} \discount}}.
        \tag{by Lem.~\ref{theorem:value-diff-bound}}
    \end{align*}
\end{proof}

\begin{proof}[Proof of Lem.~\ref{lem:pac-bounds}]
By definition of the total variation distance, we have
\begin{align*}
\transitionloss{} =&\ \expectedsymbol{\istate, \action \sim \stationary{\latentpolicy}}D\fun{\phi\probtransitions\fun{\sampledot \mid \istate, \action}, \latentprobtransitions\fun{\sampledot \mid \embed\fun{{\istate}}, \action}} \\
=&\ \expected{\istate, \action \sim \stationary{\latentpolicy}}{\frac{1}{2} \sum_{\latentstate' \in \latentstates} \abs{\Prob_{\istate' \sim \probtransitions\fun{\sampledot \mid \istate, \action}}\left[\embed\fun{\istate'} = \latentstate'\right] - \latentprobtransitions\fun{\latentstate' \mid \istate, \action}}}\\
=&\ \expected{\istate, \action \sim \stationary{\latentpolicy}}{\frac{1}{2} \sum_{\latentstate' \in \latentstates} \abs{\expectedsymbol{\istate' \sim \probtransitions\fun{\sampledot \mid \istate, \action}}\condition{\embed\fun{\istate'} = \latentstate'} - \latentprobtransitions\fun{\latentstate' \mid \istate, \action}}}.
\end{align*}
Notice that this quantity cannot be approximated from samples distributed according to $\stationary{\latentpolicy}$ alone: intuitively, we need to have access to the original transition function $\probtransitions$ to be able to estimate the expectation $\expectedsymbol{\istate' \sim \probtransitions\fun{\sampledot \mid \istate, \action}}\condition{\embed\fun{\istate'} = \latentstate'}$ for each single point drawn from $\stationary{\latentpolicy}$. 

Instead, consider now the following upper bound on $\transitionloss$:
\begin{equation*}
    \localtransitionloss{} \leq \expectedsymbol{\istate, \action \sim \stationary{\latentpolicy}} \expectedsymbol{\istate' \sim \probtransitions\fun{\sampledot \mid \istate, \action}} D\fun{\embed\fun{\sampledot \mid \istate'}, \latentprobtransitions\fun{\sampledot \mid \latentstate, \action}} = \localtransitionlossupper{},
\end{equation*}
where $\embed\fun{\latentstate' \mid \istate'}$ is defined as $\condition{\embed\fun{\istate'} = \latentstate'}$ for any $\latentstate' \in \latentstates$.
This bound directly follows from Jensen's inequality. We know from~\cite{DBLP:conf/aaai/DelgrangeN022} that $\localtransitionlossapprox{} + \error \leq \localtransitionlossupper{}$ with probability at most $\exp\fun{-2 \T \error^2}$.
We recall the proof for the sake of presentation:

    \begin{align*}
        &\ \localtransitionlossupper{\stationary{\latentpolicy}} \\
        =&\ \expectedsymbol{\istate, \latentaction, \istate' \sim \stationary{\latentpolicy}} D\fun{\embed\fun{\sampledot \mid \istate'},\latentprobtransitions\fun{\sampledot \mid \embed\fun{\istate}, \action}} \\
        =&\ \expected{\istate, \latentaction, \istate' \sim \stationary{\latentpolicy}}{ \frac{1}{2} \sum_{\latentstate' \in \istates} \left| \embed\fun{\latentstate' \mid \istate'} - \latentprobtransitions\fun{\latentstate' \mid \embed\fun{\istate}, \latentaction} \right|} \\
        =&\ \expected{\istate, \latentaction, \istate' \sim \stationary{\latentpolicy}}{ \frac{1}{2} \cdot \fun{\fun{1 - \latentprobtransitions\fun{\embed\fun{\istate'} \mid \embed\fun{\istate}, \latentaction}} + \sum_{\latentstate' \in \istates \setminus\set{\embed\fun{\istate'}}} \left| 0 - \latentprobtransitions\fun{\latentstate' \mid \embed\fun{\istate}, \latentaction} \right|}}
        \tag{because $\embed\fun{\latentstate' \mid \istate'} = 1$ if $ \embed\fun{\istate'} = \latentstate'$ and $0$ otherwise} \\
        =&\ \expected{\istate, \latentaction, \istate' \sim \stationary{\latentpolicy}}{ \frac{1}{2} \cdot \fun{\fun{1 - \latentprobtransitions\fun{\embed\fun{\istate'} \mid \embed\fun{\istate}, \latentaction}} + \sum_{\latentstate' \in \istates \setminus\set{\embed\fun{\istate'}}} \latentprobtransitions\fun{\latentstate' \mid \embed\fun{\istate}, \latentaction}}} \\
        =&\ \expected{\istate, \latentaction, \istate' \sim \stationary{\latentpolicy}}{\frac{1}{2} \cdot 2 \cdot \fun{1 - \latentprobtransitions\fun{\embed\fun{\istate'} \mid \embed\fun{\istate}, \latentaction}}} \\
        =&\ \expected{\istate, \latentaction, \istate' \sim \stationary{\latentpolicy}}{1 - \latentprobtransitions\fun{\embed\fun{\istate'} \mid \embed\fun{\istate}, \latentaction}}.
    \end{align*}
    By Hoeffding's inequality, we obtain that $\localtransitionlossapprox{} + \error \leq \localtransitionlossupper{}$ with probability at most $\exp\fun{-2\T\error^2}$.
    Equivalently, this means that $\localtransitionlossapprox{} + \error > \localtransitionlossupper{}$ with at least probability $1 - \exp\fun{-2\T\error^2}$.
    The fact that $\localtransitionlossapprox{}{} + \error > \localtransitionlossupper{} \geq \transitionloss$ finally yields the bound.

    By applying Hoeffding's inequality again, we obtain that with at most probability $\exp\fun{{-2\T\error^2}}$, we have $\stationaryapprox - \error \geq \stationary{\latentpolicy}\fun{\sreset}$.
    By the union bound, we have 
    \[
        \Prob\fun{\localtransitionlossapprox{} + \error \leq \localtransitionlossupper{} \text{ or } \stationaryapprox - \error \geq \stationary{\latentpolicy}\fun{\sreset}} \leq \exp\fun{-2\T\error^2} + \exp\fun{-2T\error^2}.
    \]
    Finding a $\T\geq0$ which yields $\delta \geq 2 \exp\fun{-2\T\error^2}$ is sufficient to ensure the bound.
    In that case, we have 
    \begin{align}
        \proberror \geq 2\exp\fun{-2\T\error^2} \Leftrightarrow\nicefrac{\proberror}{2} \geq \exp\fun{-2\T\error^2} \Leftrightarrow \log\fun{\nicefrac{\proberror}{2}} \geq {-2\T\error^2} \Leftrightarrow \T \geq \frac{-\log\fun{\nicefrac{\proberror}{2}}}{2\error^2}. \label{eq:proof-pac-bounds-2}
    \end{align}
    Then, we have that with at least probability $1 - \proberror$, $\localtransitionlossapprox{} + \error > \transitionloss$ and $\stationaryapprox - \error < \stationary{\latentpolicy}\fun{\sreset}$ if $\T \geq \left\lceil \nicefrac{-\log\fun{\proberror}}{2 \error^2}\right\rceil$.
\end{proof}

\begin{proof}[Proof of Thm.~\ref{thm:init-values-pac-learnable}]
 Let $\varerror, \proberror > 0$,
 then we know by Lem.~\ref{theorem:value-diff-bound}, Thm.~\ref{thm:initial-value-bound}, and Lem.~\ref{lem:pac-bounds} that
 \begin{enumerate}[(i)]
    \item $
       \expectedsymbol{\istate \sim \stationary{\latentpolicy}} \left|\values{\latentpolicy}{}{\istate} - \latentvalues{\latentpolicy}{}{\embed\fun{\istate}} \right| \leq \frac{\discount\transitionloss}{1 - \discount} \leq \frac{\discount\fun{\localtransitionlossapprox{} + \varerror}}{1 - \discount}
    $, with probability $1 - \proberror$. 
    Then, to ensure an error of at most $\error > 0$, we need to set $\varerror$ such that:
    \begin{align*}
    && \frac{\discount\fun{\localtransitionlossapprox{} + \varerror}}{1 - \discount} & \leq  \frac{\discount{\localtransitionlossapprox{}}}{1 - \discount} + \varepsilon &
    &\iff& \frac{\discount \varerror}{1 - \discount} & \leq \error  &
    &\iff& \varerror & \leq \frac{\error \fun{1 - \discount}}{\discount} .&
    \end{align*}
    Then, by Lem.~\ref{lem:pac-bounds}, we need $\T \geq \left\lceil \frac{- \log \proberror}{2 \varerror^2}\right\rceil = \left\lceil \frac{- \discount^2 \log \proberror}{2 \error^2(1 - \discount)^2}\right\rceil$ samples to provide an error of at most $\error$ with probability $1 - \proberror$.
    \item
     $
      \left| \values{\latentpolicy}{}{} -  \latentvalues{\latentpolicy}{}{} \right|
        \leq \frac{\localtransitionloss{\stationary{\latentpolicy}}}{\stationary{\latentpolicy}\fun{\sinit}\fun{1 {-} \discount}}
        \leq \frac{ \localtransitionlossapprox{} + \varerror}{\fun{\stationaryapprox - \varerror}\fun{1 - \discount}}
        $
 with probability at least $1 - \proberror$.
     Then, to ensure an error of at most $\error > 0$, we need to set $\varerror$ such that:
    \begin{align*}
    & \frac{\localtransitionlossapprox{}}{\stationaryapprox \cdot \fun{1 - \discount}} + \error \geq  \frac{ \localtransitionlossapprox{} + \varerror}{\fun{\stationaryapprox - \varerror}\fun{1 - \discount}}\\
    \iff & \fun{\stationaryapprox - \varerror} \fun{\frac{ \localtransitionlossapprox{}}{\stationaryapprox } + {\error \fun{1 - \discount}}} \geq \localtransitionlossapprox{} + \varerror \\
    \iff & {\color{red}\localtransitionlossapprox{}} + \stationaryapprox \cdot \error \fun{1 - \discount} - \frac{ \localtransitionlossapprox{} \cdot \varerror}{\stationaryapprox} - \error \cdot \varerror \fun{1 - \discount} \geq  {\color{red}\localtransitionlossapprox{}} + \varerror \\
    \iff & \stationaryapprox \cdot \error \fun{1 - \discount} \geq \varerror + \frac{ \localtransitionlossapprox{} \cdot \varerror}{\stationaryapprox} + \error \cdot \varerror \fun{1 - \discount}
    = \varerror \fun{1 + \frac{ \localtransitionlossapprox{}}{\stationaryapprox} + \error \fun{1 - \discount}}\\
    \iff & \frac{\stationaryapprox \cdot \error \fun{1 - \discount}}{1 + \frac{ \localtransitionlossapprox{}}{\stationaryapprox} + \error \fun{1 - \discount}} \geq \varerror \Leftrightarrow \frac{\stationaryapprox^2 \cdot  \error \fun{1 - \discount}}{\localtransitionlossapprox{} + \stationaryapprox \cdot  \fun{1 + \error\fun{1 - \discount}}} \geq \varerror.
    \end{align*}
    Notice that this upper bound on $\varerror > 0$ is well defined since 
    \begin{multicols}{2}
    \begin{enumerate}[(a)]
     \item $\stationaryapprox^2 \cdot \error \fun{1 - \discount} > 0$, and
     \item $\localtransitionlossapprox{} + \stationaryapprox \fun{1 + \error\fun{1 - \discount}} > 0$.
    \end{enumerate}
    \end{multicols}
    Then, setting $\varerror \leq \frac{\stationaryapprox^2 \cdot \error \cdot \fun{1 - \discount}}{\localtransitionlossapprox{} + \stationaryapprox \fun{1 + \error\fun{1 - \discount}}} $ means by Lem.~\ref{lem:pac-bounds} that we need \[\T \geq \left\lceil \frac{-\log\fun{\nicefrac{\proberror}{2}}}{2 \varerror^2}\right\rceil \geq \left\lceil\frac{-\log\fun{\nicefrac{\delta}{2}}\fun{ \localtransitionlossapprox{} + \stationaryapprox\fun{1 + \error\fun{1 - \discount}}}^2}{2 \stationaryapprox^4 \cdot \error^2\fun{1 - \discount}^2}\right\rceil\] samples to provide an error of at most $\error$ with probability at least $1 - \proberror$.
    \qedhere
 \end{enumerate}
\end{proof}

\section{WAE-DQN}\label{appendix:wae-dqn}
In this section, we give additional details on WAE-DQN, which combines representation (WAE-MDP) and policy (DQN) learning.
Before presenting the algorithm, we briefly recall basic RL concepts.
\paragraph{Q-Learning.}%
Q-learning is an RL algorithm whose goal is to learn the optimal solution of the Bellman equation~\cite{DBLP:books/wi/Puterman94}:
$Q^{*}\fun{\istate, \action} = \expectedsymbol{\istate' \sim \probtransitions\fun{\sampledot \mid \istate, \action}} \left[\reward\fun{\istate, \action, \istate'} + \discount \cdot \max_{\action' \in \actions}Q^{*}\fun{\istate', \action'}\right]$ for any $(\istate,\action) \in \istates \times \actions$, with
\[\expectedsymbol{\istate_0 \sim \mdpI} \left[\max_{\action \in \actions} Q^{*}\fun{\istate_0, \action}\right] = \max_{\policy} \expectedsymbol{\mdppath \sim \Pr^{\mdp}_{\policy}}\left[ \sum_{i \geq 0} \discount^i \cdot r_i \right].\] 
To do so, Q-learning relies on learning \emph{Q-values} iteratively:
at each step $i \geq 0$, a transition $\tuple{\istate, \action, r, \istate'}$ is drawn in $\mdp$, and 
$
    Q_{i + 1}\fun{\istate, \action} = Q_i\fun{\istate, \action} + \alpha (r + \discount \allowbreak \max_{\action' \in \actions} \allowbreak Q_{i}\fun{\istate', \action'} - Q_i\fun{\istate, \action})
$
for a given learning rate $\alpha \in \mathopen( 0, 1\mathclose)$. Under some assumptions, $Q_i$ is guaranteed to converge to $Q^{*}$~\cite{DBLP:journals/ml/Tsitsiklis94}.
Q-learning is implemented by maintaining a table of size $\abs{\istates \times \actions}$ of the Q-values. This is intractable for environments with large or continuous state spaces. 

\paragraph{Deep Q-networks} ({DQN}, ~\cite{DBLP:journals/nature/MnihKSRVBGRFOPB15}) is an established technique to scale Q-learning (even for continuous state spaces), at the cost of convergence guarantees, by approximating the Q-values in parameterized NNs. 
By fixing a network $Q({\sampledot, \theta})$ and, for stability~\cite{DBLP:journals/tac/TsitsiklisR97},
periodically fixing a parameter assignment $\widehat{\theta}$, DQN obtains \emph{the target network} $Q({\sampledot, \widehat{\theta}})$.
$Q$-values are then optimized by applying gradient descent on the following loss function:
\begin{equation}\label{eq:dqn}
    \ldqn(\theta) = \expectedsymbol{\istate, \action, r, \istate' \sim \mathcal{B}} \fun{r + \discount \max_{\action' \in \actions} Q(\istate', \action' \,;\, \widehat{\theta}) - Q\fun{\istate, \action \,;\, \theta}}^2,
\end{equation}
where $\policy^{\epsilon}$ is an \emph{$\epsilon$-greedy} exploration strategy, i.e., $\policy^{\epsilon}\fun{\action \mid \istate}=\fun{1 - \epsilon} \condition{a = \arg\max_{a'} Q\fun{\istate, \action'}} + \nicefrac{\epsilon}{\abs{\actions}}$ for some $\epsilon \in \mathopen(0, 1\mathclose)$.
In practice, $\stationary{\policy}$ is emulated by a \emph{replay buffer} $\mathcal{B}$ where encountered transitions  are stored and then sampled later on to minimize $\ldqn(\theta)$.

\emph{Wasserstein auto-encoded MDP} (WAE-MDP, \cite{delgrange2023wasserstein}) is a distillation technique providing PAC guarantees. Given an MDP $\mdp$, a policy $\policy$ and the number of states in $\latentmdp$, the transition probabilities and embedding function~$\embed$ (both modeled by NNs) are learned by minimizing~$\transitionloss$ via gradient descent. Also, a policy~$\latentpolicy$ in~$\latentmdp$ is distilled
such that  $\latentmdp$ 
exhibits \emph{bisimilarly close}~\cite{DBLP:conf/popl/LarsenS89,DBLP:journals/ai/GivanDG03,DBLP:conf/aaai/DelgrangeN022} behaviors to~$\mdp$ when executing~$\latentpolicy$, providing PAC guarantees on the difference of the two values from Lem.~\ref{theorem:value-diff-bound}.
WAE-MDPs enjoy \emph{representation} guarantees that any states clustered to the same latent representation yield close values when $\transitionloss$ is minimized~\cite{DBLP:conf/aaai/DelgrangeN022}: 
for any latent policy $\latentpolicy$ and $\istate_1, \istate_2 \in \istates$, $\embed\fun{\istate_1} = \embed\fun{\istate_2}$ implies
$\abs{\values{\latentpolicy}{}{\istate_1} - \values{\latentpolicy}{}{\istate_2}} \leq \frac{\discount \transitionloss}{1 - \discount} \fun{\nicefrac{1}{\stationary{\latentpolicy}\fun{\istate_1}} + \nicefrac{1}{\stationary{\latentpolicy}\fun{\istate_2}} }$.

\begin{figure}[b]
    \centering
    \includegraphics[width=.6\linewidth]{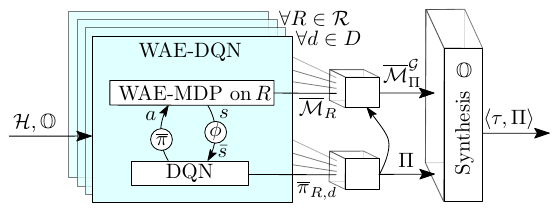}
    \caption{
    Given
    $\hmdp$ and
    $\objective$, we run WAE-DQN in each room $\room\in\rooms$ and direction $\direction\in\directions$ in parallel,
    yielding embedding $\embed$, latent MDPs, and policies $\planner$ with PAC guarantees.
    We then synthesize planner $\tau$ to maximize $\objective$ in succinct model 
    $\latentmdp_\planner^\graph$, aggregated as per the map of $\hmdp$, given as graph $\graph$.}
    \label{fig:approach-overview}
\end{figure}
\begin{algorithm}
\caption{WAE-DQN}\label{algo:wae-dqn}
\DontPrintSemicolon
\KwIn{%
steps $\mathcal{T}$, model updates $\mathcal{N}$, batch sizes $B_{\text{WAE}}, B_{\text{DQN}}$, and  $\alpha, \epsilon \in \mathopen(0, 1\mathclose)$;
} 
\SetKwComment{Comment}{$\triangleright$\ }{}
\SetCommentSty{textit}
\LinesNotNumbered 
\SetKwFor{RepTimes}{repeat}{times}{end}
\SetKwBlock{Begin}{function}{end function}
\smallskip
Initialize the taget parameters: $\tuple{\hat{\iota}, \hat{\theta}_{\text{DQN}}}\leftarrow $ \textbf{copy} the parameters $\tuple{\iota, \theta_{\text{DQN}}}$ \;
Initialize replay buffer $\mathcal{B}$ with transitions from random exploration of $\mdp$\;
    \For{$t \in \set{1, \dots, \T}$ with $\istate_0 \sim \mdpI$}{
        Embed $\istate_t$ into the latent space: $\latentstate \leftarrow \embed\fun{\istate_t}$ \;
        Choose action $a_t$: 
        $\begin{cases} 
        \text{w.p.  }\fun{1 \!-\! \epsilon}\text{, define }a_t = \arg\max_{\action} Q\fun{\latentstate, \action}\text{, and }\\
        \text{w.p. }\epsilon\text{, draw }\action_t \text{ uniformly from } $\actions$ \end{cases}$\;
        Execute $\action_t$ in the environment $\mdp$, {receive} reward $r_t$, and {observe} $\istate_{t + 1}$\;
        Store the transition in the replay buffer:
        $\mathcal{B} \leftarrow \mathcal{B} \cup \set{\tuple{\istate_t, \action_t, r_t, \istate_{t + 1}}}$ \;
    \RepTimes{$\mathcal{N}$}{
        {Sample a batch of size $B_{\text{WAE}}$ from $\mathcal{B}$}:
        $X \leftarrow \set{\tuple{\istate, \action, r, \istate'}_i}_{i = 1}^{B_{\text{WAE}}} \sim \mathcal{B}$ \;
        Update $\iota$ and $\theta_{\text{WAE}}$ on the batch $X$ by {minimizing} the WAE-MDP loss (including $\transitionloss$) for the latent policy $\latentpolicy^{\epsilon}$ \Comment*[r]{details in~\cite{delgrange2023wasserstein}}
    }
     \For{$i \in \set{1, \dots, B_{\text{DQN}}}$}{
         {Sample a transition from $\mathcal{B}$}:
     	$\istate, \action, r, \istate' \sim \mathcal{B}$ \;
      	{Compute the target}: $\widehat{y} \leftarrow r + \discount \max_{\action' \in \actions} {Q}\fun{\embed\fun{\istate';\,\hat{\iota}}, \action'; \, \hat{\theta}_{\text{DQN}}}$ \;
       {Compute the DQN loss (Eq.~\ref{eq:dqn})}: $L_i \leftarrow \fun{Q\fun{\embed\fun{\istate; \, \iota}, \action;\, {\theta}_{\text{DQN}}} - \widehat{y}}^2$ 
     }
     Update $\iota$ and $\theta_{\text{DQN}}$ by minimizing $\nicefrac{1}{B_{\text{DQN}}} \sum_{i = 1}^{B_{\text{DQN}}} L_i$\;
     Update the target params.: $\widehat{\iota} \leftarrow \alpha \cdot \iota + (1 - \alpha) \cdot \widehat{\iota};\; \widehat{\theta} \leftarrow \alpha \cdot \theta_{\text{DQN}} + (1 - \alpha) \cdot \widehat{\theta}$ 
}
\Return{$\embed$, $\latentmdp$, and $\latentpolicy$}
\end{algorithm}
\paragraph{WAE-DQN} 
Our procedure (Fig.~\ref{fig:approach-overview}) unifies the training and distillation steps (Alg.~\ref{algo:wae-dqn}).
Intuitively, a WAE-MDP and a (latent) DQN policy are learned in round-robin fashion: the WAE-MDP produces the input representation (induced by~$\embed$) that the DQN agent uses to optimize its policy~$\latentpolicy$. 
At each step $t =1,\ldots, \T$, the environment is explored via a strategy to collect transitions in a replay buffer.
Each training step consists of two optimization rounds. First, we optimize the parameters of~$\latentprobtransitions$ and~$\embed$.
Second, we optimize DQN's parameters to learn the policy as in DQN.
DQN may further backpropagate gradients through~$\embed$.
We use a \emph{target embedding function}~$\widehat{\embed}$ for stability purposes, similar to~\cite{DBLP:conf/iclr/0001MCGL21}. 
This is consistent with DQN's target-networks approach: the weights of~$\widehat{\embed}$ are periodically synchronized with those of $\embed$.
Then, $\widehat{\embed}$ is paired with the DQN's target network, which allows avoiding oscillations and shifts in the representation (a.k.a.~moving target issues).

{WAE-DQN} learns a tractable model of the environment in parallel to the agent's policy (Algorithm~\ref{algo:wae-dqn}).
Precisely, the algorithm alternates between optimizing the quality of the abstraction as well as the representation of the original state space via a WAE-MDP, and optimizing a latent policy via DQN.
We respectively denote the parameters of the state embedding function $\embed$, those of the latent transition function $\latentprobtransitions$, and those of the Deep Q-networks by $\iota$, $\theta_{\text{WAE}}$, and $\theta_{\text{DQN}}$.

\section{Explicit Construction of the MDP Plan}\label{appendix:mdp-plan-construction}

Along this section, fix a two-level model $\hmdp = \zug{\G, \labelingfn, \rooms, v_0, \tuple{d_0, d_1}}$ with its explicit MDP representation $\mdp = \mdptuple$.

To enable high-level reasoning when the rooms are aggregated into a unified model, we add the following assumption.
\begin{assumption}\label{assumption:same-reset}
    All rooms $R \in \rooms$ share the same reset state $\sreset$ in~$\hmdp$.
\end{assumption}
Note that Assumption~\ref{assumption:same-reset} is a technicality that can be trivially met in every two-level model $\hmdp$: it just requires that when a reset is triggered in a room $\room$ of $\hmdp$, the whole model is globally reset, and not only $\room$, locally.

We define an MDP $\mdpplanner{}$,
called an \emph{MDP plan}, such that policies in $\mdpplanner{}$ correspond to planners. Recall that the actions that a planner performs consist of choosing a policy once entering a room. Accordingly, we define $\mdpplanner{} = \tuple{\istatesplanner, \actionsplanner, \allowbreak\transitionfnplanner, \mdpIplanner}$.
States in $\istatesplanner$ keep track of the location in a room as well as the target of the low-level policy that is being executed. Formally, \[\istatesplanner = \fun{\cup_{\room \in \rooms} \fun{\istates_{\room} \setminus \set{\sreset}} \times \edges} \cup \set{\sreset, \bot},\]
where a pair $\zug{\istate, v,u} \in \istatesplanner$ means that the current room is $v$, the target of the low-level policy is to exit the room in direction $d=\zug{v,u}$, and the current state is $\istate \in \istates_{\labelingfn\fun{v}}$. Following Assumption~\ref{assumption:same-reset}, the rooms share the reset state $\sreset$, and $\bot$ is a special sink state that we add for technical reasons to disable actions in states. 
The initial distribution $\mdpIplanner$ has for support $\set{\tuple{\istate, \vertex, \varvertex} \in \istatesplanner \mid \vertex = \vertex_0 \text{ and } \tuple{\vertex, \varvertex} = \direction_1}$ where states $\istate \in \istates_{\labelingfn\fun{\vertex_0}}$ are distributed according to $\entrancefn_{\labelingfn\fun{\vertex_0}}\fun{\sampledot \mid d_0}$. 
Actions chosen correspond to those of the planner --- only required when entering a room --- so the action space is $\actionsplanner = \edges \cup \set{*}$, where $d \in \edges$ means that the low-level policy that is executed exits via direction $d$, and $*$ is a special action that is used inside a room, indicating no change to the low-level policy.
Note that once $d$ is chosen, we only allow exiting the room through direction $d$.
We define the transition function.
Let $\probtransitions$ be the transition function of the explicit MDP $\mdp$.
For a state $\tuple{\istate, \vertex, \varvertex} \in \istatesplanner$ with $d = \zug{\vertex, \varvertex}$, 
\begin{enumerate}[(i)]
\item \label{enum:i}
if $\istate$ is not an exit state,
i.e., $\istate \not\in \exitfn_{\labelingfn\fun{\vertex}}\fun{\direction}$,
then the action is chosen by the low-level policy $\policy_{\labelingfn\fun{\vertex}, d}$,
and the next state is chosen according to the transitions of $\labelingfn\fun{\vertex}$: 
for every $\istate' \in \istates_{\labelingfn\fun{\vertex}}\setminus\set{\sreset}$, 
\begin{equation}   
\transitionfnplanner\fun{\tuple{\istate', \vertex, \varvertex} \mid \tuple{\istate, \vertex, \varvertex}, * } = \expectedsymbol{\action \sim \policy_{\labelingfn\fun{\vertex}, d}\fun{\sampledot \mid \istate}} \probtransitions\fun{\tuple{\istate', \vertex} \mid \tuple{\istate, \vertex}, \action};
\label{eq:mdp-plan-transition-i}
\end{equation}
\item \label{enum:ii}
if $\istate$ is an exit state in direction $\direction$, i.e.,
$\istate \in \exitfn_{\labelingfn\fun{\vertex}}\fun{\direction}$, the next room is entered according to the entrance function from direction $d$ and the planner needs to choose a new target direction $d'$: 
for every $\istate' \in \istates_{\labelingfn\fun{\varvertex}} \setminus \set{\sreset}$ and edge $d' = \zug{\varvertex, t} \in \out{\varvertex}$:
\begin{equation}
\transitionfnplanner\fun{\tuple{\istate', \varvertex, t} \mid \tuple{\istate, \vertex, \varvertex},\allowbreak d'} = \transitionfn\fun{\tuple{\istate', \varvertex} \mid \tuple{\istate, \vertex}, \action_{\mathit{exit}}} = \entrancefn_{\labelingfn\fun{\varvertex}}\fun{\istate' \mid d}
\label{eq:mdp-plan-transition-ii}
\end{equation}
\item \label{enum:iii} the reset state is handled exactly as in the explicit model $\mdp$: \[\transitionfnplanner\fun{\sreset \mid \tuple{\istate, \vertex, \varvertex}, *} = \linebreak \expectedsymbol{\action \sim \policy_{\labelingfn\fun{\vertex}, d}}\allowbreak \transitionfn\fun{\sreset \mid \tuple{\istate, \vertex}, \action},\] 
and $\transitionfnplanner\fun{\sampledot \mid \sreset, \action} = \mdpIplanner$ for any $\action \in \actionsplanner$;
\item \label{enum:iv} any other undefined distribution transitions deterministically to the sink state $\bot$ so that $\transitionfnplanner\fun{\bot \mid \bot, \action} = 1$ for any $\action \in \actionsplanner$. 
\end{enumerate}

\paragraph{Proper policies.}
We say that a policy $\policy$ for $\mdpplanner{}$ is \emph{proper} if the decisions of $\policy$ ensure to almost surely avoid $\bot$, i.e., $\values{\policy}{\objective\fun{T = \set{\bot}, B = \emptyset}}{\istate} = 0$ for all states $\istate \in \istates_{\planner} \setminus \set{\bot}$.
Note that \emph{improper} policies strictly consist of those which prescribe to not follow the low-level policy corresponding to the current objective
and do not select a new target direction when exiting.

\emph{In the following proofs, we restrict our attention to proper policies.}

\begin{property}[High-level objective in the MDP plan]\label{prop:mdp-plan-obj}
In $\mdpplanner{}$, the high-level objective $\objective$ translates to the reach-avoid objective $\objective\fun{\mathbf{T}, \mathbf{B}}$ where
$\mathbf{T} = \set{\tuple{\istate, \vertex, \varvertex} \in \istatesplanner  \mid \vertex \in T} $ and
$\mathbf{B} = \set{\tuple{\istate, \vertex, \varvertex} \in \istatesplanner \mid \istate \not \in B_{\labelingfn\fun{\vertex}}} $
for the high-level objective $\reach T$
so that $B_{\room}$ is the set to avoid in room $\room$.
\end{property}

\section{Proofs from \secref{high_level}}

\begin{lemma}[Equivalence of policies in the two-level model and plan]\label{lem:mdp-planner-equiv}
    There exists an equivalence between planners with memory of size $\abs{\vertices}$ in the two-level model $\hmdp$ and proper deterministic stationary policies in the MDP plan $\mdpplanner{}$ that preserves the values of their respective objective under equivalent planners and policies.
\end{lemma}
\begin{proof}
    Let $\controller$ be a planner for $\hmdp$ with memory of size $\abs{\vertices}$. 
    Let us encode $\controller$ as a finite Mealy machine whose inputs are graph vertices $\vertices$ and outputs are directions, i.e., $\controller = \tuple{\policystates, \mealyaction{\controller}, \mealyupdate{\controller}, q_0}$ where $\policystates$ is a set of memory states with $\abs{\policystates} = \abs{\vertices}$, $\mealyaction{\controller} \colon \vertices \times \policystates \to \edges$ is the next action function, $\mealyupdate{\controller}\colon \vertices \times \policystates \times \edges \to \policystates$ is the memory update function, and $q_0$ is the initial memory state.

    Let us consider the two-level controller $\tuple{\controller, \planner}$ as a policy in the explicit MDP $\mdp$.
    Since $\controller$ is a planner, we require that
    \begin{enumerate}
        \item $\mealyaction{\controller}\fun{v_0, q_0} = d_1$, and
        \item if $\mealyaction{\controller}\fun{\vertex, q} = d$, then $d \in \out{\vertex}$ for any $\vertex \in \vertices, q \in \policystates$.
    \end{enumerate}
    Intuitively, $\mealyaction{\controller}$ chooses the direction to follow in the current room based on the current memory state $q$, and $\mealyupdate{\controller}$ describes how to update the memory, based on the current room, the current memory state, and the direction chosen.
    By definition of the two-level controller $\tuple{\controller, \planner}$ (see \secref{problem}), $\mealyaction{\controller}$ is used at each time step in the current room, to know which low-level policy to execute, and $\mealyupdate{\controller}$ is triggered once an exit state is reached, to switch to the next memory state that will determine the direction to follow in the next room.

    Then, $\pathdistribution{\mdp}{\tuple{\controller, \planner}}$ is a distribution over the product of the paths of $\mdp$ and the sequence of memory states of $\controller$. 
    Following the definition of the controller $\tuple{\controller, \planner}$ (cf.\ \secref{problem}), the measure
     $\pathdistribution{\mdp}{\tuple{\controller, \planner}}$ can be obtained inductively as follows. 
    For a state $\tuple{\istate, \vertex} \in \istates$, $\pathdistribution{\mdp}{\tuple{\controller, \planner}}\fun{{\istate, \vertex, q}} = 
    \entrancefn_{\labelingfn\fun{v_0}}\fun{\istate\mid \direction_0} $
    if $v_0 = \vertex$ and $q = q_0$, and assigns a zero probability otherwise.
        The probability of a path $\mdppath = \istate_0, \vertex_0, \policystate_0, \dots, \istate_{t - 1}, \vertex_{t - 1}, \policystate_{t - 1}, \istate_t, \vertex_t, \policystate_t$ is given as follows
    \begin{enumerate}[(a)]
        \item \label{enum:a} if $\istate_{t - 1}$ is not an exit state, the low-level policy is executed in direction $\direction = \mealyaction{\controller}\fun{\vertex_{t - 1}, \policystate_{t - 1}}$ and both the current vertex and memory state must remain unchanged: \[\pathdistribution{\mdp}{\tuple{\controller, \planner}}\fun{{{\istate_0, \vertex_0, \policystate_0, \dots, \istate_{t - 1}, \vertex_{t - 1}, \policystate_{t - 1}}}} \cdot \expectedsymbol{\action \sim \latentpolicy_{\labelingfn\fun{v_t}, d}\fun{\sampledot \mid \istate}} \transitionfn\fun{\tuple{\istate_t, \vertex_t} \mid \tuple{\istate_{t - 1}, \vertex_{t - 1}}, \action}\]
        if  $s_{t - 1} \not\in \exitfn_{\labelingfn\fun{\vertex_{t - 1}}}\fun{d}$ with $d = \mealyaction{\controller}\fun{\vertex_{t - 1}, q_{t - 1}}$, $\vertex_{t} = \vertex_{t - 1}$, and $q_{t} = q_{t - 1}$;
        \item \label{enum:b} if $\istate_{t - 1}$ is an exit state in the direction prescribed in $\policystate_{t - 1}$, then this direction should point to $\vertex_{t}$ and the memory state must be updated to $\policystate_t$:
        \[\pathdistribution{\mdp}{\tuple{\controller, \planner}}\fun{{{\istate_0, \vertex_0, \policystate_0, \dots, \istate_{t - 1}, \vertex_{t - 1}, \policystate_{t - 1}}}} \cdot \entrancefn_{\labelingfn\fun{v_t}}\fun{\istate_t \mid d}\] if $\istate_{t- 1} \in \exitfn_{\labelingfn\fun{\vertex_{t - 1}}}\fun{d}$ with $\direction = \mealyaction{\controller}\fun{\vertex_{t-1}, q_{t - 1}} = \tuple{\vertex_{t - 1}, \vertex_t}$, and
        $q_t = \mealyupdate{\controller}\fun{\vertex_{t - 1}, q_{t - 1}, d}$;
        \item \label{enum:c} if $\istate_{t - 1}$ is the reset state, by Assumptions~\ref{assumption:episodic} and~\ref{assumption:same-reset}, the planner must be reset as well: \[\pathdistribution{\mdp}{\tuple{\controller, \planner}}\fun{{{\istate_0, \vertex_0, \policystate_0, \dots, \istate_{t - 1}, \vertex_{t - 1}, \policystate_{t - 1}}}} \cdot \entrancefn_{\labelingfn\fun{v_0}}\fun{\istate_t \mid d_0}\] if $\istate_{t - 1} = \sreset$, $\vertex_{t} = \vertex_0$, and $q_{t} = q_0$;
        \item zero otherwise.
    \end{enumerate}
    Notice that renaming $\policystates$ to $\vertices$ so that for all $q \in \policystates$, $q$ is changed to $u \in \vertices$ (i.e., $q \mapsto u$) whenever $\mealyaction{\controller}\fun{\vertex, q} = \tuple{\vertex, \varvertex}$ is harmless,
    since the probability measure remains unchanged. 
    From now on, we consider that $\policystates$ has been renamed to $\vertices$ in this manner.

    Now, define the relation $\equiv$ between planners in $\mdp$ and policies in $\mdpplanner{}$ as%
    \footnote{Notice the slight asymmetry induced by Mealy machines: while the policy must decide the next direction in exit states, the planner just need update its memory state (Eq.~\ref{enum:b}).}
    \begin{gather*}
    \controller \, \equiv \, \policy \quad 
        \text{if and only if}\\
        \policy\fun{\tuple{\istate, \vertex, \varvertex}} =
        \begin{cases}
            \mealyaction{\controller}\fun{\varvertex, \sampledot} \circ \mealyupdate{\controller}\fun{\vertex, \varvertex, \direction=\sampledot} \circ \mealyaction{\controller}\fun{\vertex, \varvertex}  & \text{if } \istate \in \exitfn_{\labelingfn\fun{\vertex}}\fun{\tuple{\vertex, \varvertex}} \\
            * & \text{otherwise}.
        \end{cases}
    \end{gather*}
    By construction of $\mdpplanner{}$, modulo the renaming of $\policystates$ to $\vertices$, $\pathdistribution{\mdp}{ \tuple{\controller, \planner}} = \pathdistribution{\mdpplanner{}}{\policy}$ for any $\controller$, $\policy$ in relation $\controller \, \equiv \, \policy$: condition~\ref{enum:i} is equivalent to \ref{enum:a}, condition~\ref{enum:ii} is equivalent to \ref{enum:b}, and condition~\ref{enum:iii} is equivalent to \ref{enum:c}.
    Note that the only policies $\policy$ which cannot be in relation with some planner $\controller$ are \emph{improper policies}, i.e., those choosing actions leading to the sink state $\bot$ (see condition~\ref{enum:iv}).
    Such policies are discarded by assumption.

        The result follows from the fact that, modulo the renaming of $\policystates$ to $\vertices$, planners and policies in relation $\equiv$ lead to the same probability space.
\end{proof}

\begin{theorem}\label{thm:memory-bound-tight}
    For a fixed collection of low-level policies $\planner$, a memory of size $\abs{\vertices}$ is necessary and sufficient for the planner to maximize the values of $\objective$ in the two-level model $\hmdp$.
\end{theorem}

\begin{proof}
The necessity of a memory of size $\abs{\vertices}$ is shown in Example~\ref{ex:memory-requirements}.
The sufficiency follows from Thm.~\ref{thm:mdp-planner-equiv} and the fact that a deterministic stationary policy is sufficient to maximize constrained, discounted reachability objectives in MDPs~\cite{DBLP:books/wi/Puterman94,BK08} (in particular in $\mdpplanner{}$).

To see how, let $\policy^*$ be a proper optimal deterministic stationary policy in $\mdpplanner{}$.
Note that one can always find a proper optimal policy from an improper one: if $\policy^*$ is improper, it is necessarily because a prohibited action has been chosen \emph{after} having reached the target, which can be replaced by any other action without changing the value of the objective.
Consider a planner $\controller$ in the two-level model $\hmdp$ which is equivalent to $\policy^*$ (Lem.~\ref{lem:mdp-planner-equiv}). 
Then, $\controller$ is optimal for the high-level objective in $\hmdp$ (since the probability space of the two models is the same), and $\controller$ uses a memory of size $\abs{\vertices}$.
\end{proof}

\paragraph{Succinct MDP}
In the following, we take a closer look at the construction of the succinct MDP $\mdp_{\planner}^{\graph}$. We then prove Thm.~\ref{thm:succint-mdp-plan}.

Explicitly, the transition function can be re-formalized as follows. 
Let $\vertex, \varvertex \in \vertices$, $\direction \in \edges$, and $\direction' \in \edges \cup \set{\bot}$, $\probtransitions$ defined as
\begin{gather}
\probtransitions\fun{ \direction' \mid \tuple{\vertex, \varvertex}, \direction} = 
    \begin{cases}
    \displaystyle\expectedsymbol{\istate \sim \entrancefn_{\labelingfn\fun{\varvertex}}\fun{\sampledot \mid \tuple{\vertex, \varvertex}}}{
        V^{\latentpolicy_{\labelingfn\fun{\varvertex}, \direction}}\fun{\istate}
    } & \text{if } \direction = \direction'  \in \out{\varvertex},\\[1.2em]
    1 - \probtransitions\fun{\direction \mid \tuple{\vertex, \varvertex}, \direction} & \text{if } \direction' = \bot \text{ and } \direction  \in \out{\varvertex}, \\
    1 & \text{if } \direction' = \bot \text{ and } \direction \not\in \out{\varvertex}, \text{and} \\
    0 & \text{otherwise},
    \end{cases} \notag \\
    \label{eq:compact-mdp-transition}
\end{gather}
while $\probtransitions\fun{\bot \mid \bot, \direction} = 1$.

For convenience, in the following, we assume that~$\sreset \in {B}_{\room}$ for each room~$\room$, which is consistent with the remark made in Appendix~\ref{rmk:ergodicity}.
The construction of $\mdp_{\planner}^{\graph}$ and Theorem~\ref{thm:succint-mdp-plan} can be generalized by additionally wisely handling the reset state in $\probtransitions$.

For the sake of clarity, we formally restate Thm.~\ref{thm:succint-mdp-plan}:
\begin{theorem}[Value equality in the succinct model]\label{thm:compact-mdp-plan}
    Let $\tuple{\controller, \planner}$ be a  hierarchical controller for $\hmdp$ with a $\abs{\vertices}$-memory planner $\controller$. 
    Denote by $V_{\mdpplanner{}}^{\policy}\fun{\objective}$ the initial value of $\mdpplanner{}$ running under a policy $\policy$ equivalent\footnote{cf.~Lemma~\ref{lem:mdp-planner-equiv}.} to $\controller$ in $\mdpplanner{}$ for the reach-avoid objective $\objective$ of Property~\ref{prop:mdp-plan-obj}.
    Moreover, denote by $V^{\controller}_{\mdp_{\planner}^{\graph}}\fun{\eventually T}$ the initial value obtained in $\mdp_{\planner}^{\graph}$ when the agent follows the decisions of $\controller$ for the reachability objective to states of the set $\vertices \times T$, i.e., the reach-avoid objective $\objective\fun{\vertices \times T, \emptyset}$.
    Then, assuming $\vertex_0 \not\in T$ (the case where $\vertex_0 \in T$ is trivial),
    \[V^{\policy}_{\mdpplanner{}}\fun{\objective} = V^{\controller}_{\mdp_{\planner}^{\graph}}\fun{\eventually T}.\]
\end{theorem}
\begin{proof}
Given any MDP $\mdp = \mdptuple$, we start by recalling the definition of the value function of any reach-avoid objective of the form 
$\reachavoid{T}{B}$ with $T, B \subseteq \istates$
for a discount factor $\discount \in \mathopen(0, 1\mathclose)$ and a policy $\policy$:
\begin{equation}\label{eq:val-reach-avoid}
\val{\policy}{\objective}{\discount} = \expectedsymbol{\mdppath \sim \pathdistribution{\mdp}{\policy}}\left[  \sup_{i \geq 0} \discount^{i} \condition{\istate_i \in{T}} \cdot \condition{\forall j \leq i, \, \istate_j \not\in B} \right],
\end{equation}
where $\istate_i$ denotes the $i^{\text{th}}$ state of $\mdppath$.
Intuitively, this corresponds to the expected value of the discount scaled to the time step of 
 \emph{the first visit of the set $T$}, ensuring that the set of bad states $B$ is not encountered before this first visit.

First, notice that the reach-avoid property can be merely reduced to a simple reachability property by making absorbing the states of $B$~\cite{BK08}.
Precisely, write $\selfloop{\mdp}{B}$ for the MDP $\mdp$ where we make all states from $B$ absorbing, i.e., where $\probtransitions$ is modified so that $\probtransitions\fun{\istate \mid \istate, \action} = 1$ for any $\istate \in B$ and $\action \in \actions$.
Then, one can get rid of the indicator $\condition{\forall j \leq i, \, \istate_j \not\in B}$ in Eq.~\eqref{eq:val-reach-avoid} by considering infinite paths of $\selfloop{\mdp}{B}$:
\begin{align*}
\val{\policy}{\objective}{\discount} &= \expectedsymbol{\mdppath \sim \pathdistribution{\mdp}{\policy}}\left[  \sup_{i \geq 0} \discount^{i} \condition{\istate_i \in {T}} \cdot \condition{\forall j \leq i, \, \istate_j \not\in B} \right] \\
&= \expectedsymbol{\mdppath \sim \pathdistribution{\selfloop{\mdp}{B}}{\policy}}\left[  \sup_{i \geq 0} \discount^{i} \condition{\istate_i \in {T}} \right].
\end{align*}

Second, define
\[\textit{Paths}^{\,\textit{fin}}_{\reach T} = \set{\mdppath = \istate_0, \istate_1, \dots, \istate_{i} \mid \istate_i \in T \text{ and } s_j \not\in T \text{ for all } j < t}\]
as the set of finite paths that end up in $T$, with $T$ being visited for the first time.
Then, on can get rid of the supremum of Eq.~\eqref{eq:val-reach-avoid} follows:
\begin{align}
\val{\policy}{\objective}{\discount} &= \expectedsymbol{\mdppath \sim \pathdistribution{\selfloop{\mdp}{B}}{\policy}}\left[  \sup_{i \geq 0} \discount^{i} \condition{\istate_i \in {T}} \right] \notag \\
&= \expectedsymbol{\mdppath \sim \pathdistribution{\selfloop{\mdp}{B}}{\policy}}\left[  \sum_{t = 0}^{\infty} \discount^t \cdot \condition{\textit{pref}\,\fun{\mdppath, t} \in \textit{Paths}^{\,\textit{fin}}_{\reach T}} \right], \label{eq:value-fn-pathfinT}
\end{align}
where $\textit{pref}\,\fun{\mdppath, t} = \istate_0, \istate_1, \dots, \istate_t$ yields the prefix of $\mdppath = \istate_0, \istate_1, \dots$ which ends up in the $t^{\text{th}}$ state $\istate_t$.
The attentive reader may have noticed that the resulting expectation can be seen as the expectation of a discounted cumulative reward signal (or a \emph{discounted return}, for short), where a reward of one is incurred when visiting $T$ \emph{for the first time}. 
Taking it a step further, define the reward function
\[
\reward\fun{\istate, \action, \istate'} = \begin{cases}
{1 - \discount} & \text{if } \istate \in T, \text{ and}\\
0 & \text{otherwise}.
\end{cases}
\]
Then, the value function can be re-written as 
\begin{align*}
    \val{\policy}{\objective}{\discount} &= \expectedsymbol{\mdppath \sim \pathdistribution{\selfloop{\mdp}{B}}{\policy}}\left[  \sum_{t = 0}^{\infty} \discount^t \cdot \condition{\textit{pref}\,\fun{\mdppath, t} \in \textit{Paths}^{\,\textit{fin}}_{\reach T}} \right] \\
    &= \expectedsymbol{\mdppath \sim \pathdistribution{\selfloop{\mdp}{T \cup B}}{\policy}}\left[  \sum_{t = 0}^{\infty} \discount^t \cdot r_t \right].
\end{align*}
For any state $\istate \in T$, notice that since $T$ is absorbing in $\selfloop{\mdp}{T \cup B}$, 
\begin{align}
    \valinit{\policy}{\objective}{\discount}{\istate} &= 
    1. \label{eq:val-T-one}
\end{align}
It is folklore that the discounted return is the solution of the Bellman equation 
$
    \valinit{\policy}{\objective}{\discount}{\istate} = \discount \expectedsymbol{\action \sim \policy\fun{\sampledot \mid \istate}}\expected{\istate' \sim \probtransitions\fun{\sampledot \mid \istate, \action}}{\reward\fun{\istate, \action, \istate'} \cdot \valinit{\policy}{\objective}{\discount}{\istate'} }
$
for any $\istate \in \istates$~\cite{DBLP:books/wi/Puterman94}.
In particular, considering the reach-avoid objective $\objective$, we have by Eq.~\eqref{eq:val-T-one}
\begin{equation*}
    \valinit{\policy}{\objective}{\discount}{\istate} = \begin{cases}
        \discount \expectedsymbol{\action \sim \policy\fun{\sampledot \mid \istate}}\expected{\istate' \sim \probtransitions\fun{\sampledot \mid \istate, \action}}{ \valinit{\policy}{\objective}{\discount}{\istate'} } & \text{if } \istate \notin T \cup B, \\
        1 & \text{if } \istate \in T \setminus B, \text{ and}\\
        0 & \text{otherwise, when } \istate \in B.
    \end{cases}
\end{equation*}
Now, let us consider the values of the MDP plan $\mdpplanner{}$
for the reach-avoid objective $\objective\fun{\mathbf{T}, \mathbf{B}}$ where
$\mathbf{T} = \set{\tuple{\istate, \vertex, \varvertex} \mid \vertex \in T}$ and
$\mathbf{B} = \set{\tuple{\istate, \vertex, \varvertex} \mid \istate \not \in B_{\labelingfn\fun{\vertex}}}$
for the high-level objective $\reach T$ and set of low-level objectives $\set{\objective_{R}^{\direction} \colon \room \in \rooms, \direction \in \directions_{\room}}$ so that $B_{\room}$ is the set of states to avoid in room $\room$.
Fix a $\abs{\vertices}$-memory two-level controller $\policy = \tuple{\controller, \planner}$ in for two-level model $\hmdp$ (which is compliant with $\mdpplanner{}$, see Thm.~\ref{thm:mdp-planner-equiv} and the related proof).
We take a close look to the value of each state in $\mdpplanner{}$ by following the same structure as we used for the definition of $\mdpplanner{}$ (cf.~\secref{high_level}).
For the sake of presentation, given any pair of vertices $\vertex, \varvertex \in \vertices$,  we may note $\tuple{\sreset, \vertex, \varvertex}$ to refer to the (unified, cf.~Assumption~\ref{assumption:same-reset}) reset state $\sreset \in \istatesplanner$.
Given a state $\tuple{\istate, \vertex, \varvertex} \in \istatesplanner$ with direction $\direction = \tuple{\vertex, \varvertex}$,
\begin{enumerate}[(i)]
    \item \label{it:x:1} if $\istate$ is not an exit state, i.e., if $\istate \not\in \exitfn_{\labelingfn\fun{\vertex}}\fun{\direction}$,
    then
    \begin{align*}
    &\valinit{\policy}{\objective}{\discount}{\tuple{\istate, \vertex, \varvertex}} \\
    =&\discount \expected{\tuple{\istate', \vertex, \varvertex} \sim \probtransitions_{\planner}\fun{\sampledot \mid \tuple{\istate, \vertex, \varvertex}, *}}{ \valinit{\policy}{\objective}{\discount}{\tuple{\istate', \vertex, \varvertex}}} \tag{by Eq.~\eqref{eq:mdp-plan-transition-i}} \\
    =& \discount \sum_{\istate' \in \istates_{\labelingfn\fun{\vertex}}} \probtransitions_{\planner}\fun{\tuple{\istate', \vertex, \varvertex} \mid \tuple{\istate, \vertex, \varvertex}, *} \cdot \valinit{\policy}{\objective}{\discount}{\tuple{\istate', \vertex, \varvertex}}\\
    =& \discount \sum_{\istate' \in \istates_{\labelingfn\fun{\vertex}}} \sum_{\action \in \actions_{\labelingfn\fun{\vertex}}} \latentpolicy_{\labelingfn\fun{\vertex}, \direction}\fun{\action \mid \istate} \cdot \probtransitions_{\labelingfn\fun{\vertex}}\fun{\istate' \mid \istate, \action} \cdot \valinit{\policy}{\objective}{\discount}{\tuple{\istate', \vertex, \varvertex}};
    \end{align*}
    \item \label{it:x:2} if $\istate$ is an exit state in the direction $\direction$, i.e., $\istate \in \exitfn_{\labelingfn\fun{\vertex}}\fun{\direction}$, given the direction chosen by the planner $d' = \controller\fun{\vertex, \varvertex} = \tuple{\varvertex, t}$ for some neihbor $t \in \neighbors{\varvertex}$, we have
    \begin{align}
        &\valinit{\policy}{\objective}{\discount}{\tuple{\istate, \vertex, \varvertex}} \notag \\
        =& \discount \expected{\tuple{\istate', \varvertex, t} \sim \probtransitions_{\planner}\fun{\sampledot \mid \tuple{\istate, \vertex, \varvertex}, \direction'}}{\valinit{\policy}{\objective}{\discount}{\tuple{\istate', \varvertex, t}}} \notag\\
        =& \discount \expected{\istate' \sim \entrancefn_{\labelingfn\fun{\varvertex}}\fun{\sampledot \mid \direction}}{\valinit{\policy}{\objective}{\discount}{\tuple{\istate', \varvertex, t}}}\tag{by Eq.~\eqref{eq:mdp-plan-transition-ii}} \\
        =& \discount \sum_{\istate' \in \istates_{\labelingfn\fun{\varvertex}}} \entrancefn_{\labelingfn\fun{\varvertex}}\fun{\istate' \mid \direction} \cdot {\valinit{\policy}{\objective}{\discount}{\tuple{\istate', \varvertex, t}}}; \label{eq:val-entrance-fn}
    \end{align}
    \item if $\vertex$ is the target, i.e., $\vertex \in T$, $\valinit{\policy}{\objective}{\discount}{\istate} = 1$; and
    \item otherwise, when $\istate$ is a bad state, i.e., $\istate \in B_{\labelingfn\fun{\vertex}}$, $\valinit{\policy}{\objective}{\discount}{\istate} = 0$.
\end{enumerate}
Take $\room = \labelingfn\fun{\vertex}$.
By~\ref{it:x:1}~and~\ref{it:x:2}, when $\istate$ is not an exit state, i.e., $\istate \not\in \exitfn_{\labelingfn\fun{\vertex}}\fun{\direction}$, we have
\begin{align*}
\valinit{\policy}{\objective}{\discount}{\tuple{\istate, \vertex, \varvertex}} = 
\sum_{\istate_0, \istate_1, \dots, \istate_i \in \textit{Path}^{\,\textit{fin}}_{{\objective_{R}^{\direction}}}} \discount^{i} \pathdistribution{\room_{\istate}}{\latentpolicy_{\room, \direction}}
\fun{\istate_0, \istate_1, \dots, \istate_i} \cdot \valinit{\policy}{\objective}{\discount}{\tuple{\istate_i, \vertex, \varvertex}},
\end{align*}
so that \[\textit{Path}^{\,\textit{fin}}_{\objective_\room^{\direction}} = \textit{Path}^{\,\textit{fin}}_{\reach \exitfn_{\room}\fun{\direction}} \setminus \set{\mdppath = \istate_0, \istate_1, \dots, \istate_n \mid \exists 1 \leq i \leq n, \, \istate_i \in B_\room}, \]
where we denote by $\room_{\istate}$ the room $\room$ where we change the initial distribution by the Dirac $\mdpI_{\room}\fun{\istate_0} = \condition{\istate_0 = \istate}$,
and
$\pathdistribution{\room_\istate}{\latentpolicy_{\room, \direction}}$ is the distribution over paths of $\room$ which start in state $\istate$ which is induced by the choices of the low-level latent policy $\latentpolicy_{\room, \direction}$.

Following Eq.~\eqref{eq:val-entrance-fn}, notice that $\valinit{\policy}{\objective}{\discount}{\tuple{\istate_{\text{exit}}, \vertex, \varvertex}} = \valinit{\policy}{\objective}{\discount}{\tuple{\istate'_{\text{exit}}, \vertex, \varvertex}}$ for any $\istate_{\text{exit}}, \istate'_{\text{exit}} \in \exitfn_{R}\fun{\direction = \tuple{\vertex, \varvertex}}$: the probability of going to the next room $R' = \labelingfn\fun{\varvertex}$ from an exit state of the current room $R$ only depends on the entrance function $\entrancefn_{R'}$ and is independent from the exact exit state which allowed to leave the current room $R$.
Therefore, we further denote by $\valinit{\policy}{\objective}{\discount}{\tuple{\sampledot, \vertex, \varvertex}}$ the value of any exit state of $\room$ in direction $\direction$, i.e., $\valinitfast{\tuple{\sampledot, \vertex, \varvertex}} = \valinitfast{\tuple{\istate_{\text{exit}}, \vertex, \varvertex}}$
for all $\istate_{\text{exit}} \in \exitfn_{\room}\fun{\direction}$.
Then, we have
\begin{align*}
    &\valinit{\policy}{\objective}{\discount}{\tuple{\istate, \vertex, \varvertex}} \\
    =& 
\sum_{\istate_0, \istate_1, \dots, \istate_i \in \textit{Path}^{\,\textit{fin}}_{\objective_{\direction}^{\room}}} \discount^{i} \pathdistribution{\room_{\istate}}{\latentpolicy_{\room, \direction}}
\fun{\istate_0, \istate_1, \dots, \istate_i} \cdot \valinit{\policy}{\objective}{\discount}{\tuple{\istate_i, \vertex, \varvertex}}\\
    =& 
\sum_{\istate_0, \istate_1, \dots, \istate_i \in \textit{Path}^{\,\textit{fin}}_{\objective_{\room}^{\direction}}} \discount^{i} \pathdistribution{\room_\istate}{\latentpolicy_{\room, \direction}}
\fun{\istate_0, \istate_1, \dots, \istate_i} \cdot \valinit{\policy}{\objective}{\discount}{\tuple{\sampledot, \vertex, \varvertex}}\\
=& \valinitfast{\tuple{\sampledot, \vertex, \varvertex}} \cdot \sum_{\istate_0, \istate_1, \dots, \istate_i \in \textit{Path}^{\,\textit{fin}}_{\objective_{\room}^{\direction}}} \discount^{i} \pathdistribution{\room_{\istate}}{\latentpolicy_{\room, \direction}}
\fun{\istate_0, \istate_1, \dots, \istate_i}\\
=& \valinitfast{\tuple{\sampledot, \vertex, \varvertex}} \cdot \valinit{\latentpolicy_{\direction, \room}}{\discount}{\objective_{\room}^{\direction}}{\istate} \tag{by Eq.~\eqref{eq:value-fn-pathfinT}},
\end{align*}
where $\valinit{\latentpolicy_{\room, \direction}}{\discount}{\objective_{\room}^{\direction}}{\istate}$ denotes the value of the reach-avoid objective $ \objective_{\room}^{\direction} = \objective\fun{\exitfn_{\room}\fun{\direction}, B_R}$ in the room $\room$ from state $\istate \in \room$. 
    Then, by~\ref{it:x:2}, assuming $\vertex \not\in G$, we have
    \begin{enumerate}
        \item if $\varvertex \not \in G$,\label{enum:1-value-plan}
    \begin{align}
        &\valinitfast{\tuple{\sampledot, \vertex, \varvertex}} \notag\\
        =& \discount \cdot {\sum_{\istate' \in \istates_{\labelingfn\fun{\varvertex}}} \entrancefn_{\labelingfn\fun{\varvertex}}\fun{\istate' \mid \direction=\tuple{\vertex, \varvertex}} \cdot \valinitfast{\tuple{\istate', \controller\fun{\vertex, \varvertex}}}} \label{eq:value-plan-init}\\
        =& \discount \cdot {\sum_{\istate' \in \istates_{\labelingfn\fun{\varvertex}}} \entrancefn_{\labelingfn\fun{\varvertex}}\fun{\istate' \mid \direction=\tuple{\vertex, \varvertex}} \cdot 
        \valinit{\latentpolicy_{\labelingfn\fun{\varvertex}, \controller\fun{\vertex, \varvertex}}}{\objective_{\room}^{\direction}}{\discount}{\istate}   
        \cdot \valinitfast{\tuple{\sampledot, \controller\fun{\vertex\cdot \varvertex}}}} \notag \\
        =& \discount \cdot \probtransitions\fun{\controller\fun{\vertex, \varvertex} \mid \tuple{\vertex, \varvertex}, \controller\fun{\vertex, \varvertex}} \cdot \valinitfast{\tuple{\sampledot, \controller\fun{\vertex\cdot \varvertex}}} \tag{where $\probtransitions$ is the transition function of $\mdp^{\graph}_{\planner}$, see~Eq.~\eqref{eq:succint-mdp-transition}}
    \end{align}
    \item \label{enum:2-value-plan} if $\varvertex \in G$, 
    \begin{align*}
        &\valinitfast{\tuple{\sampledot, \vertex, \varvertex}}\\
        =& \discount \cdot {\sum_{\istate' \in \istates_{\labelingfn\fun{\varvertex}}} \entrancefn_{\labelingfn\fun{\varvertex}}\fun{\istate' \mid \direction=\tuple{\vertex, \varvertex}} \cdot \valinitfast{\tuple{\istate', \controller\fun{\vertex, \varvertex}}}}\\
        =& \discount \cdot {\sum_{\istate' \in \istates_{\labelingfn\fun{\varvertex}}} \entrancefn_{\labelingfn\fun{\varvertex}}\fun{\istate' \mid \direction=\tuple{\vertex, \varvertex}}} \cdot 1 \tag{since $\controller\fun{\vertex, \varvertex} = \tuple{u, t}$ for some $t \in \neighbors{u}$}\\
        =& \discount.
    \end{align*}
    \end{enumerate}
    
    Now, respectively denote by $V_{\mdpplanner{}}^{\policy}\fun{\sampledot, \objective} \coloneqq \valinitfast{\sampledot}$ and $V_{\mdp_{\planner}^{\graph}}^{\controller}\fun{\sampledot, \eventually T}$ the value functions of $\mdpplanner{}$ and $\mdp_{\planner}^{\graph}$ for the objectives $\objective$ and $\eventually T$.
    By \ref{enum:1-value-plan} and \ref{enum:2-value-plan}, and by construction of $\mdp_{\planner}^{\graph}$, we have for any pair of vertices $\vertex, \varvertex \in \vertices$ that
    \[
        V_{\mdpplanner{}}^{\policy}\fun{\tuple{\sampledot, \vertex, \varvertex}, \objective} = \discount \cdot V_{\mdp_{\planner}^{\graph}}^{\controller}\fun{\tuple{\vertex, \varvertex}, \reach T},
    \]
    On the one hand, notice that, by construction of $\mdp_{\planner}^{\graph}$, we have for any pair of vertices $\tuple{\vertex, \varvertex} \in \edges$ that the initial values
       $ V_{\mdp_{\planner}^{\graph}}^{\controller}\fun{\reach T}$ are $\expectedsymbol{\direction \sim \mdpI} V_{\mdp_{\planner}^{\graph}}^{\controller}\fun{\direction, \reach T} = V_{\mdp_{\planner}^{\graph}}^{\controller}\fun{\direction_0, \reach T}$.
    On the other hand, we have 
    \begin{align*}
        V_{\mdpplanner{}}^{\policy}\fun{\objective} = \expectedsymbol{\istate' \sim \entrancefn_{\labelingfn\fun{\vertex_0}}\fun{\sampledot \mid \direction_0}} V_{\mdpplanner{}}^{\policy}\fun{\tuple{\istate', \controller\fun{\direction_0}}, \objective} = \nicefrac{1}{\discount} \cdot V_{\mdpplanner{}}^{\policy}\fun{\sampledot, \direction_0, \objective}
        \tag{by Eq.~\eqref{eq:value-plan-init}}.
    \end{align*}
    Then, we finally have:
    \[
        V_{\mdpplanner{}}^{\policy}\fun{\objective} = \nicefrac{1}{\discount} \cdot V_{\mdpplanner{}}^{\policy}\fun{\sampledot, \direction_0, \objective} = \nicefrac{\discount}{\discount} \cdot V_{\mdp_{\planner}^{\graph}}^{\controller}\fun{\direction_0, \reach T} = V_{\mdp_{\planner}^{\graph}}^{\controller}\fun{\reach T},
    \]
    which concludes the proof.
\end{proof}

\section{Initial Distribution Shifts: Training vs. Synthesis}\label{appendix:distribution-shifts}
\label{sec:training-vs-synthesis}
Our two-level controller construction occurs in two phases. First, we create a set of low-level policies $\planner$ by running Algorithm~\ref{algo:wae-dqn} in each room (\secref{wae-dqn}). Notably, training in each room is independent and can be executed in parallel. However, independent training introduces a challenge: an \emph{initial distribution shift} emerges when combining low-level policies using a planner. Our value bounds for a room $\room$ in direction $\direction$ depend on a loss $\transitionloss^{\room, \direction}$, computed based on the stationary distribution. This distribution may significantly change depending on a planner's choices. In this section, we address this challenge by showing, under mild assumptions on the initial distribution of each room $R$, that their transition losses $\transitionloss^{\room, \direction}$ obtained under any latent policy $\latentpolicy_{\room, \direction}$ for direction $\direction$ still guarantee to bound the gap between the values of the original and latent two-level models.

\highlight{
In the following, we first give details on this ``distribution shift,'' and then we prove Theorem~\ref{thm:lifting-guarantees} through Theorems~\ref{thm:reusable-components} and~\ref{thm:final-value-bound}.
}

\paragraph{Training rooms.}%
To construct $\planner$, we train low-level policies via Algorithm~\ref{algo:wae-dqn} {by simulating each room individually}.
Precisely, for room $\room \in \rooms$ and direction $\direction \in \directions_\room$, we train a WAE-DQN agent by considering $\room$ as episodic MDP with \emph{some} initial distribution $\mdpI_{\room}$, yielding (i)~low-level latent policy $\latentpolicy_{\room, \direction}$, (ii)~latent MDP $\latentmdp_{\room}$, and (iii)~state-embedding function $\embed_{\room}$.
Since $\latentpolicy_{\room, \direction}$ must learn to maximize the values of the objective $\objective^{\direction}_{\room}$, which asks to reach the exit state in direction $\direction$, we restart the simulation when the latter is visited.
Formally, the related training room is an episodic MDP $\room_{\direction} = \tuple{\istates_{\room}, \actions_{\room}, \probtransitions_{\room}^{\direction}, \mdpI_{\room}}$, where $\sreset \in \istates_{\room}$, $\probtransitions_{\room}^{\direction}\fun{\sampledot \mid \istate, \action}= \probtransitions_{\room}\fun{\sampledot \mid \istate, \action}$ when $\istate \not\in \exitfn_{\room}\fun{\direction}$, and $\probtransitions_{\room}^{\direction}\fun{\sreset \mid \istate, \action}=1$ otherwise.
We define $\latentprobtransitions_{\room}^{\direction}$ similarly for $\latentmdp_{\room}$ when the direction $\direction$ is considered.

\paragraph{Distribution shift.}%
Crucially, by considering rooms individually, a noticeable \emph{initial distribution shift} occurs when switching between training and synthesis phases. During training, there is no two-level controller, so the 
initial distribution of room~$\room$ is just~$\mdpI_{\room}$.
During synthesis, room entries and exits are determined by the distributions influenced by the choices made by the controller in the hierarchical MDP~$\hmdp$.
This implies that the induced initial distribution of each room depends on the likelihood of visiting other rooms and is further influenced by the other low-level policies.

We contend that this shift may induce significant consequences: denote by $\transitionloss^{\room, \direction}$ the transition loss of the room $\room_{\direction}$  operating under $\latentpolicy_{\room, \direction}$
and by $\transitionloss^{\controller, \planner}$ the transition loss of the two-level model $\hmdp$ operating under $\tuple{\controller, \planner}$.
Then, in the worst case, $\transitionloss^{\controller, \planner}$ and $\transitionloss^{\room, \direction}$ might be completely unrelated whatever the room $\room$ and direction $\direction$.
To see why, recall that transition losses are defined over stationary distributions of the respective models (Eq.~\ref{eq:transition-loss}). One can see this shift as a perturbation in the transition function of the rooms. Intuitively, by Assumption~\ref{assumption:episodic}, each room is almost surely entered infinitely often, meaning that such perturbations are also repeated infinitely often, possibly leading to completely divergent stationary distributions~\cite{OCinneide1993}, meaning that we loose the abstraction quality guarantees possibly obtained for each individual training room.

\paragraph{Entrance loss.}Fortunately, we claim that under some assumptions, when the initial distribution of each training room $\mdpI_{\room}$ is wisely chosen, we can still link the transition losses $\transitionloss^{\room, \direction}$ minimized in the training rooms to $\transitionloss^{\controller, \planner}$.
To provide this guarantee, the sole remaining missing component to our framework  is learning a \emph{latent entrance function}:
we define the entrance loss as
\begin{equation}
	\entranceloss = \expectedsymbol{\room, \direction \sim \stationary{\policy}} \divergence\fun{{\embed\entrancefn_{\room}\fun{\sampledot \mid \direction}, \,  \latententrancefn_{\room}\fun{\sampledot \mid \direction}}}, 
\end{equation}
where $\embed\entrancefn_{\room}\fun{\sampledot \mid \direction} = \expectedsymbol{\istate \sim \entrancefn_{\room}\fun{\sampledot \mid \direction}} \condition{\latentstate = \embed_{\room}\fun{\istate}}$, $\latententrancefn_{\room} \colon \directions_{\room} \to \distributions{\latentstates}$ is the latent entrance function, ${\policy}$ is the stationary policy in $\mdpplanner{}$ corresponding to the two-level controller $\tuple{\controller, \planner}$ where $\controller$ has a memory of size $\abs{\vertices}$,
$\divergence$ is total variation, and
and $\stationary{\policy}$ is the stationary distribution induced by $\policy$ in $\mdpplanner{}$.
The measure $\stationary{\policy}$ can also be seen as a distribution over rooms and directions chosen under the controller:%
\footnote{For simplicity, we consider here the special state $\tuple{\sreset, \vertex, \vertex_0}$ with $\tuple{\vertex, \vertex_0} = \direction_0$ as the joint reset state of the model (Assumption~\ref{assumption:same-reset}).}
\begin{align*}
\stationary{\policy}\fun{\room, \direction} \allowbreak =\allowbreak \expectedsymbol{\istate, \vertex, \varvertex \sim \stationary{\policy}} \left[\condition{\istate = \sreset, \room = \labelingfn\fun{\vertex}, \direction = \direction_0} +  \condition{\room = \labelingfn\fun{\vertex}, \direction = \tuple{\vertex, \varvertex}} \right].
\end{align*}

\begin{theorem}[Reusable RL components]\label{thm:reusable-components}
    Let $\tuple{\controller, \planner}$ be a two-level controller in $\hmdp$ where $\controller$ has finite memory of size $\abs{\vertices}$ and let $\policy$ be the equivalent stationary policy in the MDP plan $\mdpplanner{}$.
    Assume
    \begin{enumerate*}[(i)]
        \item
        $\planner$ only consists of latent policies and
        \item \label{asmp:projection-included}
        for any training room $\room \in \rooms$ and direction $\direction \in \directions_{\room}$, the projection\footnote{Formally speaking, this is the projection to $\istates_{\room}$ of the intersection of the BSCC of $\mdp_{\planner}$ operating under $\policy$ with $\istates_{\room} \times \directions_{\room}$.}of the BSCC of $\mdpplanner{}$ under $\policy$ to $\istates_{\room}$ is included in the BSCC of $\room_{\direction}$ under low-level policy~$\latentpolicy_{\room, \direction}$.
        Let
    \end{enumerate*}
    \begin{align*}
    \istates_{\room, \direction} &= \set{\tuple{\istate, \vertex, \varvertex} \in \istatesplanner \mid  \labelingfn\fun{\vertex} = \room \text{ and }\tuple{\vertex, \varvertex} = \direction},\\
    \stationary{\policy}\fun{\sreset \mid \room, \direction} &= \expected{\tuple{\istate, \vertex, \varvertex}, \action \sim \stationary{\policy}}{\probtransitions_{\planner}\fun{\sreset \mid \tuple{\istate, \vertex, \varvertex}, \action} \mid \istates_{\room, \direction}}, \text{ and}\\
    \stationary{\text{continue}}^{\min} &=1 - \max_{\room \in \rooms, \direction \in \directions} (\stationary{\policy}\fun{\sreset \mid \istates_{\room, \direction}} + \stationary{\policy}(\exitfn_{\room}\fun{\direction} \times{\set{\direction}} \mid \istates_{\room, \direction})).
    \end{align*}
    Then, there is a~$\kappa \geq 0$ with
     $   \transitionloss^{\controller, \planner} \leq  \entranceloss +  \frac{\kappa }{\stationary{\text{continue}}^{\min}} \; \expectedsymbol{\room, \direction \sim \stationary{\policy}} \transitionloss^{\room, \direction}.$
    Define
    the expected entrance function in room~$\room$ as
    \begin{equation*}
    \,\mdpI^{ \policy}_{\room}\fun{\istate} = \expectedsymbol{\dot{\istate}, \tuple{{\varvertex}, {\vertex}} \sim \stationary{\policy}}\left[ \entrancefn_{\room}\fun{\istate \mid \direction = \tuple{{\varvertex}, {\vertex}}}  \mid \dot{\istate} \in \exitfn_{\labelingfn\fun{{\varvertex}}}\fun{\tuple{{\varvertex}, {\vertex}}} \text{ and }\labelingfn\fun{\vertex} = \room \right] \text{ for any } \istate \in \istates_{\room}.
    \end{equation*}
    With $\support{P} = \set{x \in \measurableset \mid P(x) > 0}$ the support of distribution~$P$, if $\support{\mdpI_{\room}} = \support{\mdpI_{\room}^{\policy}}$, $\kappa$ can be set to the maximum probability ratio of room entry during training and synthesis:
    \begin{align*}
    \kappa = \max_{\room \in \rooms}\fun{\max_{\istate \in \support{\mdpI_{\room}}} \max\set{\frac{\mdpI_{\room}^{ \policy}\fun{\istate}}{\mdpI_{\room}\fun{\istate}}, \frac{\mdpI_{\room}\fun{\istate}}{\mdpI_{\room}^{ \policy}\fun{\istate}}}}^{\abs{\istates}}.
    \end{align*}
\end{theorem}

\begin{proof}
    For simplicity, assume that the reset state in $\istatesplanner$ is a triplet of the form $\tuple{\sreset, \vertex, \vertex_0}$ so that $\tuple{\vertex, \vertex_0} = \direction_0$ and $\exitfn_{\labelingfn\fun{\vertex_{\text{reset}}}}\fun{\direction_0} = \set{\sreset}$.
    We also may write $\embed\fun{\latentstate}$ for $\embed_{\room}\fun{\latentstate}$ when it is clear from the context that $\latentstate \in \istates_{\room}$.
    We respectively denote the marginal stationary distribution of states and directions by $\stationary{\policy}\fun{\istate} = \expectedsymbol{\istate', \vertex, \varvertex \sim \stationary{\policy}}\left[\condition{\istate = \istate'}\right]$ and 
    $\stationary{\policy}\fun{d} = \expectedsymbol{\istate, \vertex, \varvertex \sim \stationary{\policy}}\left[\condition{d = \tuple{\vertex, \varvertex}}\right]$.
    Furthermore, given a direction $d \in \edges$, we denote the conditional stationary distribution by 
    \begin{align*}
    \stationary{\policy}\fun{\istate, \action \mid \direction} &= 
    \expectedsymbol{\istate', \vertex, \varvertex \sim \stationary{\policy}}\left[ \latentpolicy_{\labelingfn\fun{\vertex}, \direction}\fun{\action \mid \embed\fun{\istate}} \cdot \condition{\istate = \istate'} \mid \set{\tuple{{\istate'}, \vertex, \varvertex} \in \istates_{\planner} \mid \tuple{\vertex, \varvertex} = \direction} \right] \\
    &=\expectedsymbol{\istate', \vertex, \varvertex \sim \stationary{\policy}}\left[ \latentpolicy_{\labelingfn\fun{\vertex}, \direction}\fun{\action \mid \embed\fun{\istate}} \cdot \condition{\istate = \istate'} \frac{\condition{d = \tuple{\vertex, \varvertex}}}{\stationary{\policy}\fun{\vertex, \varvertex}}\right]
    \end{align*}
        In the following, we also write $\transitionfn\fun{\istate' \mid \istate, \action}$ as shorthand for $\transitionfn\fun{\tuple{\istate', \vertex} \mid \tuple{\istate, \vertex}, \action}$ (the transition function of the explicit MDP of $\hmdp$) if and only if $\istate, \istate' \in \istates_{\labelingfn\fun{\vertex}}$ and $\istate \not\in \exitfn_{\labelingfn\fun{\vertex}}\fun{\direction}$ for some $\vertex \in \vertices$, $\direction \in \out{\vertex}$.
    Denote by $\latentprobtransitions_{\planner}$ the latent transition function of the \emph{latent MDP plan} $\latentmdp_{\planner}$, constructed from the collection of low-level policies $\planner$, the latent rooms $\set{\latentmdp_{\room} \colon \room \in \rooms}$, and the latent entrance functions $\set{\latententrancefn_{\room} \colon \room \in \rooms}$.
    Then:
\allowdisplaybreaks
\begin{align*}
    &\transitionloss^{\controller, \planner} \\
    =&\ \frac{1}{2} \, \expectedsymbol{\tuple{\istate, \vertex, \varvertex}, \action \sim \stationary{\policy}} 
     \norm{\embed\transitionfnplanner\fun{\sampledot \mid \tuple{\istate, \vertex, \varvertex}, \action} - \latenttransitionfnplanner\fun{\sampledot \mid \tuple{\embed\fun{\istate}, \vertex, \varvertex}, \action} }_1\\
    =&\ 
        \frac{1}{2} \, \expectedsymbol{\istate, \vertex, \varvertex \sim \stationary{\policy}} \left[\condition{\istate \neq \sreset} \, \condition{\istate \not\in \exitfn_{\labelingfn\fun{\vertex}}\fun{\tuple{\vertex, \varvertex}}} \right.
         \norm{\embed\transitionfnplanner\fun{\sampledot \mid \tuple{\istate, \vertex, \varvertex}, *} - \latenttransitionfnplanner\fun{\sampledot \mid \tuple{\embed\fun{\istate}, \vertex, \varvertex}, *} }_1 \Bigg]\\
    &\ 
    +\frac{1}{2} \, \expectedsymbol{\tuple{\istate, \vertex, \varvertex}, d' \sim \stationary{\policy}} \Bigg[ \condition{\istate \neq \sreset} \, \condition{\istate \in \exitfn_{\labelingfn\fun{\vertex}}\fun{\tuple{\vertex, \varvertex}}} 
    \norm{\embed\transitionfnplanner\fun{\sampledot \mid \tuple{\istate, \vertex, \varvertex}, d'} - \latenttransitionfnplanner\fun{\sampledot \mid \tuple{\embed\fun{\istate}, \vertex, \varvertex}, d'} }_1 \Bigg]
    \\
    &\ + \frac{1}{2} \, \expectedsymbol{\istate, \vertex, \varvertex \sim \stationary{\policy}} \left[\condition{\istate = \sreset} \, 
         \norm{\embed\transitionfnplanner\fun{\sampledot \mid \tuple{\istate, \vertex, \varvertex}, *} - \latenttransitionfnplanner\fun{\sampledot \mid \tuple{\embed\fun{\istate}, \vertex, \varvertex}, *} }_1 \right] \tag{$\policy$ is proper}
    \\
    =&\
    \frac{1}{2} \, \expectedsymbol{\istate, \vertex, \varvertex \sim \stationary{\policy}} \left[\condition{\istate \neq \sreset}\, \condition{\istate \not\in \exitfn_{\labelingfn\fun{\vertex}}\fun{\tuple{\vertex, \varvertex}}}\right.
     \left.\norm{\expectedsymbol{a \sim \latentpolicy_{\labelingfn\fun{\vertex}, \tuple{\vertex, \varvertex}}\fun{\sampledot \mid \embed\fun{\istate}}}\bigg[\embed\transitionfn\fun{\sampledot \mid \istate, \action} - \latentprobtransitions\fun{\sampledot \mid {\embed\fun{\istate}, \action}} \bigg]}_1 \right]
    \tag{by definition of $\mdpplanner{}$ \ref{enum:i}}\\
    &\ 
    +\frac{1}{2} \, \expectedsymbol{\tuple{\istate, \vertex, \varvertex}, d' \sim \stationary{\policy}} \Bigg[ \condition{\istate \neq \sreset} \, \condition{\istate \in \exitfn_{\labelingfn\fun{\vertex}}\fun{\tuple{\vertex, \varvertex}}} 
    \norm{ \embed\entrancefn_{\labelingfn\fun{\varvertex}}\fun{\sampledot \mid \tuple{\vertex, \varvertex}} - \latententrancefn_{\labelingfn\fun{\varvertex}}\fun{\sampledot \mid \tuple{\vertex, \varvertex}} }_1 \Bigg]
    \tag{by definition of $\mdpplanner{}$ \ref{enum:ii}}\\
    &\ + \frac{1}{2} \, \expectedsymbol{\istate, \vertex, \varvertex \sim \stationary{\policy}} \left[\condition{\istate = \sreset} \,  \norm{ \embed\entrancefn_{\labelingfn\fun{\vertex_0}}\fun{\sampledot \mid d_0} - \latententrancefn_{\labelingfn\fun{\vertex_0}}\fun{\sampledot \mid d_0} }_1 \right]
    \tag{by definition of $\mdpplanner{}$ \ref{enum:iii}}\\
    = &\
    \frac{1}{2} \, \expectedsymbol{\istate, \vertex, \varvertex \sim \stationary{\policy}} \left[\condition{\istate \neq \sreset}\, \condition{\istate \not\in \exitfn_{\labelingfn\fun{\vertex}}\fun{\tuple{\vertex, \varvertex}}}\right.
     \left.\norm{\expectedsymbol{a \sim \latentpolicy_{\labelingfn\fun{\vertex}, \tuple{\vertex, \varvertex}}\fun{\sampledot \mid \embed\fun{\istate}}}\bigg[\embed\transitionfn\fun{\sampledot \mid \istate, \action} - \latentprobtransitions\fun{\sampledot \mid {\embed\fun{\istate}, \action}} \bigg]}_1 \right]
    \\
    &\ + \frac{1}{2} \, \expectedsymbol{\room, \direction \sim \stationary{\policy}}  \norm{ \embed\entrancefn_{\room}\fun{\sampledot \mid \direction} - \latententrancefn_{\room}\fun{\sampledot \mid \direction} }_1  \\
    =&\
    \frac{1}{2} \, \expectedsymbol{\istate, \vertex, \varvertex \sim \stationary{\policy}} \left[\condition{\istate \neq \sreset}\, \condition{\istate \not\in \exitfn_{\labelingfn\fun{\vertex}}\fun{\tuple{\vertex, \varvertex}}}\right.
     \left.\norm{\expectedsymbol{a \sim \latentpolicy_{\labelingfn\fun{\vertex}, \tuple{\vertex, \varvertex}}\fun{\sampledot \mid \embed\fun{\istate}}}\bigg[\embed\transitionfn\fun{\sampledot \mid \istate, \action} - \latentprobtransitions\fun{\sampledot \mid {\embed\fun{\istate}, \action}} \bigg]}_1 \right]
    \\
    &\ + \entranceloss \\
    \leq &\
    \frac{1}{2} \, \expectedsymbol{\istate, \vertex, \varvertex \sim \stationary{\policy}} \expectedsymbol{a \sim \latentpolicy_{\labelingfn\fun{\vertex}, \tuple{\vertex, \varvertex}}\fun{\sampledot \mid \embed\fun{\istate}}}\left[\condition{\istate \neq \sreset}\, \condition{\istate \not\in \exitfn_{\labelingfn\fun{\vertex}}\fun{\tuple{\vertex, \varvertex}}}\right.
     \left. \norm{\embed\transitionfn\fun{\sampledot \mid \istate, \action} - \latentprobtransitions\fun{\sampledot \mid {\embed\fun{\istate}, \action}} }_1 \right]+ \entranceloss
    \tag{Jensen's inequality}\\
    = &\
    \frac{1}{2} \, \expectedsymbol{\direction \sim \stationary{\policy}} \expectedsymbol{\istate, \action \sim \stationary{\policy}\fun{\sampledot \mid \direction}}\left[\condition{\istate \neq \sreset}\, \condition{\istate \not\in \exitfn_{\labelingfn\fun{\vertex}}\fun{{\direction}}}\right.
     \left. \norm{\embed\transitionfn\fun{\sampledot \mid \istate, \action} - \latentprobtransitions\fun{\sampledot \mid {\embed\fun{\istate}, \action}} }_1 \right]+ \entranceloss \tag{$\star$}\label{eq:proof-star}
\end{align*}
Now, let $\direction = \tuple{\vertex, \varvertex} \in \edges$ be a \emph{target direction} for the room $\room = \labelingfn\fun{\vertex}$. 
We consider the room $\room$ as an episodic MDP (cf.\ Assumption~\ref{assumption:episodic}) where (i)~the initial distribution corresponds to the expected entrance probabilities in $\room$ under $\policy$: for any $\istate \in \istates_{\room}$
\begin{equation*}
  \mdpI^{\policy}_{\room}\fun{\istate} = \expectedsymbol{\dot{\istate}, \tuple{\dot{\varvertex}, \dot{\vertex}} \sim \stationary{\policy}}\left[ \entrancefn_{\room}\fun{\istate \mid \direction_I = \tuple{\dot{\varvertex}, \dot{\vertex}}}  \mid \dot{\istate} \in \exitfn_{\labelingfn\fun{\dot{\varvertex}}}\fun{\tuple{\dot{\varvertex}, \dot{\vertex}}}\text{ and } \dot{\vertex} = \vertex \right] 
\end{equation*}
(where $\direction_I$ is the direction from which $\room$ is entered);
and (ii)~the room is reset when an exit state in direction $d$ is visited: for any $\istate, \istate' \in \istates_{\room}$, $\action \in \actions_{\room}$,
\begin{equation}
    \probtransitions^{\direction, \policy}_{\room}\fun{\istate' \mid  \allowbreak\istate, \action} = 
    \begin{cases}
        1 & \text{if $\istate' = \sreset$ and $\istate \in \exitfn_{\room}\fun{\direction}$}, \\
        \mdpI^{ \policy}_{\room}\fun{\istate'} & \text{if $\istate = \sreset$, and} \\
         \probtransitions^{}_{\room}(\istate' \mid \istate, \action) & \text{otherwise.}
    \end{cases} \label{eq:room-transition}
\end{equation}
We call the resulting MDP the \emph{individual room} version of $\room$ that we denote by $\room_{\direction, \policy}$.
The stationary distribution of the room $\room_{\direction, \policy}$ for the low-level policy $\latentpolicy_{R, \direction}$ is $\stationaryroom$.
Observe that $\stationaryroom$ is over $\istates_{\room}$, which includes the reset state $\sreset$, while $\stationary{\policy}\fun{\sampledot \mid \direction}$ is over the exact same state space but without the reset state (since the reset state is a special state outside $\room$, shared by all the rooms in the two-level model; cf.\ Assumption~\ref{assumption:same-reset} and the definition of $\mdpplanner{}$).
\begin{figure}[t]
    \centering
    \includegraphics[width=.8\linewidth]{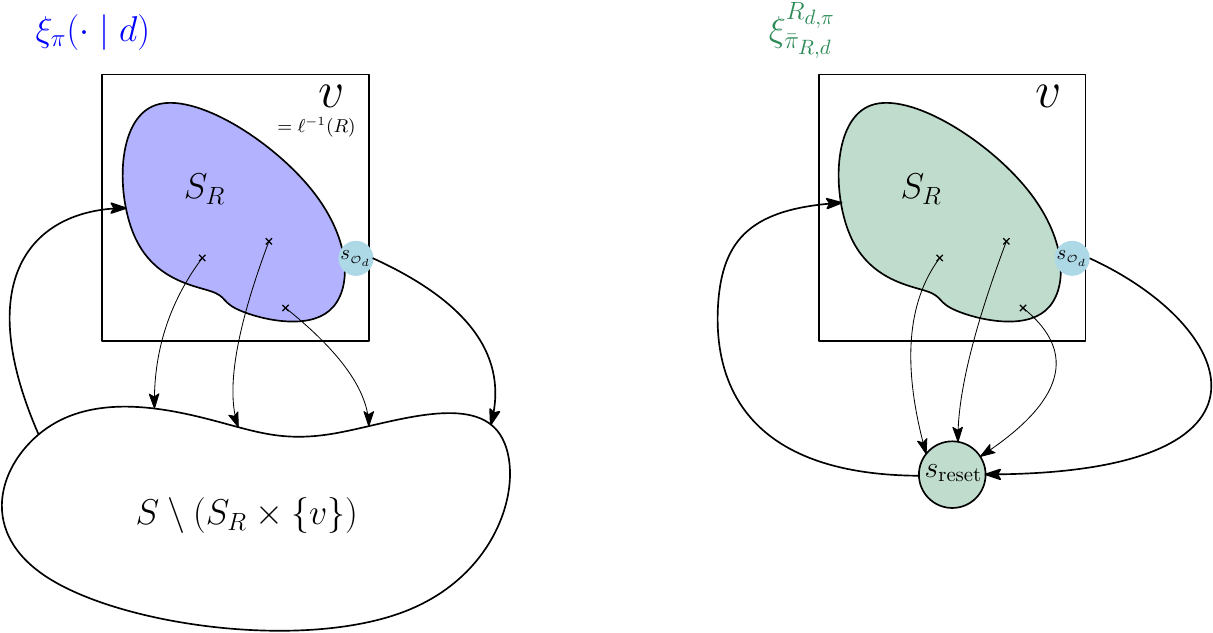}
    \caption{Room $\room = \labelingfn\fun{\vertex}$ in the two-level model (left) and the same room taken individually (right). 
    Both distributions $\stationary{\policy}\fun{\sampledot \mid \direction}$ and $\stationaryroom$ correspond to the limiting distributions over $\istates_{\room}$ when $\latentpolicy_{\room, \direction}$ is executed in $\room$.
    The sole difference remains in the fact that the reset is considered outside $\room$ in the two-level model (Assumption~\ref{assumption:same-reset}) while it is considered to be part of the state space when $\room$ is taken individually (Assumption~\ref{assumption:episodic}).}
    \label{fig:room-distribution}
\end{figure}
Furthermore, notice that, modulo this reset state, the two distributions are the same (see Fig.~\ref{fig:room-distribution}): they both consist of the limiting distribution over $\istates_{\room}$  when $\latentpolicy_{\room, \direction}$ is executed in $\room$.
All the transition distributions remain the same, except those of the exit states: in the two-level model $\hmdp$, every state $\istate \in \exitfn_{\room}\fun{\direction=\tuple{\vertex, \varvertex}}$ transitions to $\varvertex$ deterministically, while in the individual room $\room_{\direction, \policy}$, they transition to the reset state deterministically.
Still, in both cases, $\room$ is entered and exited with the same probability
(respectively from and to $\fun{\istates \setminus \istates_{\room} \times \set{\vertex}}$ in $\hmdp$ and $\sreset$ in the individual room $\room_{\direction, \policy}$).
Therefore, we have:
\begin{equation}
\stationary{\policy}\fun{\istate \mid \direction} = \stationaryroom\fun{\istate \mid \istates_{\room} \setminus \set{\sreset}} = \frac{\stationaryroom\fun{\istate} \cdot \condition{\istate \neq \sreset}}{1 - \stationaryroom\fun{\sreset}}.
\end{equation}

Instead of sampling from $\stationary{\policy}\fun{\istate \mid \direction}$ in Eq.~\eqref{eq:proof-star}, we would rather like to sample from  the distribution of the individual room $\stationaryroom\fun{\istate \mid \istates_{\room} \setminus \set{\sreset}}$.
We have:
\allowdisplaybreaks
\begin{align}
&\ \expectedsymbol{\istate, \action \sim \stationary{\policy}\fun{\sampledot \mid \direction}}\left[\condition{\istate \neq \sreset}\, \condition{\istate \not\in \exitfn_{\labelingfn\fun{\vertex}}\fun{{\direction}}}
 \norm{\embed\transitionfn\fun{\sampledot \mid \istate, \action} - \latentprobtransitions\fun{\sampledot \mid {\embed\fun{\istate}, \action}} }_1 \right] \notag \\
=&\ 
\sum_{\istate \in \istates } \sum_{\action \in \actions} \left[\stationary{\policy}\fun{\istate \mid \direction} \, \latentpolicy_{\room, {\direction}}\fun{\action \mid \embed\fun{\istate}} \, \right. 
\left. \condition{\istate \neq \sreset}\, \condition{\istate \not\in \exitfn_{\labelingfn\fun{\vertex}}\fun{{\direction}}} \, \norm{\embed\transitionfn\fun{\sampledot \mid \istate, \action} - \latentprobtransitions\fun{\sampledot \mid {\embed\fun{\istate}, \action}} }_1 \right] \label{eq:from-cond-prob}
\\
=&\ 
\sum_{\istate \in \istates } \sum_{\action \in \actions} \left[\frac{\stationaryroom\fun{\istate} \condition{\istate \neq \sreset}}{1 - \stationaryroom\fun{\sreset}} \, \latentpolicy_{\room, {\direction}}\fun{\action \mid \embed\fun{\istate}} \, \right. 
\left. \condition{\istate \neq \sreset}\, \condition{\istate \not\in \exitfn_{\labelingfn\fun{\vertex}}\fun{{\direction}}} \, \norm{\embed\transitionfn_{\room}\fun{\sampledot \mid \istate, \action} - \latentprobtransitions_{\room}\fun{\sampledot \mid {\embed\fun{\istate}, \action}} }_1 \right]  
 \label{eq:to-stationary}\\
=&\ \expectedsymbol{\istate, \action \sim \stationaryroom}\!\!\left[\frac{\condition{\istate \neq \sreset}}{1 - \stationaryroom\fun{\sreset}} \condition{\istate \not\in \exitfn_{\labelingfn\fun{\vertex}}\fun{{\direction}}}
 \norm{\embed\transitionfn_{\room}\fun{\sampledot \mid \istate, \action} - \latentprobtransitions_{\room}\fun{\sampledot \mid {\embed\fun{\istate}, \action}} }_1 \right] \notag
\end{align}
Notice that we can pass from Eq.~\eqref{eq:from-cond-prob} to~\eqref{eq:to-stationary} because we only consider states $\istate \neq \sreset$ and $\istate \not\in \exitfn_{\labelingfn\fun{\vertex}}\fun{\direction}$.
States that do not satisfy both constraints are the only ones for which  
$\probtransitions\fun{\sampledot \mid \istate, \action}$ differs from $\probtransitions_{\room}^{\direction, \policy}\fun{\sampledot \mid \istate, \action}$ (Eq.~\ref{eq:room-transition}).
Furthermore, in that case, we have $\probtransitions_{\room}^{\direction, \policy}\fun{\sampledot \mid \istate, \action} = \probtransitions_{\room}\fun{\sampledot \mid \istate, \action}$.
Then we have:
\allowdisplaybreaks
\begin{align*}
&\ \expectedsymbol{\istate, \action \sim \stationaryroom}\!\!\left[\frac{\condition{\istate \neq \sreset}}{1 - \stationaryroom\fun{\sreset}} \condition{\istate \not\in \exitfn_{\labelingfn\fun{\vertex}}\fun{{\direction}}}
 \norm{\embed\transitionfn_{\room}\fun{\sampledot \mid \istate, \action} - \latentprobtransitions_{\room}\fun{\sampledot \mid {\embed\fun{\istate}, \action}} }_1 \right]\\
= &\ \expectedsymbol{\istate, \action \sim \stationaryroom}\!\!\left[\frac{\condition{\istate \neq \sreset}}{1 - \stationaryroom\fun{\sreset}} 
 \norm{\embed\transitionfn_{\room}^{\direction}\fun{\sampledot \mid \istate, \action} - \latentprobtransitions_{\room}^{\direction}\fun{\sampledot \mid {\embed\fun{\istate}, \action}} }_1 \right] \tag{by definition of $\probtransitions_{\room}^{\direction}$ and $\latentprobtransitions_{\room}^{\direction}$}\\
= &\ \frac{1}{1 - \stationaryroom\fun{\sreset}} \expectedsymbol{\istate, \action \sim \stationaryroom} \left[ \condition{\istate \neq \sreset} \,
 \norm{\embed\transitionfn_{\room}^{\direction}\fun{\sampledot \mid \istate, \action} - \latentprobtransitions_{\room}^{\direction}\fun{\sampledot \mid {\embed\fun{\istate}, \action}} }_1 \right] \\
 \leq &\ \frac{1}{1 - \stationaryroom\fun{\sreset}} \expectedsymbol{\istate, \action \sim \stationaryroom} 
 \norm{\embed\transitionfn_{\room}^{\direction}\fun{\sampledot \mid \istate, \action} - \latentprobtransitions_{\room}^{\direction}\fun{\sampledot \mid {\embed\fun{\istate}, \action}} }_1 \\
\end{align*}
Assuming that the projection of the BSCC of $\mdpplanner{}$ under $\policy$ to $\istates_{\room}$ is included in the BSCC of $\room$ when it operates under $\latentpolicy_{\room, \direction}$, we have that $\support{\stationaryroom} \subseteq  \support{\stationaryroomtrain}$, where $\stationary{\latentpolicy_{\room, \direction}}^{\room}$ denotes the stationary distribution of the training room $\room_{\direction}$ under the latent policy $\latentpolicy_{\room, \direction}$.
Then:
\begin{align*}
&\ \frac{1}{1 - \stationaryroom\fun{\sreset}} \expectedsymbol{\istate, \action \sim \stationaryroom} 
 \norm{\embed\transitionfn_{\room}^{\direction}\fun{\sampledot \mid \istate, \action} - \latentprobtransitions_{\room}^{\direction}\fun{\sampledot \mid {\embed\fun{\istate}, \action}} }_1\\
= &\ 
\frac{1}{1 - \stationaryroom\fun{\sreset}} \sum_{\istate \in \support{\stationaryroom}} \sum_{\action \in \actions_{\room}} \left[ 
\stationaryroom\fun{\istate}\, \latentpolicy_{\room, \direction}\fun{\action \mid \embed\fun{\istate}} \right.
\left.\norm{\embed\transitionfn_{\room}^{\direction}\fun{\sampledot \mid \istate, \action} - \latentprobtransitions_{\room}^{\direction}\fun{\sampledot \mid {\embed\fun{\istate}, \action}} }_1 \right]
\\
= &\ 
\frac{1}{1 - \stationaryroom\fun{\sreset}} \sum_{\istate \in \support{\stationaryroomtrain}} \sum_{\action \in \actions_{\room}} \left[ 
\frac{\stationaryroom\fun{\istate}}{\stationaryroomtrain\fun{\istate}}\, \stationaryroomtrain\fun{\istate}  \, \latentpolicy_{\room, \direction}\fun{\action \mid \embed\fun{\istate}} \right.
\left.\norm{\embed\transitionfn_{\room}^{\direction}\fun{\sampledot \mid \istate, \action} - \latentprobtransitions_{\room}^{\direction}\fun{\sampledot \mid {\embed\fun{\istate}, \action}} }_1 \right]
\\
=&\ \frac{1}{1 - \stationaryroom\fun{\sreset}} \expectedsymbol{\istate, \action \sim \stationaryroomtrain} \left[
\frac{\stationaryroom\fun{\istate}}{\stationaryroomtrain\fun{\istate}}\,
 \norm{\embed\transitionfn_{\room}^{\direction}\fun{\sampledot \mid \istate, \action} - \latentprobtransitions_{\room}^{\direction}\fun{\sampledot \mid {\embed\fun{\istate}, \action}} }_1\right] \\
\leq &\ 
\frac{1}{1 - \stationaryroom\fun{\sreset}} \expectedsymbol{\istate, \action \sim \stationaryroomtrain} \left[
\max_{\istate' \in \support{\stationaryroomtrain}}\fun{\frac{\stationaryroom\fun{\istate'}}{\stationaryroomtrain\fun{\istate'}}}\right.
 \left.\norm{\embed\transitionfn_{\room}^{\direction}\fun{\sampledot \mid \istate, \action} - \latentprobtransitions_{\room}^{\direction}\fun{\sampledot \mid {\embed\fun{\istate}, \action}} }_1\right] 
 \\
= &\ 
\frac{1}{1 - \stationaryroom\fun{\sreset}} \max_{\istate \in \support{\stationaryroomtrain}}\fun{\frac{\stationaryroom\fun{\istate}}{\stationaryroomtrain\fun{\istate}}}
 \expectedsymbol{\istate, \action \sim \stationaryroomtrain} \left[\norm{\embed\transitionfn_{\room}^{\direction}\fun{\sampledot \mid \istate, \action} - \latentprobtransitions_{\room}^{\direction}\fun{\sampledot \mid {\embed\fun{\istate}, \action}} }_1\right] 
 \\
 = &\ \max_{\istate \in \support{\stationaryroomtrain}}\fun{\frac{\stationaryroom\fun{\istate}}{\stationaryroomtrain\fun{\istate}}} \frac{2 \transitionloss^{\room, \direction}}{1 - \stationaryroom\fun{\sreset}}
\end{align*}

If the initial distributions of the individual room $\room_{\direction, \policy}$ and the training room $\room_{\direction}$ have the same support, then 
the projection and the BSCCs coincide since the same set of states is eventually visited under $\latentpolicy$ from states of $\support{\mdpI_{\room}} = \support{\mdpI_{\room}^{ \policy}}$.
Furthermore, by~\cite[Thm.~1]{OCinneide1993}, we have
\allowdisplaybreaks
\begin{align}
	&\ \max_{\istate \in \support{\stationaryroomtrain}}\fun{\frac{\stationaryroom\fun{\istate}}{\stationaryroomtrain\fun{\istate}}}\\
    \leq&\ \max_{\istate \in \support{\stationaryroomtrain}} \max\set{\frac{\stationaryroom\fun{\istate}}{\stationaryroomtrain\fun{\istate}}, \frac{\stationaryroomtrain\fun{\istate}}{\stationaryroom\fun{\istate}}} \notag \\
    \leq &\  \fun{\max_{\istate \in \support{\mdpI_{\room}}} \max\set{\frac{\mdpI_{\room}^{\policy}\fun{\istate}}{\mdpI_{\room}\fun{\istate}}, \frac{\mdpI_{\room}\fun{\istate}}{\mdpI_{\room}^{ \policy}\fun{\istate}}}}^{\abs{\istates}} \tag{cf.\ Eq.~\eqref{eq:room-transition}}\\
    =&\ \kappa_{\room, \direction};
\end{align}
otherwise, we set $\kappa_{\room, \direction}$ to $\max_{\istate \in \support{\stationaryroomtrain}}\fun{\frac{\stationaryroom\fun{\istate}}{\stationaryroomtrain\fun{\istate}}}$.
Moreover, let $\istates_{\room, \direction} = \set{\tuple{\istate, \vertex, \varvertex} \in \istatesplanner \mid  \labelingfn\fun{\vertex} = \room \text{ and }\tuple{\vertex, \varvertex} = \direction}$ and define \[\stationary{\policy}\fun{\sreset \mid \room, \direction} = \expected{\tuple{\istate, \vertex, \varvertex}, \action \sim \stationary{\policy}}{\probtransitions_{\planner}\fun{\sreset \mid \tuple{\istate, \vertex, \varvertex}, \action} \mid \istates_{\room, \direction}}.\]
Notice that
\begin{align*}
    \stationaryroom\fun{\sreset} &=
    \expected{\tuple{\istate, \vertex, \varvertex}, \action \sim \stationary{\policy}}{\probtransitions_{\planner}\fun{\sreset \mid \tuple{\istate, \vertex, \varvertex}, \action} + \condition{\istate \in \exitfn_{\room}\fun{\direction}} \mid \istates_{\room, \direction}}\\
    &= \stationary{\policy}\fun{ \sreset  \mid {\room, \direction} } + \stationary{\policy}\fun{  \exitfn_{\room}\fun{\direction} \times \set{\direction} \mid \istates_{\room, \direction} }
\end{align*}
by (i)~the stationary property, (ii)~definition of $\probtransitions_{\room}^{\direction, \policy}$ (cf.\ Eq.~\eqref{eq:room-transition} and Fig.~\ref{fig:room-distribution}), (iii)~the fact that the probability of exiting the room is equal to the probability of visiting an exit state, and (iv)~the fact that resetting the room and visiting an exit state are disjoint events (when an exit state is visited, it always transitions to the next room, never to the reset state).

By putting all together, we have
\allowdisplaybreaks
\begin{align*}
    &\ \transitionloss^{\controller, \planner} \\
    \leq&\ 
    \entranceloss + \frac{1}{2} \, \expectedsymbol{\direction \sim \stationary{\policy}} \expectedsymbol{\istate, \action \sim \stationary{\policy}\fun{\sampledot \mid \direction}}\left[\condition{\istate \neq \sreset}\, \condition{\istate \not\in \exitfn_{\labelingfn\fun{\vertex}}\fun{{\direction}}}\right.
     \left. \norm{\embed\transitionfn\fun{\sampledot \mid \istate, \action} - \latentprobtransitions\fun{\sampledot \mid {\embed\fun{\istate}, \action}} }_1 \right]
     \\
    \leq &\ \entranceloss + \expectedsymbol{\room, \direction \sim \stationary{\policy}}  \frac{\kappa_{\room, \direction} \transitionloss^{\room, \direction}}{1 - \stationary{\latentpolicy_{\room, \direction}}^{\room_{\direction, \policy}}\fun{\sreset}} \\
    = &\ \entranceloss  + \expectedsymbol{\room, \direction \sim \stationary{\policy}}  \frac{\kappa_{\room, \direction} \transitionloss^{\room, \direction}}{1 - \stationary{\policy}\fun{\sreset \mid {\room, \direction}} - \stationary{\policy}\fun{\exitfn_{\room}\fun{\direction} \times \set{\direction} \mid \istates_{\room, \direction}}} \\
    \leq &\ 
    	 \entranceloss + \expectedsymbol{\room, \direction \sim \stationary{\policy}} \frac{\max \set{{\kappa_{\room^{\star}, \direction^{\star}}} \colon \room^{\star} \in \rooms, \direction^{\star} \in \directions_{\room^{\star}}} \, \transitionloss^{\room, \direction}}{1 - \max_{\room^{\star} \in \rooms, \direction^{\star} \in \directions_{\room}}\fun{\stationary{\policy}\fun{\sreset \mid {\room^{\star}, \direction^{\star}}} + \stationary{\policy}\fun{\exitfn_{\room^{\star}}\fun{\direction^{\star}} \times \set{\direction^{\star}} \mid \istates_{\room^{\star}, \direction^{\star}}}}}
    	\\
    \leq &\ \entranceloss + \frac{\kappa}{ \stationary{\text{continue}}^{\min}} \, \expectedsymbol{\room, \direction \sim \stationary{\policy}}{ \transitionloss^{\room, \direction}}
\end{align*}
where $\kappa = \max \set{{\kappa_{\room^{\star}, \direction^{\star}}} \colon \room^{\star} \in \rooms, \direction^{\star} \in \directions_{\room^{\star}}}$ and
\begin{equation}
  \stationary{\text{continue}}^{\min} = 1 - \max_{\room \in \rooms, \direction \in \directions_{\room}}\fun{\stationary{\policy}\fun{\sreset \mid {\room, \direction}} + \stationary{\policy}\fun{\exitfn_{\room}\fun{\direction} \times \set{\direction}\mid \istates_{\room, \direction}}}.
\end{equation}

This concludes the proof.
\end{proof}

\paragraph{Discussion.} Assumption~\ref{asmp:projection-included}
boils down to design an initial distribution for the simulator of each room that provides a sufficient coverage of the state space: the latter should include the states likely to be seen when the room is entered under any planner.
Then, if this initial distribution is powerful enough to provide an exact coverage of the entrance states visited under the planner $\controller$, the multiplier of the transition loss $\kappa$ can be determined solely based on the ratio of the initial distributions obtained during training and synthesis.
We summarize the results as follows.

\begin{theorem}[Value bound in $\hmdp$]\label{thm:final-value-bound}
Under the assumptions {of Thm.~\ref{thm:reusable-components}},
    \begin{align}
      \left| \values{\policy}{}{} -
      \latentvalues{\policy}{}{} \right| &\leq \frac{\entranceloss +  \nicefrac{\kappa }{\stationary{\text{continue}}^{\min}} \; \expectedsymbol{\room, \direction \sim \stationary{\policy}} \transitionloss^{\room, \direction}}{\stationary{\policy}\fun{\sreset}\fun{1 {-} \discount }}.
    \end{align}
\end{theorem}

\section{Experiments}\label{appendix:experiments}
In this section, we provide additional details on the experiments we performed.

\paragraph{Setup.}
Models have been trained on a cluster running under \texttt{CentOS Linux 7 (Core)} composed of a mix of nodes containing Intel processors with the following CPU microarchitectures:
(i) \texttt{10-core INTEL E5-2680v2}, (ii) \texttt{14-core INTEL E5-2680v4}, and (iii) \texttt{20-core INTEL Xeon Gold 6148}.
We used $8$ cores and $42$ GB of memory for each run during the hyperparameter search.
\paragraph{Learning the low-level policies.}%
We run WAE-DQN to learn the set of low-level policies~$\planner$ along with their latent-space models.
Recall the representation quality guarantees of our algorithm (cf.\ \secref{wae-dqn}): the same latent space can be used for rooms sharing similar features.
We leverage this property to learn only four latent policies (one per direction). 
In other words, only one pair of latent MDP and policy is learned per direction, which encompasses and generalize the behavior of the agent in all the \emph{training rooms} (cf.\ Appendix~\ref{sec:training-vs-synthesis}).
For instance, in a grid world environment composed of~$9$ rooms
with similar shapes,
we only train one latent policy per exit direction $\set{\leftarrow, \rightarrow, \uparrow, \downarrow}$ instead of~$9 \cdot 4 = 36$.
For training in a room~$\room$, we let~$\mdpI_{\room}$ uniformly distribute the agent's possible entry positions.
Adversaries' initial positions are randomly set by~$\mdpI_{\room}$ but may vary according to the function~$\entrancefn_\room$ in the high-level model (unknown at training time).
Objectives~$\objective^{\direction}_{\room}$ specify reaching the target exit while avoiding adversaries before the episode ends.

\paragraph{Synthesis.}%
To estimate the latent entrance function, we explore the high-level environment through random execution of the low-level latent policies.
We further consider \emph{Masked Autoencoders}
(MADEs, \cite{DBLP:conf/icml/GermainGML15}), which
allow
to learn complex distributions from a dataset.
With
the data collected via this exploration,
we train a
MADE to learn~$\latententrancefn_{\room}$ for any room~$\room$.
To learn those latent entrance functions,
consistently with WAE-MDPs, we use the same kind of MADE as the one introduced by \cite{delgrange2023wasserstein} for estimating the probability of the latent transition function.
We finally construct~$\latentmdp_{\planner}^{\graph}$ (cf.\ \secref{high_level}) and apply the synthesis procedure to obtain a two-level controller~$\policy = \tuple{\controller, \planner}$.
Tab.~\ref{tab:synthesis} reports the values of~$\policy$ obtained for various environment sizes.

\paragraph{Hyperparameter search.}
To train our WAE-DQN agent, we ran $4$ environments in parallel per training instance and used a replay buffer of capacity $7.5 \cdot 10^{5}$.
We performed a grid search to find the best parameters for our WAE-DQN algorithm.
Tab.~\ref{table:hyperparameter-search} presents the range of hyperparameters used. 
In particular, we found that prioritized experience replay does not improve the results in our environments significantly.
We used a batch size of $128$ for the WAE-MDP.

For the MADE modeling the latent entrance function, we used a dataset of size $25600$, and the training was split into $100$ epochs (i.e., the model performed $100$ passes through the entire dataset) with a learning rate of $10^{-3}$.
We used a batch size of $32$ or $64$, and two hidden layers, either with $64$ or $128$ neurons.

For generating the set of low-level policies $\planner$, we used the hyperparameters that worked the best for each specific direction.
We used the same parameters for the DQN training instances shown in  \figref{DRL}.

\subsection{Grid World Environment}
We provide additional details on the state space of our environment.
The agent has~$\textsc{Lp}$ life points, decrementing upon adversary contact or timeout.
Collecting power-ups (appearing randomly) shortly makes the agent invincible.
The state space comprises two components:
\begin{enumerate*}[(i)]
\item A $4$-dimensional bitmap $\mathbf{M} \in \set{0, 1}^{N \times l \times m \times n}$, where each layer in $k \in \set{1, \dots, l}$ corresponds to an item type on the grid; entry~$\mathbf{M}_{R, k, i, j}$ is~$1$ iff room~$R$ has item~$k$ in cell~$\fun{i, j}$;
\item step, power-up, and life-point counters~$\tuple{a, b, c}$.
\end{enumerate*}
\figref{grid-world-example} shows examples of rooms composed of $20\times20$ cells.
\begin{figure}
 \centering
 \includegraphics[width=.35\linewidth]{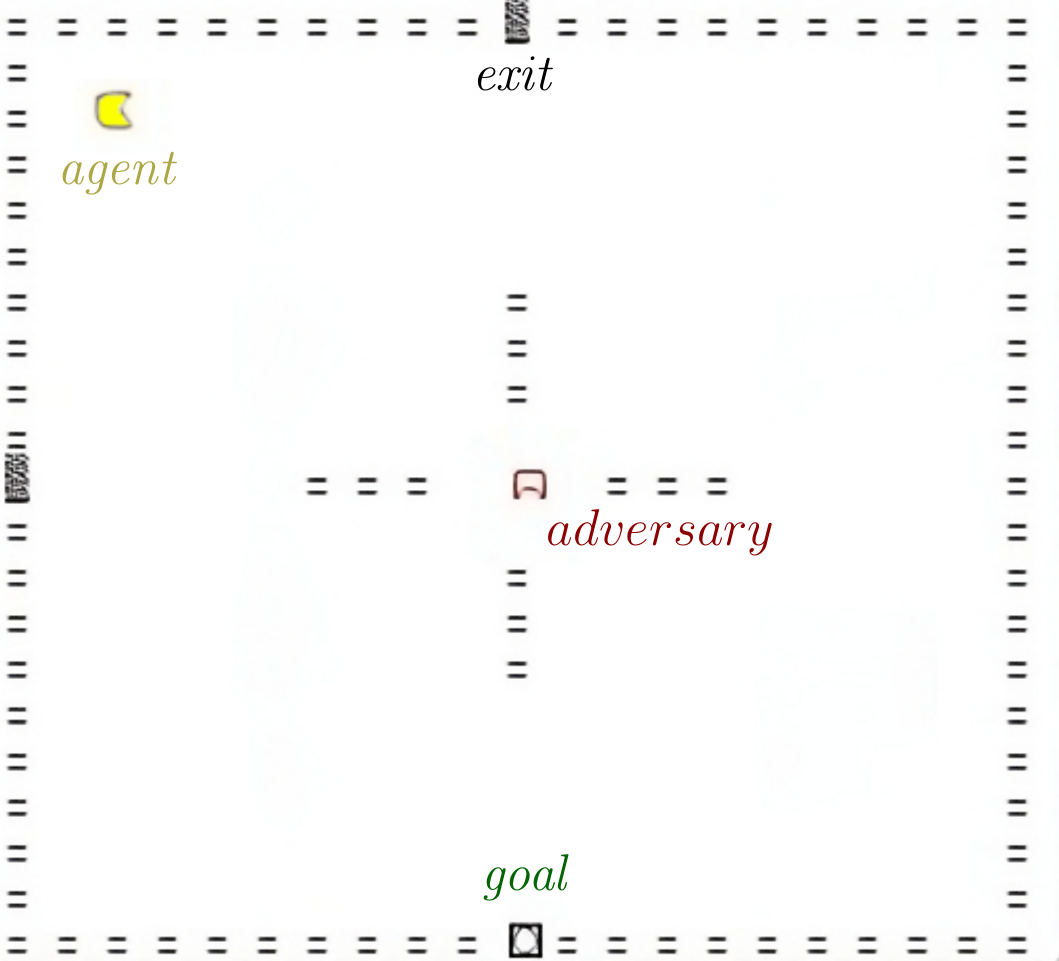}
 \includegraphics[width=.35\linewidth]{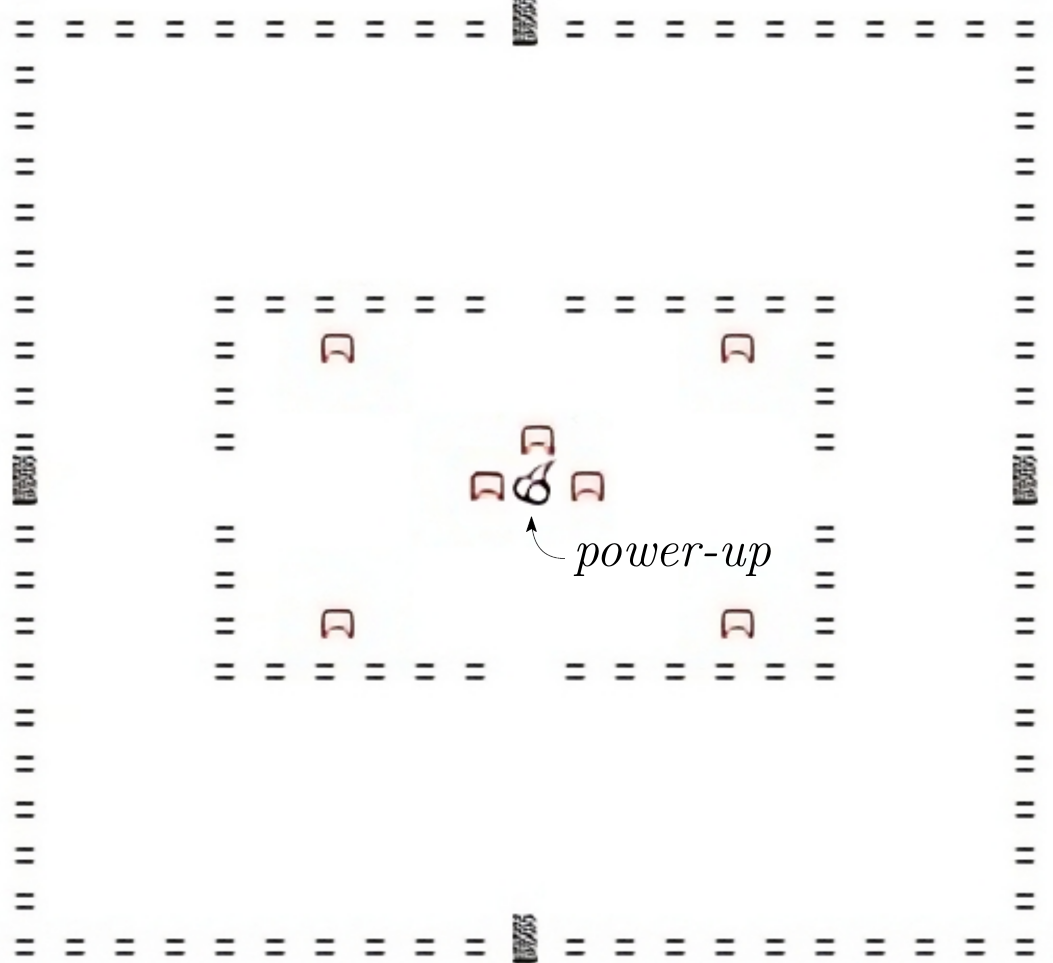}
 \caption{
 Two rooms of $20 \times 20$ cells ($9$ rooms in~\autoref{fig:grid-world-wide}).
 Demonstration with $9$ rooms: \texttt{\url{https://youtu.be/crowN8-GaRg}}.
 } 
 \label{fig:grid-world-example}
\end{figure}

\paragraph{DRL components.} 
We use
CNNs~\cite{DBLP:journals/neco/LeCunBDHHHJ89} to process bitmaps~$\mathbf{M}$ and a sparse reward signal $\reward\fun{\istate, \action, \istate'} = r^* \cdot \condition{\istate \in T} - r^* \cdot \condition{\istate \in B}$, where~$r^{*} > 0$ is an arbitrary reward (or conversely, a penalty) obtained upon reaching the target~$T$ (or an undesirable state in~$B$).
To guide the agent, we add a \emph{potential-based reward shaping}~\cite{DBLP:conf/icml/NgHR99,DBLP:journals/jair/Wiewiora03} based on the~$L_1$ distance to the target.
The resulting reward function is $\reward_{\Phi}\fun{\istate, \action, \istate'} = \discount\Phi\fun{\istate'} - \Phi\fun{\istate} + \reward\fun{\istate, \action, \istate'}$ where
\begin{equation}   
\Phi\fun{\istate} = 1 - \frac{\min\set{\abs{x(t_1) - x\fun{\istate}} + \abs{y\fun{t_2} - y\fun{\istate}} \colon t_1, t_2 \in T}}{N \cdot \fun{m + n}},
\label{eq:reward-shaping}
\end{equation}
and~$x\fun{\istate}$, $y\fun{\istate}$ respectively return the Euclidean coordinates
along the horizontal and vertical axes
corresponding to state~$\istate \in \istates$.
Intuitively, $ \abs{\Phi\fun{\istate} - 1}$ reflects the normalized distance of state~$\istate$ to the targets~$T$.
When the agent gets closer (resp.~further) to~$T$ when executing an action, the resulting reward is positive (resp.~negative).
Our DQN implementation uses state-of-the-art extensions and improvements from~\cite{DBLP:conf/aaai/HesselMHSODHPAS18}. 
Nevertheless, as demonstrated in Fig.~\ref{fig:DRL}, while DQN reduces contact with adversaries,
the two-level nature of the decisions required to reach a target hinders learning the high-level objective.

\paragraph{DQN and WAE-DQN experiments.}
We provide a more detailed version of \figref{DRL} in \figref{DRL-2}, where the WAE-DQN performance is specified per direction.
Precisely, we trained five different instances of the algorithm per policy with different random seeds, where the solid line is the mean and the shaded interval is the standard deviation.
To train the DQN agent, we set a time limit five times longer than that used for training rooms with the WAE-DQN agents. Furthermore, the DQN agent is equipped with~$3$ life points, while the WAE-DQN agents are limited to one.

\begin{figure}[t]
    \centering
    \begin{subfigure}[c]{0.475\textwidth}
        \centering
        \includegraphics[width=.9\textwidth]{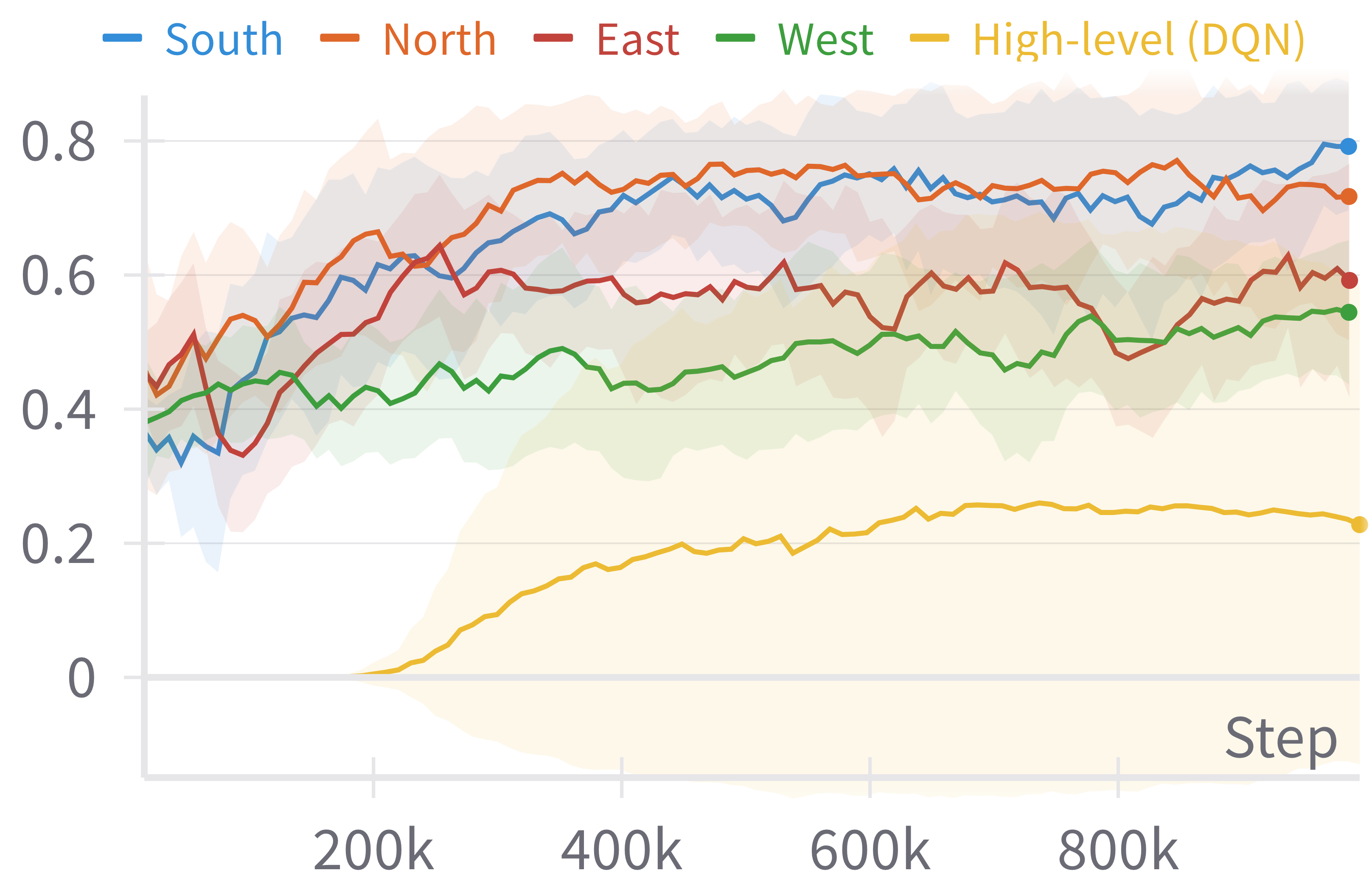}
        \caption{Goal reached (average over $30$ episodes).}
        \label{subfig:goal-reached-2}
    \end{subfigure}
    \hfill
    \begin{subfigure}[c]{0.475\textwidth}
        \centering
        \includegraphics[width=.9\textwidth]{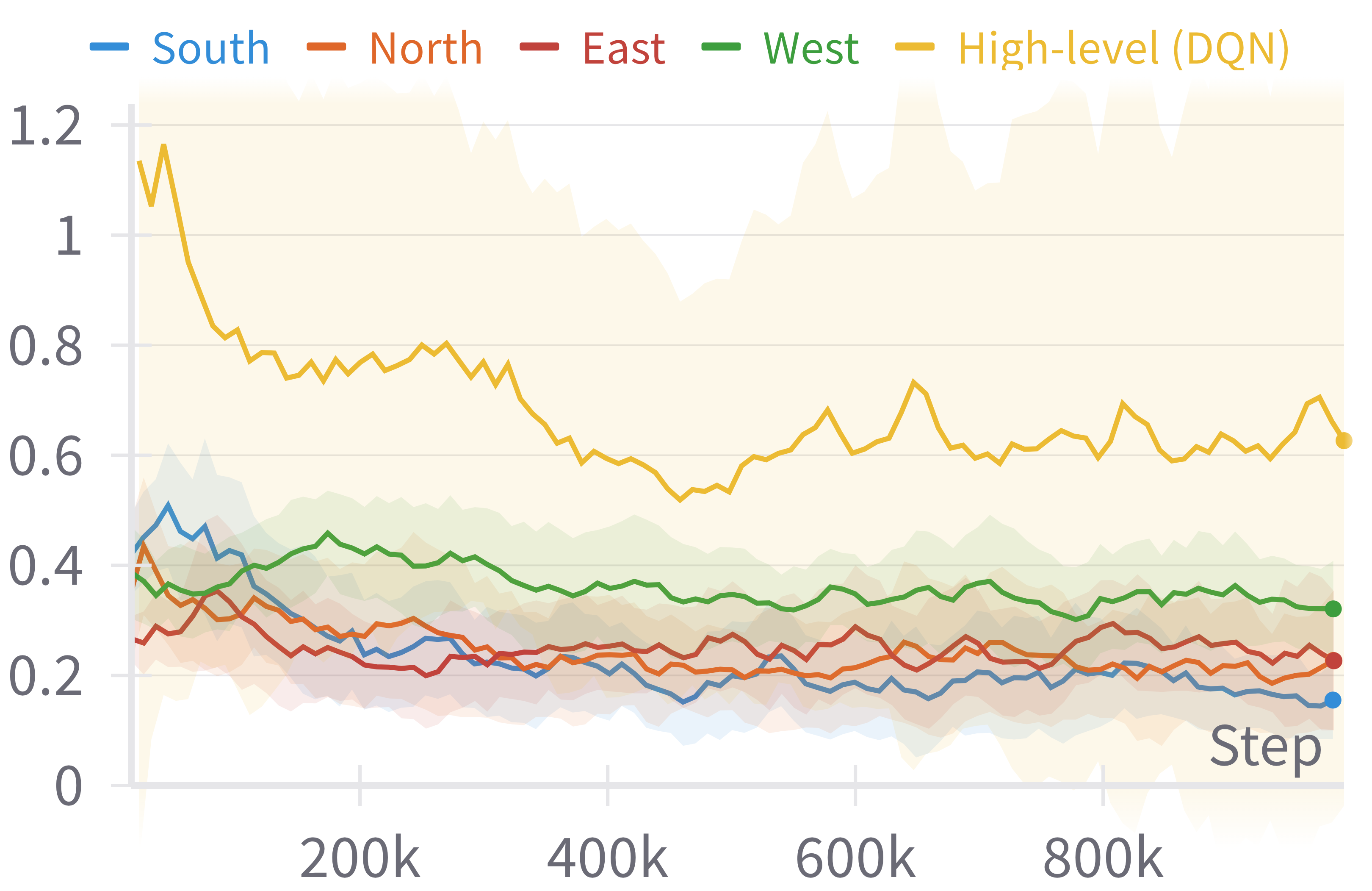}
        \caption{Failures: adversaries hit (averaged over $30$ episodes).}
        \label{subfig:failures-2}
    \end{subfigure}
    \caption{A more detailed version of \figref{DRL}, where the WAE-DQN performance is specified per direction.
    We train five different instances of the algorithm per policy with different random seeds, where the solid line corresponds to the mean and the shaded interval to the standard deviation. To train the DQN agent, we set a time limit five times longer than that used for training rooms with the WAE-DQN agents. Note that the DQN agent is equipped with $3$ life points, while the WAE-DQN agents are limited to one.
    }
    \label{fig:DRL-2}
\end{figure}

\begin{figure}
  \centering
  \includegraphics[width=.9\linewidth]{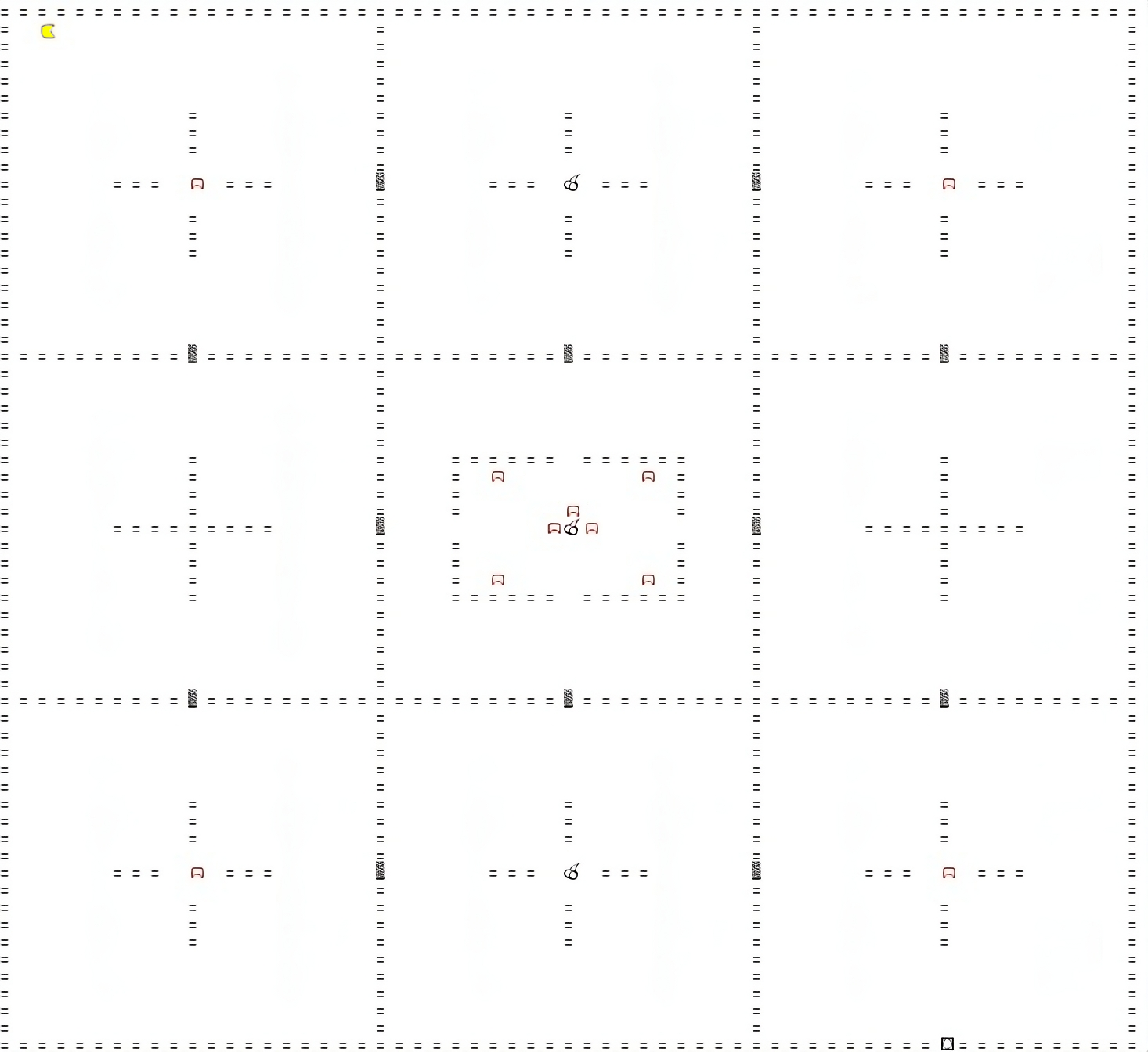}
  \caption{Environment for $N = 9$ rooms of $20 \times 20$ cells. The agent is depicted in yellow (top left), adversaries in red, power-ups as cherries, and the goal at the bottom right.}\label{fig:grid-world-wide}
  \label{fig:grid-world}
\end{figure}

\highlight{%
\subsection{\texttt{ViZDoom} Environment}
\begin{figure}
    \centering
    \begin{subfigure}[c]
    {.6\linewidth}
    \centering\includegraphics[width=\linewidth]{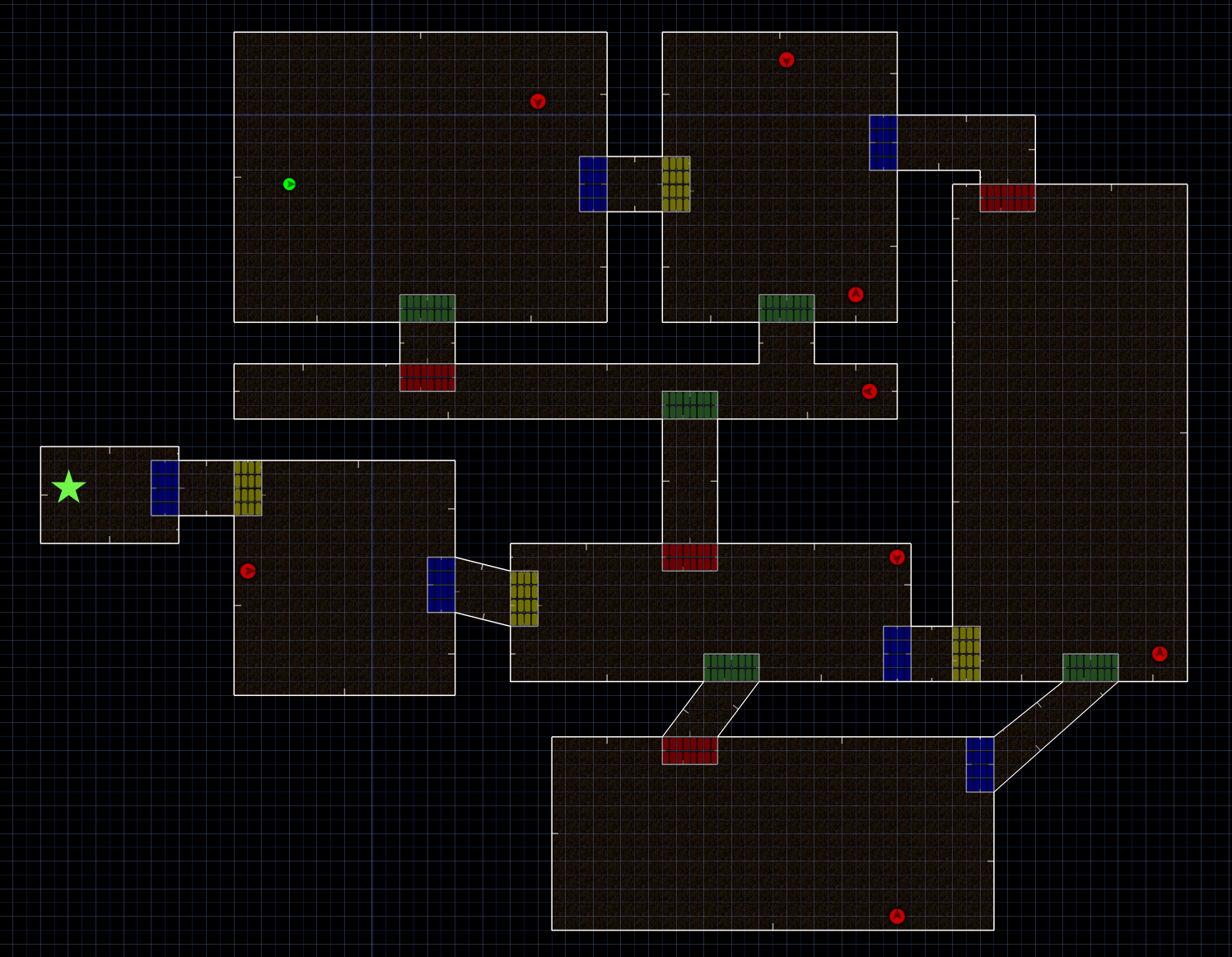}
    \caption{Top view of the \texttt{ViZDoom} environment ($8$ rooms). 
    The agent is located at the top leftmost room (depicted by a {\color{green}green point}). 
    {\color{red}Red points} depict enemies on the map. 
    The goal of the agent is to reach the goal position ${\color{green}\star}$. 
    Exit positions are depicted by colored rectangles: in {\color{blue}blue}, the agent needs to exit to the right, in {\color{yellow} yellow} to the left, in {\color{green}green} below, and in {\color{red} red} above.
    Note that enemies spawn at random positions every $250$ ``ticks'' ($4$ game ticks correspond to $1$ step in the environment).
    }   
    \label{fig:doom-map}
    \end{subfigure}
    \hfill
    \begin{subfigure}[c]{.38\linewidth}
        \centering
        \includegraphics[width=.475\textwidth]{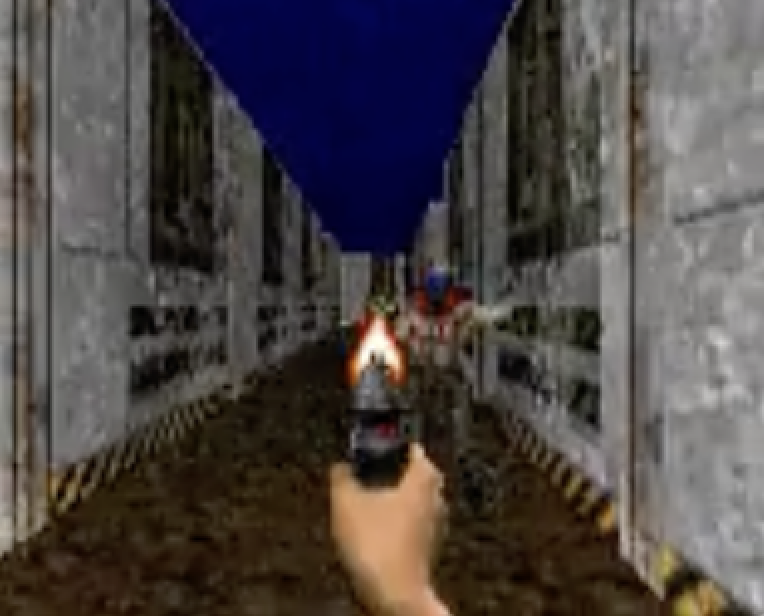}
        \hfill
        \includegraphics[width=.5\textwidth]{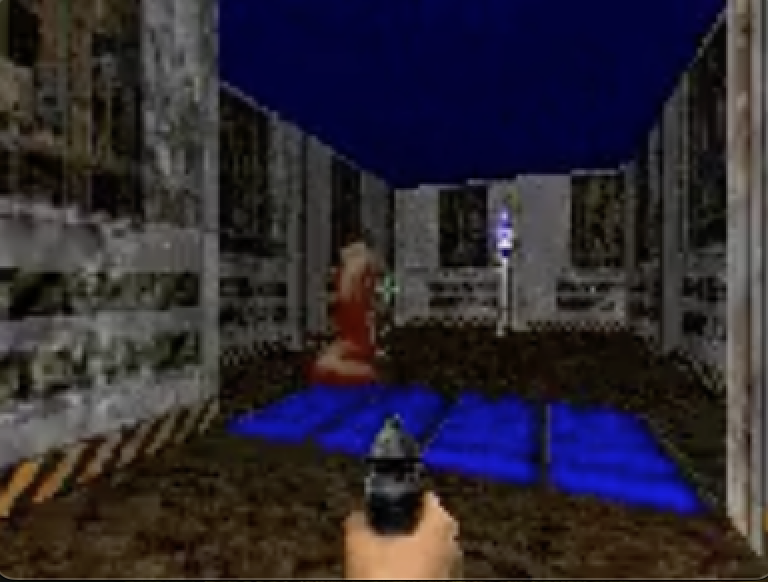}
        \hfill
        \includegraphics[width=.475\textwidth]{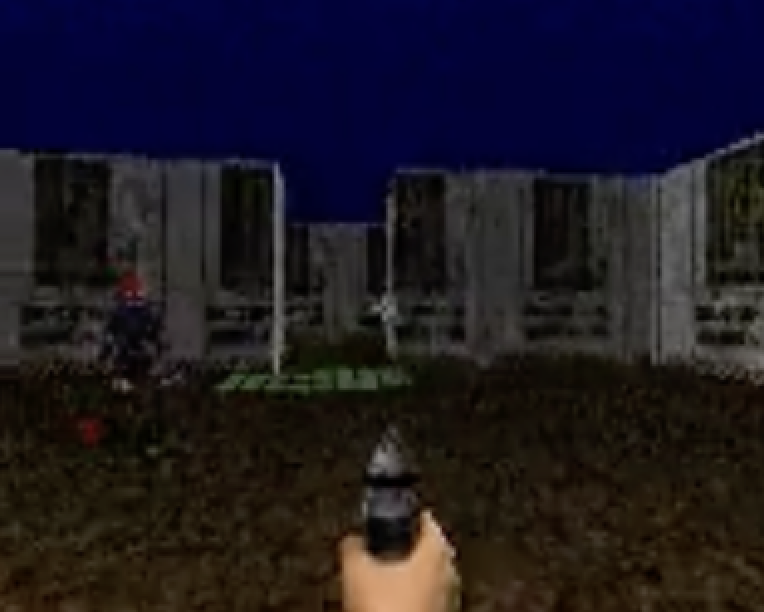}
        \hfill
        \includegraphics[width=.475\textwidth]{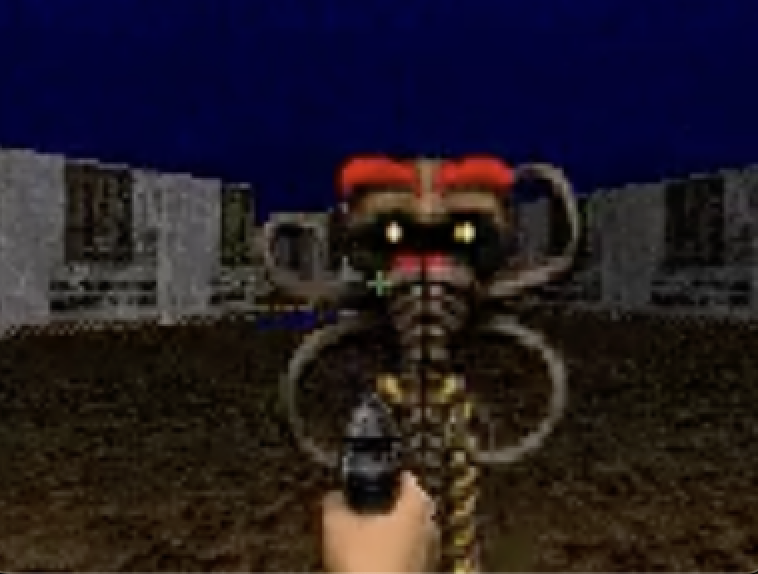}
        \caption{Samples of frames processed by the agent.}
    \end{subfigure}
    \caption{\texttt{ViZDoom environment}.
    Demonstration of a synthesized controller: \url{https://youtu.be/BAVLmsWEaQY}.}
    \label{fig:vizdoom}
\end{figure}
A top view of the environment and samples of visual inputs processed by the agent are provided in Fig.~\ref{fig:vizdoom}.
Each game frame consists of an image of size $60 \times 45$. 
Note that we do not stack frames in the agent's observation. 
Indeed, the agent additionally processes velocities, its angle in the map, and its health, which makes the observation Markovian.
We used CNNs to process the game frames with the same architecture as in~\cite{DBLP:conf/cig/KempkaWRTJ16}.

As mentioned in Sect.~\ref{sec:evaluation}, this environment is very challenging due to the nature of its observation space. 
In our experiments, we found that, instead of using the discrete latent states directly, it was beneficial to use the continuous relaxation of discrete random variables \cite{DBLP:conf/iclr/MaddisonMT17} learned by WAE-MDPs (see~\cite{delgrange2023wasserstein}) as latent states to explore the environment.
This allowed for a smoother optimization. 
Continuous relaxations rely on a temperature parameter, which intuitively controls the continuity of the latent space. 
When annealed to zero (the ``zero-temperature limit''), the latent space is guaranteed to be discrete. 
We used the latent space at its zero-temperature limit to verify the values in the latent model learned.

\paragraph{Reward function}
Denote by $\textit{health}\fun{\istate}$ the health of the agent in state $\istate \in \istates$ and
define $\textit{hit}\fun{\istate, \action, \istate'}$ as the function that returns a constant $C > 0$ when $\action \in \actions$ is the ``shoot'' action and an enemy is hit in $\istate'$, $-C$ if no enemy is hit, and $0$ if $\action$ is any other action.
The base reward function of the agent is given by 
$\reward\fun{\istate, \action, \istate'} = \textit{health}\fun{\istate'} - \textit{health}\fun{\istate} + r^* \cdot \condition{\istate \in T} + \textit{hit}\fun{\istate, \action, \istate'}$, where, as for the grid world environment, $r^* > 0$ is an arbitrary reward obtained upon reaching the target set $T$. 
To guide the agent, we use exactly the same reward shaping scheme as the one we defined for the grid world environment (Eq.~\ref{eq:reward-shaping}).
}

\definecolor{mediumgreen1}{RGB}{0, 0, 0}
\begin{table}
\centering
\caption{Hyperparameter range used for (WAE-)DQN.
If ``optimization scheme'' is set to ``round robin'', separate optimizers are used for the policy, the state embedding function, the WAE-MDP minimizer, and the WAE-MDP maximizer.
If ``concatenate'' is used, the policy, the state embedding function, and the WAE-MDP minimizer share the same optimizer. 
In that case, all those components are optimized at once by concatenating their loss functions.
For details about the WAE-MDP parameters, see~\cite{delgrange2023wasserstein}.}\label{table:hyperparameter-search}
\label{tab:hyperparameters}
\begin{tabular}{@{}llll@{}}
\toprule
\textbf{Parameter} &
  \textbf{Range} &
  \textbf{Grid World} &
  \texttt{\highlight{ViZDoom}} \\ \midrule
\multicolumn{4}{c}{Common to DQN and WAE-DQN} \\ \midrule
Activation &
  $\set{\text{\color{mediumgreen1}ReLU}, \text{\color{mediumgreen1}LeakyReLu}, \text{ELU}, {\tanh}, \text{sigmoid}}, \text{SiLU}$ &
  LeakyReLu &
  sigmoid \\
\# Hidden layers per network &
  $\set{1, 2, {\color{mediumgreen1}3}}$ &
  3 &
  2 \\
\# Neurons per layer &
  $\set{{\color{mediumgreen1}128}, 256, {512}}$ &
  128 &
  256 \\
CNN filters &
   &
  $3 \to 3 \to 3$ &
  $7 \to 4$ (strides $2$) \\
CNN kernels &
   &
  $64 \to 32 \to 16$ &
  $32 \to 32$ \\
Optimization scheme &
  $\set{\text{round robin, concatenate}}$ &
  round robin &
  concatenate \\ \midrule
\multicolumn{4}{c}{DQN} \\ \midrule
Use Boltzmann exploration &
  $\set{\text{\color{mediumgreen1}Yes}, \text{No}}$ &
  Yes &
  Yes \\
Boltzmann temperature &
  $\set{0.25, 0.10, 0.5, {\color{mediumgreen1}0.75}, {\color{mediumgreen1}1}, {\color{mediumgreen1}10}, 100}$ &
  0.75 &
  0.5 \\
Use $\epsilon$-greedy exploration (decay to $\epsilon=0.1$) &
  $\set{\text{Yes}, \text{\color{mediumgreen1}No}}$ &
  No &
  No \\
Target update period &
  $\set{{\color{mediumgreen1}1}, {\color{mediumgreen1}250}, 500, 1000}$ &
  250 &
  250 \\
Target update scale ($\alpha$ in Algorithm~\ref{algo:wae-dqn}) &
  $\set{{\color{mediumgreen1}10^{-4}}, 5 \cdot 10^{-4}, {\color{mediumgreen1}10^{-3}}, 5 \cdot 10^{-3}, 1}$ &
  1 &
  1 \\
Reward scaling &
  $\set{1, 10, 25, {\color{mediumgreen1}100}}$ &
  100 &
  1 \\
Learning rate &
  $\set{{\color{mediumgreen1}6.25 \cdot 10^{-5}}, {10}^{-4}, 2.5 \cdot 10^{-4}, 10^{-3} }$ &
  $6.25 \cdot 10^{-5}$ &
  $6.25 \cdot 10^{-5}$ \\
Batch size &
  $\set{32, {\color{mediumgreen1}64}, 128}$ &
  $64$ &
  $64$ \\
Use double $Q$-networks~\cite{DBLP:conf/aaai/HasseltGS16} &
  $\set{\text{\color{mediumgreen1}Yes}, \text{No}}$ &
  Yes &
  Yes \\
Categorical network \cite{DBLP:conf/icml/BellemareDM17} &
  $\set{\text{\color{mediumgreen1}Yes}, \text{No}}$ &
  No &
  Yes \\ \midrule
\multicolumn{4}{c}{WAE-MDP} \\ \midrule
Latent state size (power of $2$) &
  $\set{12, {\color{mediumgreen1}13}, 14, 15}$ &
  13 &
  14 \\
State embedding function temperature &
  $\set{{\color{mediumgreen1}\nicefrac{1}{3}}, \nicefrac{1}{2}, \nicefrac{2}{3}, \nicefrac{3}{4}, 0.999}$ &
  $\nicefrac{1}{3}$ &
  $0.999$ \\
Transition function temperature &
  $\set{ {\color{mediumgreen1}\nicefrac{1}{3}}, \nicefrac{1}{2}, \nicefrac{2}{3}, \nicefrac{3}{4}, 0.999}$ &
  $\nicefrac{1}{3}$ &
  $\nicefrac{1}{2}$ \\
Steady-state regularizer scale factor &
  $\set{0.01, 0.1, 1, 10, 25, {\color{mediumgreen1}50}, 75}$ &
  50 &
  $0.01$ \\
Transition regularizer scale factor &
  $\set{10^{-2} , 10^{-1}, 1, 10, 25, {\color{mediumgreen1}50}, 75}$ &
  50 &
  $0.1$ \\
Minimizer learning rate &
  $\set{{\color{mediumgreen1}10^{-4}}, 5\cdot10^{-4}, {\color{mediumgreen1}10^{-3}}}$ &
  $10^{-4}$ &
  / \\
Maximizer learning rate &
  $\set{{\color{mediumgreen1}10^{-4}}, 5\cdot10^{-4}, 10^{-3}}$ &
  $10^{-4}$ &
  $10^{-3}$ \\
State embedding function learning rate &
  $\set{{\color{mediumgreen1}10^{-4}}, 5\cdot10^{-4}, 10^{-3}}$ &
  $10^{-4}$ &
  / \\
\# critic updates &
  $\set{3, {\color{mediumgreen1}5}, 10, 15}$ &
  $5$ &
  $3$ \\
State reconstruction function &
  $\set{\text{none}, {\color{mediumgreen1}L_2}, \text{binary cross entropy (for $\mathbf{M}$)} }$ &
  $L_2$ &
  No reconstruction \\ \bottomrule
\end{tabular}%
\end{table}

\section{Broader Impact}\label{sec:broader-impact}
Our work presents primarily theoretical and fundamental results, enhancing the reliability of RL solutions.
Our claims are also illustrated experimentally with an experimental environment (involving an agent moving within a grid world amid moving adversaries).
Specifically, our approach focuses on providing performance (``reach-'') and safety (``avoid'') guarantees with RL policies.
We believe our work may have positive societal impacts in the long-term, including
\begin{enumerate*}[(i)]
    \item \emph{safety-critical applications}: prevent failures (in, e.g., autonomous driving, healthcare, robotics);
    \item \emph{trust and wide adoption}: builds and improves confidence in RL solutions;
    \item \emph{avoiding harmful behavior}: mitigates unintended, risky actions; and
    \item \emph{performance compliance}: check whether performance standard are met (e.g., in industry).
\end{enumerate*}

\end{document}